\pdfoutput=1
\documentclass{article}
\usepackage{microtype}
\usepackage{graphicx}
\usepackage{subcaption}
\usepackage{booktabs} 
\usepackage{amsthm}
\usepackage{hyperref}
\usepackage{fullpage}
\usepackage{algorithm}
\usepackage{algorithmic}
\usepackage{natbib}
\usepackage{authblk}
\usepackage[table,xcdraw]{xcolor}

\usepackage[textsize=tiny,
disable
]{todonotes}
\makeatletter
\renewcommand{\todo}[2][]{\@todo[#1]{#2}}
\makeatother

\newcommand{\mA}{\mathcal{A}}
\newcommand{\mS}{\mathcal{S}}

\usepackage[utf8]{inputenc} 
\usepackage{hyperref}       
\usepackage{url}            
\usepackage{booktabs}       
\usepackage{amsfonts}       
\usepackage{nicefrac}       
\usepackage{amsmath}
\usepackage{amssymb}

\usepackage{color}

\newcommand{\cV}{\mathcal{V}}

\newcommand{\cS}{\mathcal{S}}
\newcommand{\cA}{\mathcal{A}}
\newcommand{\cF}{\mathcal{F}}

\newcommand{\RR}{\mathbb{R}}

\newcommand{\EE}{\mathbb{E}}

\newcommand{\cB}{\mathcal{B}}

\def\cX{\mathcal{X}}

\newcommand{\cN}{\mathcal{N}}

\ifx\theorem\undefined
\newtheorem{theorem}{Theorem}
\fi
\ifx\example\undefined
\newtheorem{example}{Example}
\fi
\ifx\property\undefined
\newtheorem{property}{Property}[theorem]
\fi
\ifx\lemma\undefined
\newtheorem{lemma}[theorem]{Lemma}
\fi
\ifx\proposition\undefined
\newtheorem{proposition}{Proposition}
\fi
\ifx\remark\undefined
\newtheorem{remark}{Remark}
\fi
\ifx\corollary\undefined
\newtheorem{corollary}[theorem]{Corollary}
\fi
\ifx\definition\undefined
\newtheorem{definition}{Definition}
\fi
\ifx\conjecture\undefined
\newtheorem{conjecture}[theorem]{Conjecture}
\fi
\ifx\fact\undefined
\newtheorem{fact}[theorem]{Fact}
\fi
\ifx\claim\undefined
\newtheorem{claim}[theorem]{Claim}
\fi
\ifx\assump\undefined
\newtheorem{assumption}{Assumption}
\fi
\ifx\cond\undefined
\newtheorem{cond}[theorem]{Condition}
\fi
\ifx\problem\undefined
\newtheorem{problem}{Problem}
\fi

\DeclareMathOperator{\diam}{diam}
\newcommand{\Prob}[1]{\mathbb{P}(#1)}
\newcommand{\E}{\mathbb{E}}
\newcommand{\ip}[1]{\langle #1 \rangle}

\newcommand{\R}{\mathbb{R}}
\DeclareMathOperator{\argmin}{argmin}
\DeclareMathOperator{\argmax}{argmax}
\newcommand{\norm}[1]{\| #1\|}
\newcommand{\cP}{\mathcal{P}}
\DeclareMathOperator{\dimE}{dim_{\mathcal{E}}}

\def\cX{\mathcal{X}}
\def\cZ{\mathcal{Z}}
\def\FF{\mathbb{F}}

\makeatletter
\newcommand*{\bigcdot}{}
\DeclareRobustCommand*{\bigcdot}{%
  \mathbin{\mathpalette\bigcdot@{}}%
}
\newcommand*{\bigcdot@scalefactor}{.5}
\newcommand*{\bigcdot@widthfactor}{1.15}
\newcommand*{\bigcdot@}[2]{%
  \sbox0{$#1\vcenter{}$}
  \sbox2{$#1\cdot\m@th$}%
  \hbox to \bigcdot@widthfactor\wd2{%
    \hfil
    \raise\ht0\hbox{%
      \scalebox{\bigcdot@scalefactor}{%
        \lower\ht0\hbox{$#1\bullet\m@th$}%
      }%
    }%
    \hfil
  }%
}
\makeatother

\def\red#1{}
\def\cs#1{}

\title{Model-Based Reinforcement Learning with Value-Targeted Regression}
\author[1]{Alex Ayoub \thanks{aayoub@ualberta.ca}}
\author[2]{Zeyu Jia \thanks{jiazy@pku.edu.cn}}
\author[1,6]{Csaba Szepesv{\'a}ri \thanks{szepesva@ualberta.ca}}
\author[3,4,6]{Mengdi Wang \thanks{mengdiw@princeton.edu}}
\author[5]{Lin F. Yang \thanks{linyang@ee.ucla.edu}}

\affil[1]{Department of Computing Science, University of Alberta}
\affil[2]{School of Mathematical Science, Peking University}
\affil[3]{Department of Electrical Engineering, Princeton University}
\affil[4]{Center for Statistics and Machine Learning, Princeton University}
\affil[5]{Department of Electrical and Computer Engineering, University of California, Los Angeles}
\affil[6]{DeepMind}

\begin{document}

\maketitle

\begin{abstract}
	This paper studies model-based reinforcement learning (RL) for regret minimization.
We focus on finite-horizon episodic RL where the transition model $P$ belongs to a known family of models $\mathcal{P}$, a special case of which is when models in $\mathcal{P}$ take the form of linear mixtures:
$P_{\theta} = \sum_{i=1}^{d} \theta_{i}P_{i}$. 
We propose a model based RL algorithm that is based on optimism principle:
In each episode, the set of models that are `consistent' with the data collected is constructed.
The criterion of consistency is based on the total squared error of that the model incurs on the task of predicting \emph{values} as determined by the last value estimate along the transitions.
The next value function is then chosen by solving the optimistic planning problem with the constructed set of models.
We derive a bound on the regret, which, in the special case of linear mixtures, 
the regret bound takes the form $\tilde{\mathcal{O}}(d\sqrt{H^{3}T})$, where $H$, $T$ and $d$ are the horizon, total number of steps and dimension of $\theta$, respectively. 
In particular, this regret bound is independent of the total number of states or actions, and is close to a lower bound $\Omega(\sqrt{HdT})$. 
For a general model family $\mathcal{P}$, the regret bound is derived 
using the notion of the so-called Eluder dimension proposed by \citet{RuVR14}.
\end{abstract}

\section{Introduction}
Reinforcement learning (RL) enables learning to control complex environments through trial and error. 
It is a core problem in artificial intelligence~\citep{RuNo03,sutton2018introduction}  
and recent years has witnessed phenomenal empirical advances in various areas
such as: games, robotics and science \cite[e.g.,][]{mnih2015human,silver2017mastering,alquraishi2019alphafold,arulkumaran2019alphastar}.
In online RL, an agent has to learn to act in an unknown environment ``from scratch'', 
collect data as she acts, and adapt the policy to maximize the reward collected. 
An important problem is to design algorithms that provably achieve sublinear regret in a large class of environments.
Regret minimization for RL has received considerable attention in recent years
(e.g., \citealt{jaksch2010near, osband2014generalization,azar2017minimax, dann2017unifying, dann2018policy, agrawal2017optimistic,osband2017deep, jin2018q,yang2019reinforcement,jin2019provably}).
While most of these existing works focus on the tabular or linear-factored MDP, only a handful of prior efforts have studied RL with general model classes.
In particular, in a pioneering paper \citet{Stre00} proposed to use posterior sampling,
which was later analyzed in the Bayesian setting by \citet{osband2014model,abbasi2015bayesian,GeZhAYVl17}.
The reader is referred to Section~\ref{sec:relatedwork} for a discussion of these and other related works.

In this paper, we study episodic reinforcement learning in an environment where the unknown probability transition model is known to belong to a family of models, i.e., $P\in\mathcal{P}$. The model family $\mathcal{P}$ is a general set of models, and it may be either finitely parametrized or nonparametric. 
In particular, our approach accommodates working with smoothly parameterized models \citep[e.g.,][]{abbasi2015bayesian}, and can find use in both robotics \citep{kober2013reinforcement} and queueing systems \citep{kovalenko1968introduction}. An illuminating special case is the case of linear parametrization when elements of $\mathcal{P}$ take the form $P_{\theta} = \sum_{i}\theta_i P_i$ where $P_1, P_2, \ldots, P_{d}$ are fixed, known basis models and $\theta=(\theta_1,\dots,\theta_d)$ are unknown, real-valued parameters. 
Model $P_\theta$ can be viewed as a mixture model that aggregates a finite family of known basic dynamical models \citep{modi2019sample}. As an important special case, linear mixture models include the linear-factor MDP model of 
\citet{yang2019reinforcement}, a model 
that allows the embedding of possible transition kernels into an appropriate space of finite matrices.

The main contribution of this paper is a model-based upper confidence RL algorithm
where the main novelty is the criterion to select models that are deemed consistent with past data.
As opposed to standard practice where the models are selected based on their ability to predict next states or raw observations there 
(cf. \citet{jaksch2010near, yang2019reinforcement}
or
 \citep{Stre00,osband2014model,abbasi2015bayesian,OuGaNaJa17,agrawal2017optimistic} in a Bayesian setting),
we propose to evaluate models based on their ability to predict the values at next states 
as computed using the last value function estimate produced by our algorithm.
In effect, the algorithm aims to select models based on their ability to produce small losses in 
a \emph{value-targeted} regression problem.

Value-targeted regression is attractive for multiple reasons:
{\em (i)} First and foremost, value-targeted regression holds the promise that model learning will \emph{focus on task-relevant aspects} of the transition dynamics and can ignore aspects of the dynamics that are not relevant for the task.
This is important as the dynamics can be quite complicated and modelling irrelevant aspects of the dynamics can draw valuable resources away from modelling task-relevant aspects.
{\em (ii)} A related advantage is that building faithful probability models with high-dimensional state variables (or observations) can be challenging. 
Value-targeted regression sets up model learning as a real-valued regression problem, which intuitively 
feels easier than either building a model with maximum likelihood or setting up 
a vector-valued regression problem to model next state probabilities.
{\em (iii)} Value-targeted regression aims at directly what matters in terms of the model accuracy or regret. Specifically the objective used in value-targeted is obtained from an expression that upper bounds the regret, hence it is natural to expect that minimizing this will lead to a small regret.

In addition, our approach is attractive as the algorithm has a modular structure and this allows any advances on components (optimistic planning, improvements in designing confidence sets) to be directly translated into a decreased regret. One may also question whether value-targeted regression is going ``too far'' in ignoring details of the dynamics. Principally, one may think that since the value function used in defining the regression targets is derived based on imperfect knowledge, the model may never be sufficiently refined in a way that allows the regret to be kept under control.  Secondly, one may worry about that by ignoring the rich details of observations (in our simple model, the identity of the state), the approach advocated is ignoring information available in the data, which may slow down learning. 
To summarize, the main question, to which we seek an answer in this paper, is the following:
\begin{center}
\emph{Is value-targeted regression sufficient and efficient for model-based online RL?}
\end{center}
Based on the theoretical and the experimental evidence that we provide in this paper, our conclusion is that the answer is `yes'.

Firstly, the regret bounds we derive conclusively show that the despite  the imperfection and non-stationarity of the value targets, the algorithm cannot get ``stuck'' (i.e., it enjoys sublinear regret).
Our results also suggest that perhaps there is no performance degradation as compared to the performance of competing algorithms.
We are careful here as this conclusion is based on comparing worst-case upper bounds, which cannot provide a definitive answer.

To complement the theoretical findings, our experiments also confirm that our algorithm is competitive. 
The experiments also allow us to conclude that it is value-targeted regression \emph{together} with optimistic planning that is effective.
In particular, 
if optimism is taken away (i.e., $\epsilon$-greedy is applied for the purpose of providing sufficient exploration), 
value-targeted regression performs worse than using a canonical approach to estimate the model. 
Similarly, if value-targeted regression is taken away, optimism together with the canonical model-estimation approach is less sample-efficient.

This still leaves open the possibility that certain combinations of value-targeted regression and canonical model building can be more effective than value-targeted regression. In fact, given the vast number of possibilities, we find this to be a quite plausible hypothesis. We note in passing that our proofs can be adjusted to deal with adding simultaneous alternative targets for constraining the set of data-consistent models.
However, sadly, our current theoretical tools are unable to exhibit the tradeoffs that one expects to see here.

It is interesting to note that, 
in an independent and concurrent work, 
value-targeted regression has also been suggested as the main model building tool of the MuZero algorithm.
The authors of this algorithm
empirically evaluated MuZero on a number of RL benchmarks, such as the 57 Atari ``games'', the game of ``Go'', chess and shogi \citep{schrittwieser2019mastering}. In these benchmarks, despite the fact that MuZero does not use optimistic exploration or any other ``smart'' exploration technique, MuZero was found to be highly competitive with its state-of-the-art alternatives, which reinforces the conclusion that training models using value-targeted regression is indeed a good approach to build effective model-based RL algorithms. 
The good results of MuZero on these benchmark may seem to contradict our experimental findings that value-targeted regression is ineffective without an appropriate, `smart' exploration component. However, there is no contradiction: Smart exploration may be optional in some environments; our experiments show that it is not optional on some environments. In short, for robust performance across a wide range of environments, smart exploration is necessary but smart exploration may be optional in some environments.

As to the organization of the rest of the paper, the next section (Section~\ref{sec:model}) introduces the formal problem definition. This is followed by the description of the new algorithm (Section~\ref{sec:alg}) and the main theoretical results (Section~\ref{sec:results}). In particular, we first give a regret bound for the general case where the regret is expressed as a function of the ``richness'' of the model class $\mathcal{P}$. This analysis is based on the Eluder dimension of an appropriately defined function class and its metric entropy at an appropriate scale. 
It is worth noting that the regret bound does not depend on either the size of the state \emph{or} the size of the action space. 
To illustrate the strength of this general technique, we specialize the regret bound for the case of linear mixture models,

for which we prove that the expected cumulative regret is at most $O(d \sqrt{H^{3}T})$, where $H$ is the episode length, $d$ is the number of model parameters and $T$ is the total number of steps that the RL algorithm interacts with its environment.
To complement the upper bound, for the linear case we also provide a regret lower bound $\Omega( \sqrt{HdT})$ by adapting a lower bound that has been derived earlier for tabular RL. 
After these results, we discuss the connection of our work to prior art (Section~\ref{sec:relatedwork}).
This is followed by the presentation of our empirical results (Section~\ref{sec:experiment}), where, as it was alluded to earlier, the aim is to explore how the various parts of the algorithm interact with each other.
Section~\ref{sec:conc} concludes the paper.

\section{Problem Formulation}
\label{sec:model}
We study episodic Markov decision processes (MDPs, for short), described by a tuple $M=(\mS, \mA, P, r, H,s_\circ)$. 
Here, $\mS$ is the state space, $\mA$ is the action space, $P$ is the transition kernel, $r$ is a reward function, $H>0$ is the episode length, or horizon, and $s_\circ\in \mS$ is the initial state.
In the online RL problem, the learning agent is given $\mS$, $\mA$, $H$ and $r$ but does not know $P$.%
\footnote{Our results are easy to extend to the case when $r$ is not known.} 
The agent interacts with its environment described by $M$ in episodes.
Each episode begins at state $s_\circ$ and ends after the agent made $H$ decisions.
At state $s\in\mS$, the agent, after observing the state $s$, 
can choose an action $a\in\mA$. As a result, the immediate reward $r(s, a)$ is incurred.
Then the process transitions to a random next state $s'\in \mS$ according to the transition law $P(\cdot|s, a)$.%
\footnote{
The precise definitions require measure-theoretic concepts \citep{BeSh78}, i.e., $P$ is a Markov kernel, mapping from $\mS \times \mA$ to distributions over $\mS$, hence, all these spaces need to be properly equipped with a measurability structure. For the sake of readability and also because they are well understood, we omit these technical details.
}

The agent's goal is to maximize the total expected reward received over time.

If $P$ is known, the behavior that achieves this over any number of episodes can be described by applying a deterministic policy $\pi$. Such a policy is a mapping from $\mS\times [H]$ into $\mA$, where we use the convention that for a natural number $n$, $[n] = \{1,\dots,n\}$.
Following the policy means that the agent upon encountering state $s$ in stage $h$ will choose action $\pi(s,h)$.
In what follows, we will use $\pi_h(s)$ as an alternate notation, as this makes some of the formulae more readable.
We will also follow this convention when it comes to other functions whose domain is $\mS\times [H]$.
We will find it convenient to move the stage $h$ into the index. In particular, for policies, we will also write $\pi_h(s)$ for $\pi(s,h)$ but we will use the same convention for other similar objects, like the value function, defined next. 

The value function $V^\pi: \mS\times [H]\to \RR$ 
of a policy $\pi$ is defined via
\begin{equation*}
	V_{h}^{\pi}(s) = \mathbb{E}_{\pi}\left[\sum_{i=h}^{H}r(s_{i}, \pi(s_{i})) \,\big|\, s_{h} = s\right],\qquad  s\in \mS\,,
\end{equation*}
where the subscript $\pi$ (which we will often suppress) signifies that the probabilities underlying the expectation
are governed by $\pi$.
An optimal policy $\pi^{*}$ and  the optimal value function $V^{*}$ are defined to be a policy and the value function such that $V_{h}^{\pi}(s)$ achieves the maximum among all possible policies for any $s\in\mS$ and $h\in [H]$. 
As noted above, there is no loss of generality in restricting the search of optimal policies to deterministic policies.

In online RL, a good agent of course uses all past observations to come up with its decisions.
The performance of such an agent is measured by its regret, which is the total reward the agent misses because they did not follow the optimal policy from the beginning. In particular, the total expected regret of an agent $\mathcal{A}$ across $K$ episodes is given by
\begin{align}
	R(T) = \sum_{k = 1}^{K}\left(V_{1}^{*}(s_{1}^{k}) - \sum^H_{h=1 }r(s_{h}^{k}, a_h^k) \right),
	\label{eq:regretdef}
\end{align}
where $T=KH$ is the total number of time steps that the agent interacts with its environment, 
$s_{1}^{k}=s_\circ$ is the initial state at the start of the $k$-th episode, 
and $s_1^1,a_1^1,\ldots,s_H^k, a_H^k, \ldots,s_1^K, a_1^K,\ldots,s_H^K, a_H^K$ are the $T=KH$ state-action pairs in the order that they are encountered by the agent.
The regret is sublinear of $R(T)/T \to 0$ as $T\to\infty$. 
For a fixed $T$, let $R^*(T)$ denote the worst-case regret.
As is well known, no matter the algorithm used, $R^*(T)$, grows at least as fast as $\sqrt{T}$ \citep[e.g.,][]{jaksch2010near}.%

In this paper, we aim to design a general model-based reinforcement learning algorithm, with a guaranteed sublinear regret, for any given family of transition models.
\begin{assumption}[Known Transition Model Family]\label{assump1}
The unknown transition model $P$ belongs to a family of models $\mathcal{P}$ which is available to the learning agent.
The elements of $\mathcal{P}$ are transition kernels mapping state-action pairs to signed distributions over $\mS$.

\end{assumption}
That we allow signed distributions increases the generality; this may be important when one is given a model class that can be compactly represented but only when it also includes non-probability kernels (see \citealt{PiSze16:FLM} for a discussion of this).

An important special case is the class of linear mixture models:
\begin{definition}[Linear Mixture Models]\label{def:lpm}
We say that $\mathcal{P}$ is the class of linear mixture models with component models
$P_1,\dots,P_d$ if $P_1,\dots,P_d$ are transition kernels that map state-action pairs to signed measures
and $P\in \mathcal{P}$ if and only if there exists $\theta\in \R^d$ such that
\begin{equation}\label{eq-model-linear}
	P(ds'|s, a) = \sum_{j=1}^{d} \theta_{j} P_{j}(ds'|s, a) = P_{\bigcdot}(ds'|s, a)^{\top}\theta_{*},
\end{equation}
for all $(s,a)\in \mS \times \mA$.

\end{definition}

\def\rank{\text{dim}{$\Omega$}}

Parametric and nonparametric transition models are common in modelling complex stochastic controlled systems. 
For one example, robotic systems are often smoothly parameterized by unknown mechanical parameters such as friction, or just parameters that describe the geometry of the robot. 

The linear mixture model can be viewed as a way of aggregating 
a number of known basis models as considered by
\citet{modi2019sample}.
 We can view each $P_{j}(\cdot|\cdot)$ as a basis latent ``mode" and the actual transition is a probabilistic mixture of these latent modes. 
For another example, consider large-scale queueing networks where the arrival rate and job processing speed for each queue is not known. By using a discrete-time Bernoulli approximation, the transition probability matrix from time $t$ to $t+\Delta t$ becomes increasingly close to linear with respect to the unknown arrival/processing rates as $\Delta t\to 0$. In this case, it is common to model the discrete-time state transition as a linear aggregation of arrival/processing processes with unknown parameters \cite{kovalenko1968introduction}. 

Another interesting special case is the linear-factored MDP model 
of \citet{yang2019reinforcement}
where, assuming a discrete state space for a moment, $P$ takes the form
\begin{equation*}
	\begin{aligned}
		P(s'|s, a) & = \phi(s, a)^{\top}M\psi(s')= \sum_{i=1}^{d_1}\sum_{j=1}^{d_2}M_{ij}\left[\psi_{j}(s')\phi_{i}(s, a)\right],
	\end{aligned}
\end{equation*}
where $\phi(s, a)\in\mathbb{R}^{d_1}, \psi(s')\in\mathbb{R}^{d_2}$ are given features for every $s, s'\in\mS$ and $a\in\mA$
(when the state space is continuous,  $\psi$ becomes an $\mathbb{R}^{d_2}$-valued measure over $\mS$).
The matrix $M\in\mathbb{R}^{d_1\times d_2}$ is an unknown matrix and is to be learned. 
It is easy to see that the factored MDP model is a special case of the linear mixture model \eqref{eq-model-linear} with each $\psi_{j}(s')\phi_{i}(s, a)$ being a basis model (this should be replaced by $\psi_j(ds')\phi_i(s,a)$ when the state space is continuous). 
In this case, the number of unknown parameters in the transition model is $d = d_1\times d_2$. In this setting, without any additional assumption, our regret bound matches the result of \cite{yang2019reinforcement}.

\section{Upper Confidence RL with Value-Targeted Model Regression}\label{section 3}
\label{sec:alg}
\begin{algorithm}[t]
	\caption{UCRL-VTR}
	\label{alg}
	\begin{algorithmic}[1]
		\STATE \textbf{Input: } Family of MDP models $\mathcal{P}$, $d, H, T=KH$;
		\STATE \textbf{Initialize: } pick the sequence $\{\beta_k\} $ as in Eq.~\eqref{eq:bkdef} of Theorem \ref{thm:mainbound}

		\STATE $B_1= \cP$
		\FOR{$k = 1, 2, \dots, K$}
			\STATE Observe the initial state $s_{1}^{k}$ of episode $k$
			\STATE {\bf Optimistic planning:}
			\begin{align*}
			&P_k = \argmax_{P'\in B_k} V^*_{P',1}(s_{1}^{k})\\
			&\hbox{Compute $Q_{1,k},\ldots Q_{H,k}$ for $P_k $ using \eqref{eq-q}}
			\end{align*}
			\FOR{$h = 1, 2,\dots,H$} 
				\STATE Choose the next action greedily with respect to $Q_{h,k}$:
				\[a_{h}^{k} = \arg\max_{a\in\mA}Q_{h, k}(s_{h}^{k}, a)\]
				\STATE Observe state $s_{h+1}^k$
				 \STATE Compute and store value predictions:  
				$y_{h,k} \leftarrow V_{h+1, k}(s_{h+1}^{k})$

			\ENDFOR

			\STATE {\bf Construct confidence set using value-targeted regression as described in Section~\ref{eq-Bk}}:
			\[
			B_{k+1}= \{P'\in\mathcal{P} | L_{k+1}(P',\hat P_{k+1})\leq \beta_{k} \}
			\]
		\ENDFOR
\end{algorithmic}
\end{algorithm}

Our algorithm can be viewed as a generalization of UCRL \citep{jaksch2010near},
following ideas of \citet{osband2014model}.

In particular, at the beginning of episode $k=1,2,\dots,K$, 
the algorithm first computes a subset $B_k$ of the model class $\mathcal{P}$ that contains 
the set of models that are deemed to be consistent with all the data that has been collected in the past.
The new idea, value-targeted regression is used in the construction of $B_k$. The details of how this is done are postponed to a later section.

Next, the algorithm needs to find the model that maximizes the optimal value, and the corresponding optimal policy. 
Denoting by $V^*_{P}$ the optimal value function under a model $P$, this amounts to finding the model $P\in B_k$ that maximizes the value $V^*_{P,1}(s_1^k)$.
Given the model $P_k$ that maximizes this value, an optimal policy is extracted from the model as described in the next section (this is standard dynamic programming).
At the end of the episode, the data collected is used to refine the confidence set $B_k$.
The pseudocode of the algorithm can be found in Algorithm \ref{alg}. 

\subsection{Model-Based Optimistic Planning}
Upper confidence methods are prominent in online learning. In our algorithm, we will maintain a confidence set $B_k$ for the estimated transition model and use it for optimistic planning:
\[
P_k = \argmax_{P' \in B_k}  V^*_{P',1}(s_1^k)
\]
where $s_1^k$ is the initial state at the beginning of episode $k$
and we use $V^*_{P',1}$ to denote the optimal value function of stage one, when the transition model is $P'$.
Given model $P_k$, the optimal policy for $P_k$ can be computed using dynamic programming.
In particular, for $1\le h\le H+1$, define

\begin{equation}\label{eq-q}
	\begin{aligned}
		& Q_{H+1, k}(s, a) = 0,\\
		& V_{h, k}(s) =\max_{a\in\mA}Q_{h, k}(s, a),\\
		& Q_{h, k}(s, a) = r(s, a) + \langle P_k(\cdot|s, a), V_{h+1, k} \rangle,
	\end{aligned}
\end{equation}
where we with a measure $\mu$ and function $f$ over the same domain, 
we use $\langle \mu,f \rangle$ to denote the integral of $f$ with respect to $\mu$.
Then, taking the action at stage $h$ and state $s$ that maximizes $Q_{h,k}(s,\cdot)$ gives an optimal policy for model $P_k$.
As long as $P\in B_{k}$ with high probability, the preceding calculation gives an optimistic (that is, upper) 
estimate of value of an episode.
Next we show how to construct the confidence set $B_{k}$.

\subsection{Value-Targeted Regression for Confidence Set Construction}

 Every time we observe a transition $(s,a,s')$ with $s'\sim P(\cdot|s,a)$, 
 we receive information about the model $P$. 
 Instead of regression onto fixed target like probabilities or raw states, 
 we will refresh the model estimate by regression using the estimated value functions as target.

This leads to the model 
\begin{align}
\label{eq:vtrreg}
		\hat P_{k+1} &= \hbox{argmin}_{P'\in\mathcal{P}} \sum^k_{k'=1}\sum^{H}_{h=1} \quad\left( \ip{P'(\cdot|s_h^{k'},a_h^{k'}),V_{h+1,k'}} - y_{h,k'}\right)^{2}\,, \quad \text{where}\\
		& y_{h,k'} = V_{h+1,k'}(s_{h+1}^{k'})\,, \quad h\in [H], k'\in [k]\,. \nonumber
\end{align}

In the above regression procedure, the regret target keeps changing as the algorithm constructs increasingly accurate value estimates. This is in contrast to typical supervised learning for building models, where the regression targets are often fixed objects (such as raw observations, features or keypoints; e.g. \cite{jaksch2010near,osband2014model,abbasi2015bayesian,XiPaMo16,agrawal2017optimistic,yang2019reinforcement,KaBa19}).

For a confidence set construction, we get inspiration from  Proposition 5 in the paper of \citet{osband2014model}. The set is centered at $\hat P_{k+1}$.%
Define
\[
L_{k+1}(P,\hat P_{k+1} ) = \sum^k_{k'=1}\sum^{H}_{h=1} 
\quad\left( \ip{P
(\cdot|s_h^{k'},a_h^{k'})-\hat P_{k+1}(\cdot|s_h^{k'},a_h^{k'}), V_{h+1,k'}} 
\right)^{2}\,.
\]
Then we let
\begin{equation*}
\label{eq-Bk}
 B_{k+1}= \{P'\in\mathcal{P} \mid L_{k+1}(P',\hat P_{k+1})\leq \beta_{k+1} \}
\end{equation*}
and the value of $\beta_k$ can be obtained using a calculation similar to that done in Proposition~5 of the paper of
\citet{osband2014model}, which is based on the nonlinear least-squares confidence set construction from \citet{RuVR14}, which we describe in the appendix.

It is not hard to see that the confidence set can also be written in the alternative form
\begin{equation*}
 B_{k+1}= \{P'\in\mathcal{P} \mid \tilde{L}_{k+1}(P')\leq \tilde{\beta}_{k+1} \}
\end{equation*}
with a suitably defined $\tilde{\beta}_{k+1}$ and
where
\[
\tilde{L}_{k+1}(P') = 
\sum^k_{k'=1}\sum^{H}_{h=1} 
\quad\left( \ip{P'(\cdot|s_h^{k'},a_h^{k'}), V_{h+1,k'}} - y_{h,k'}\right)^2\,.
\]
Note that the above formulation strongly exploits that the MDP is time-homogeneous: The same transition model is used at all stages of an episode. When the MDP is time-inhomogeneous, the construction can be easily modified to accommodate that the transition kernel may depend on the stage index. 

\subsection{Implementation of UCRL-VTR}

Algorithm \ref{alg} gives a general and modular template for model-based RL 
that is compatible with regression methods/optimistic planners. 
While the algorithms is conceptually simple, and the optimization and evaluation of the loss
in value-targeted regression appears to be at advantage in terms of computation as compared to standard losses typically used in model-based RL, the implementation of UCRL-VTR is nontrivial in general and for now it requires a case-by-base design.

Computation efficiency of the algorithm depends on the specific family of models chosen. For the linear-factor MDP model considered by \citet{yang2019reinforcement}, the regression is linear and admits efficient implementation; further, optimistic planning for this model can be implemented in $\text{poly}(d)$ time by using Monte-Carlo simulation and sketching as argued in the cited paper. 
Other ideas include loosening the confidence set to come up with computationally tractable methods,
or relaxing the requirement that the same model is used in all stages. 

In the general case, optimistic planning is computationally intractable.
However, we expect that randomized (eg \cite{osband2017deep,osband2014generalization,lu2017ensemble}) and approximate dynamic programming methods (tree search, roll out, see eg \cite{bertsekas1996neuro}) will often lead to tractable and good approximations. As was mentioned above, in some special cases these have been rigorously shown to work. In similar settings, the approximation errors are known to mildly impact the regret \cite{abbasi2015bayesian} and we expect the same will hold in our setting.

If we look beyond methods with rigorous guarantees, there are practical deep RL algorithms that implement parts of UCRL-VTR. As mentioned earlier, the Muzero algorithm of \citet{schrittwieser2019mastering} is a state-of-the-art algorithm on the Atari domain and this algorithm implements both value-targeted-regression to learn a model and Monte Carlo tree search for planning based on the learned model, although it does not incorporate optimistic planning.

\section{Theoretical Analysis}
\label{sec:results}

\def\cG{\mathcal{G}}

We will need the concept of Eluder dimension.
Let $\cF$ be a set of real-valued functions with domain $\cX$. To measure the complexity of interactively identify an element of $\cF$, \citet{RuVR14} defines the {\it Eluder dimension of $\cF$ at scale $\epsilon>0$}.
For $f\in \cF$, $x_1,\dots,x_t\in \cX$, introduce the notation $f|_{(x_1,\dots,x_t)} = (f(x_1),\dots,f(x_t))$. We say that $x\in \cX$ is $\epsilon$-independent of $x_1,\dots,x_t\in \cX$ given $\cF$
if there exists
$f,f'\in \cF$ such that
$\| (f - f')|_{(x_1,\dots,x_t)} \|_2 \le \epsilon$ while $f(x)-f'(x)> \epsilon$.

\begin{definition}[Eluder dimension \citet{RuVR14}]\rm
The Eluder dimension $\dimE(\cF,\epsilon)$ of $\cF$ at scale $\epsilon$ is the length of the longest sequence $(x_1,\dots,x_n)$ in $\cX$ such that for some $\epsilon'\ge \epsilon$, for any $2\le t\le n$, $x_t$ is $\epsilon'$-independent of $(x_1,\dots,x_{t-1})$ given $\cF$.
\end{definition}

Let $\cV$ be the set of optimal value functions under some model in $\cP$: $\cV = \{ V^*_{P'}\,:\, P' \in \cP \}$. 
Note that $\cV \subset \cB(\cS,H)$, where 
$\cB(\cS,H)$ denotes the set of real-valued measurable functions with domain $\cS$ that are bounded by $H$. 
We let $\cX = \cS \times \cA \times \cV$. 
In order to analyze the confidence of nonlinear regression, choose $\cF$ to be the collection of functions $f: \cX \to \R$ as
\begin{align}
\cF = \left\{ f\,\, \big|\,\,  \exists P\in \mathcal{P} \text{ s.t. for any } (s,a,v)\in \cS\times\cA\times\cV, \,\, f(s,a,v) = \int  P(ds'|s,a) v(s') \right\}\,.
\label{eq:inducedF}
\end{align}
Note that $\cF \subset \cB(\cX,H)$.
For a norm $\norm{\cdot}$ on $\cF$ and
 $\alpha>0$ let $\cN(\cF,\alpha, \norm{\cdot})$ denote the $(\alpha, \norm{\cdot})$-covering number of $\cF$.
 That is, this if $m=\cN(\cF,\alpha, \norm{\cdot})$ then one can find $m$ points of $\cF$ such that any point in $\cF$ is at most $\alpha$ away from one of these points in norm $\norm{\cdot}$.
Denote by $\norm{\cdot}_\infty$ the supremum norm: $\norm{f}_\infty = \sup_{x\in \cX} |f(x)|$.

Now we analyze the regret of UCRL-VTR. Define the $K$-episode {\it pseudo-regret} as
\[
R_K = \sum^K_{k=1} \left( V^*(s_0^k) - V^{\pi_k}(s_0^k) \right)\,.
\]
Clearly, $R(KH) = \E{ R_K}$ holds for any $K>0$ where $R(T)$ is the expected regret after $T$ steps of interaction as defined in \eqref{eq:regretdef}.
Thus, to study the expected regret, it suffices to study $R_K$.

Our  main result is as follows. 

\begin{theorem}[Regret of Algorithm \ref{alg}]
\label{thm:mainbound} Let Assumption \ref{assump1} hold and let $\alpha\in(0,1).$
For $k\in [K]$, let $\beta_k$ be
\begin{align} \label{eq:bkdef}
\beta_k &= 2 H^2 \log\left(\frac{2 \cN(\cF,\alpha, \|\cdot\|_{\infty})}{\delta}\right) +  2H (kH-1) \alpha \left\{2+\sqrt{ \log\left(\frac{4 kH(kH-1)}{\delta}\right)}\right\}\,.
\end{align}

Then, with probability $1-2\delta$,
\begin{align*}
R_K \le
\alpha + &H(d \wedge K(H-1)) + 4 \sqrt{d \beta_K K(H-1)} +H \sqrt{2K(H-1)\log(1/\delta)}\,,
\end{align*}
where $d = \dimE(\cF,\alpha)$ is the Eluder dimension with $\cF$ given by \eqref{eq:inducedF}.
\end{theorem}

A typical choice of $\alpha$ is $\alpha = 1/(KH)$. 
In the special case of linear transition model, Theorem \ref{thm:mainbound} implies a worst-case regret bound that depends linearly on the number of parameters. 

\begin{corollary}[Regret of Algorithm \ref{alg} for Linearly-Parametrized Transition Model]
\label{cor:linmix}
Let $P_1,\dots,P_d$ be $d$ transition models, 
$ \Theta\subset \R^d$ a nonempty set with diameter $R$ measured in $\norm{\cdot}_1$ and let
$\cP = \{  \sum_j \theta_j P_j \,:\, \theta \in  \Theta\}$.
Then, for any $0<\delta<1$, with probability at least $1-\delta$, the pseudo-regret $R_K$ of Algorithm \ref{alg} when it uses the confidence sets given in Theorem \ref{alg} satisfies
\begin{align*}
R_K = \tilde O(d \sqrt{H^3 K \log(1/\delta)})\,.
\end{align*}
\end{corollary}

\par We also provide a lower bound for the regret in our model. The proof is by reduction to a known lower bound and is left to Appendix \ref{proof-thm2}. 
\begin{theorem}[Regret Lower Bound]\label{thm2}
	For any $H\ge 1$ and $d\ge 8$, 
	there exist a state space $\cS$ and action set $\cA$, a reward function $r:\cS \times \cA \to [0,1]$,
	$d$ transition models $P_1,\dots,P_d$ and a set $\Theta$ of diameter at most one
	such that for any algorithm
	there exists $\theta\in \Theta$  such that
	for sufficiently large number of episodes $K$, the expected regret of the algorithm on the $H$-horizon
	MDP with reward $r$ and transition model $P = \sum_j \theta_j P_j$ is 
	at least $\Omega(H\sqrt{dK})$.

\end{theorem}
\cite{rusmevichientong2010linearly} gave a regret lower bound of $\Omega(d\sqrt{T})$ for linearly parameterized bandit with actions on the unit sphere (see also Section~24.2 of \citet{lattimore2018bandit}). 
Our regret upper bound matches this bandit lower bound in $d,T$.
Whether the upper or lower bound is tight (or none of them) remains to be seen.

The theorems validate that, in the setting we consider, it is sufficient to use the predicted value functions as regression targets. That  for the special case of linear mixture models the lower bound is close to the upper bound
appears to suggest that little benefit if any can be derived from fitting the transition model to predict well future observations. We conjecture that this is in fact true when considering the worst-case regret. 
Of course, a conclusion that is concerned with the worst-case regret has no implication for the behavior of the respective methods on particular MDP instances.

We note in passing that by appropriately increasing $\beta_k$, the regret upper bounds can be extended to the so-called misspecified case when $P$ can be outside of $\cP$ (for related results, see, e.g.,
\citealt{jin2019provably,LaSze19:GoodFeature}). However, the details of this are left for future work.

\section{Related Work}
\label{sec:relatedwork}
A number of prior efforts have established efficient RL methods with provable regret bounds. 
For tabular MDPs with $S$ states and $A$ actions, building on the pioneering work of \citet{jaksch2010near} who studied the technically more challenging continuing setting, 
a number of works obtained results for the episodic setting,
both with model-based
(e.g., \citealt{osband2014generalization, azar2017minimax, dann2017unifying, dann2018policy, agrawal2017optimistic}), and model-free methods (e.g., \citealt{jin2018q,Ru19,ZhZhJi20}), both for the time-homogeneous case (i.e., the same transition kernel governs the dynamics in all stages of the $H$-horizon episode) and the time-inhomogeneous case.
Results developed for the time-inhomogeneous case apply to the time-inhomogeneous case and since in this case the number of free parameters to learn is at least $H$ times larger than for the time-homogeneous case, the regret bounds are expected to be $\sqrt{H}$ larger.

As far as regret lower bounds are concerned, \cite{jaksch2010near} established a worst-case regret lower bound 
of $\Omega(\sqrt{DSAT})$
for the continuing case for MDPs with diameter bounded by $D$ (see also Chapter~38 of \citealt{lattimore2018bandit}).
This lower bound can be adapted to the episodic by setting  $D=H$.
This way one obtains a lower bound of $\Omega(\sqrt{HSAT})$ for the homogeneous, and $\Omega(\sqrt{H^2 SAT})$ for the inhomogeneous case (because here the state space size is effectively $HS$).
This lower bound, up to lower order terms, is matched by upper bounds both for the time-homogeneous case
 \citep{azar2017minimax,kakade2018variance} 
and the time inhomogeneous case \citep{dann2018policy,ZhZhJi20}. 
Except the work of \citet{ZhZhJi20}, these results are achieved by algorithms that estimate models.
With a routine adjustment, the near-optimal model-based algorithms available for the homogeneous case 
are also expected to deliver near-optimal worst-case regret growth in the inhomogeneous case.
A further variation is obtained by considering different scalings of the reward \citep{WaDuYaKa20}.

Moving beyond tabular MDP, there have been significant theoretical and empirical advances on RL with function approximation, including but not limited to \cite{baird1995residual, tsitsiklis1997analysis, parr2008analysis, mnih2013playing, mnih2015human, silver2017mastering, yang2019sample, bradtke1996linear}.
Among these works, many papers aim to uncover algorithms that are provably efficient.
Under the assumption that the optimal action-value function is captured by linear features,
\cite{ZLK19} considers the case when the features are ``extrapolation friendly'' and a simulation oracle is available,
\cite{wen2013efficient, wen2017efficient} tackle problems where the transition model is deterministic, \cite{du2019provably} deals with a relaxation of the deterministic case when the transition model has low variance.
\cite{yang2019sample} considers the case of linear factor models,
while \cite{LaSze19:GoodFeature}
considers the case when all the action-value functions of all deterministic policies
are well-approximated using a linear function approximator.
These latter works handle problems when the algorithm has access to a simulation oracle of the MDP.
As for regret minimization in RL using linear function approximation, \cite{yang2019reinforcement} assumed the transition model admits a matrix embedding of the form $P(s'|s, a) = \phi(s, a)^{\top}M\psi(s')$, and proposed a model-based MatrixRL method with regret bounds $\tilde{\mathcal{O}}(H^{2}d\sqrt{T})$ with stronger assumptions and $\tilde{\mathcal{O}}(H^{2}d^2\sqrt{T})$ in general, where $d$ is the dimension of state representation $\phi(s, a)$.

\cite{jin2019provably} studied the setting of linear MDPs and constructed a model-free least-squares action-value iteration algorithm, which was proved to achieve the regret bound $\tilde{\mathcal{O}}(\sqrt{H^{3}d^{3}T})$. \citep{modi2019sample} considered a related setting where the transition model is an ensemble involving state-action-dependent features and basis models and proved a sample complexity $\frac{d^3K^2H^2}{\epsilon^2}$ where $d$ is the feature dimension, $K$ is the number of basis models and $d\cdot K$ is their total model complexity.
Very recently, \cite{wang2020provably} propose an model-free algorithm for general reward function approximation and show that the learning complexity of the function class can be bounded by the eluder dimension, which is similar to our model-based setting.

As for RL with a general model class, the seminal paper \cite{osband2014model} provided a general posterior sampling RL method that works for any given classes of reward and transition functions. It established a Bayesian regret upper bound $O(\sqrt{d_Kd_ET})$, where $d_K$ and $d_E$ are the Kolmogorov and the Eluder dimensions of the model class. In the case of linearly parametrized transition model (Assumption 2 of this paper), this Bayesian regret becomes $O(d\sqrt{T})$, and our worst-case regret result matches with the Bayesian one. \cite{abbasi2015bayesian,GeZhAYVl17} also considered the Bayesian regret and in particular \citet{abbasi2015bayesian} considered a smooth parameterization with a somewhat unusual definition of smoothness.
To the authors' best knowledge, there are no prior works addressing the problem of designing low-regret algorithms for MDPs with a general model family. In particular, while \cite{osband2014model} sketch the main ideas of an optimistic model-based optimistic algorithm for a general model class, they left out the details. When the details are filled based on their approach for the Bayesian case, unlike in the present work, the confidence sets would be constructed by losses that measure how well the model predict future observations and not by the value-targeted regression loss studied here.
A preliminary version of the present paper appeared at L4DC 2020, which included results for the linear transition model only.

\section{Numerical Experiments}
\label{sec:experiment}
\emph{The goal of our experiments is to provide insight into 
the benefits and/or pitfalls of using value-targets for fitting models, both with and without optimistic planning.}
We run our experiments in the tabular setting as in this setup it is easy to keep all the aspects of the test environments under control and the tabular setting also lets us avoid approximate computations.
Note that tabular environments are a special case of the linear model where $P_j(s'|s,a) = \mathbb{I}(j=f(s,a,s'))$, where $j\in  [S^2A]$ and $f$ is a bijection that maps its arguments to the set $[S^2A]$. Thus, $d = S^2 A$ in this case.

The algorithms that we compare have a model-fitting objective which is either used to fit a nominal model or to calculate confidence sets.
The objective is either to minimize mean-squared error of predicting next states (alternatively, maximize log-likelihood of observed data), which leads to standard frequency based model estimates,
or it is based on minimizing the value targets as proposed in our paper.
The other component of the algorithms is whether they implement optimistic planning, or planning with the nominal model and then implementing an $\epsilon$-greedy policy with respect to the estimated model (``dithering'').
We also consider mixing value targets and next state targets.
In the case of optimistic planning, the algorithm that uses mixed targets uses a union bound and takes the smallest value upper confidence bounds amongst the two bounds obtained with the two model-estimation methods.
These leads to six algorithms, as shown in Table~\ref{tab:expalgs}. Results for the ``mixed'' variants are very similar to the variant that uses value-targeted regression. As such, the results for the ``mixed'' variants are shown in the appendix only, as they would otherwise make the graphs overly cluttered.%

\begin{table}[]
\begin{center}
\begin{tabular}{|
>{\columncolor[HTML]{EFEFEF}}l |l|l|}
\hline
\textbf{\begin{tabular}[c]{@{}l@{}}Exploration/\\ Targets\end{tabular}} & \cellcolor[HTML]{EFEFEF}\textbf{Optimism} & \cellcolor[HTML]{EFEFEF}\textbf{Dithering} \\ \hline
\textbf{Next states}                                                    & UC-MatrixRL                             & EG-Freq                                    \\ \hline
\textbf{Values}                                                         & UCRL-VTR                                  & EG-VTR                                     \\ \hline
\textbf{Mixed}                                                          & UCRL-Mixed                                & EG-Mixed                                   \\ \hline
\end{tabular}
\caption{Legend to the algorithms compared.
Note that UC-MatrixRL of \citet{yang2019reinforcement} in the tabular case essentially becomes UCRL of \citet{jaksch2010near}.
The mixed targets use both targets.
}
\label{tab:expalgs}
\end{center}
\end{table}
In the experiments we use confidence bounds that are specialized to the linear case. 
For the details, see Appendix \ref{sec:implement}.
For $\epsilon$-greedy, we optimize the value of $\epsilon$ in each environment to get the best results. This gives $\epsilon$-greedy an unfair advantage; but as we shall see despite this advantage, $\epsilon$-greedy will not fair particularly well in our experiments.

\subsection{Environments}
We compare these algorithms on the episodic RiverSwim environment due to  \cite{RiverSwim} 
and a novel finite horizon MDP we tentatively call WideTree. 
The RiverSwim environment, whose detailed description is given below in Section~\ref{sec:riverswim},
 is chosen because it is known that in this environment
 ``dithering'' type exploration methods (e.g., $\epsilon$-greedy) are ineffective.
We vary the number of states in the RiverSwim environment in order to highlight some of the advantages and disadvantages of Value-Targeted Regression.

WideTree is designed in order it highlight the advantages of Value-Targeted Regression when compared with more tradition frequency based methods. In this environment, only one action effects the outcome thus the other actions are non-informative.
The detailed description of WideTree is given in Section~\ref{sec:WideTree}.

\subsection{Measurements}
We report the cumulative regret as a function of the number of episodes and the weighted model error to indicate how well the model is learned. The results are obtained from $30$ independent runs for the $\epsilon$-greedy algorithms and $10$ independent runs for the UC algorithms,
The weighted model error reported is as follows.
Given the model estimate $\hat P$, its weighted error is
\begin{equation}\label{eq:weightedL1}
E(\hat P) = \sum_{s,a}  \sum_{s'}\frac{N(s,a,s')}{N(s,a)} |\hat{P}(s'\mid s,a)-P^*(s'\mid s,a)|,
\end{equation}
where $N(s,a)$ is the observation-count of the state-action pair $(s,a)$, $N(s,a,s')$ is the count of transitioning to $s'$ from $(s,a)$, and $P^*(s' \mid s,a)$ is the probability of $s'$ when action $a$ is chosen in state $s$, according to the true model. 
Here, for the algorithms that use value-targeted regression 
the estimated model $\hat{P}$ is the model obtained through Eq.~\eqref{eq:vtrreg}.
The weighting is introduced so that an algorithm that discards a state-action pair is not penalized. 
This is meant to prevent penalizing
good exploration algorithms that may quickly discard some state-action pairs.
We are interested in this error metric to monitor whether UCLR-VTR, which is not forced to model next-state distributions, will learn the proper next state distribution. 
In fact, we will see one example both for the case when this and also when this does not happen.

\subsection{Results for RiverSwim} 
\label{sec:riverswim}
The schematic diagram of the RiverSwim environment is shown in Figure~\ref{fig:riverswim}.
RiverSwim consists of $S$ states arranged in a chain. The agent begins on the far left and has the choice of swimming left or right at each state. There is a current that makes swimming left much easier than swimming right. Swimming left with the current always succeeds in moving the agent left, but swimming right against the current sometimes moves the agent right (with a small probability of moving left as well), but more often than not leaves the agent in the current state. Thus smart exploration is a necessity to learn a good policy in this environment.
\begin{figure}
\begin{center}
\includegraphics[width=0.5\textwidth]{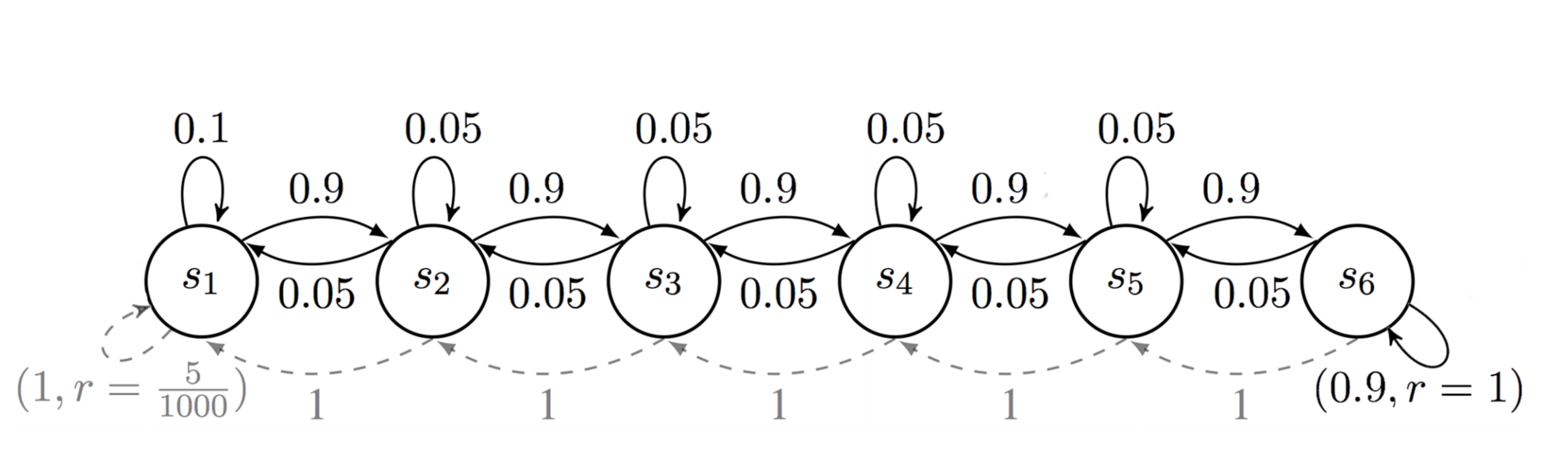}
\end{center}
\caption{The ``RiverSwim'' environment with $6$ states.
State $s_1$ has a small associated reward, state $s_6$ has a large associated reward.
The action whose effect is shown with the dashed arrow deterministically ``moves the agent'' towards state $s_1$.
The other action is stochastic, and with relatively high probability moves the agent towards state $s_6$: This represents swimming ``against the current''. None of these actions incur a reward.
}
\label{fig:riverswim}
\end{figure}
We experiment with small environments with $S\in \{3,4,5\}$ states and set the horizon to $4S$ for each case.
The optimal values of the initial state are $5.72$, $5.66$ and $5.6$, respectively, in these cases. The initial state is the leftmost state ($s_1$ in the diagram). 
The value that we found to work the best for $\epsilon$ greedy is $\epsilon=0.01$.

Results are shown in Figure \ref{fig:RiverSwim}, except for UCRL-Mixed and EG-Mixed, whose results are given in Appendix~\ref{sec:mixture_model}. 
As noted before, the results of these algorithms are very close to those of the VTR-versions, hence, they are not included here.
The columns correspond to environments with $S=3$, $S=4$ and $S=5$, respectively, which are increasingly more challenging.
The first row shows the algorithm's performance measured in terms of their respective cumulative regrets, 
the second row shows results for the weighted model error as defined above.
The regret per episode for an algorithm that ``does not learn'' is expected to be in the same range as the respective optimal values. 
Based on this we see that $10^5$ episodes is barely sufficient for the algorithms other than UCRL-VTR to learn a good policy.
Looking at the model errors we see that EGRL-VTR is doing quite poorly, 
EG-Freq is also lacking (especially on the environment with $5$ states), 
the others are doing reasonably well.
That EG-Freq is not doing well is perhaps surprising. 
However, this is because EG-Freq visits more uniformly than the other methods the 
various state-action pairs. 

The results clearly indicate that 
{\em (i)} fitting to the state-value function alone provides enough of a signal for learning as evident by UCRL-VTR obtaining low regret as predicted by our theoretical results, and that {\em (ii)} optimism is necessary when using value targeted regression to achieve good results,
as evident by UCRL-VTR achieving significantly better regret than EGRL-VTR and even in the smaller RiverSwim environment where EG-Freq performed best.

It is also promising that value-targeted regression with optimistic exploration outperformed optimism based on the ``canonical'' model estimation procedure.
We attribute this to the fact that value-targeted regression will learn a model faster that predicts the optimal values well than the canonical, frequency based approach. 

That value-targeted regression also learns a model with small weighted error appears to be an accidental feature of this environment.
Our next experiments are targeted at  further exploring this effect.

\begin{figure}[htb]
\centering
\begin{subfigure}{.32\textwidth}
  \centering
  \includegraphics[width=1.1\linewidth]{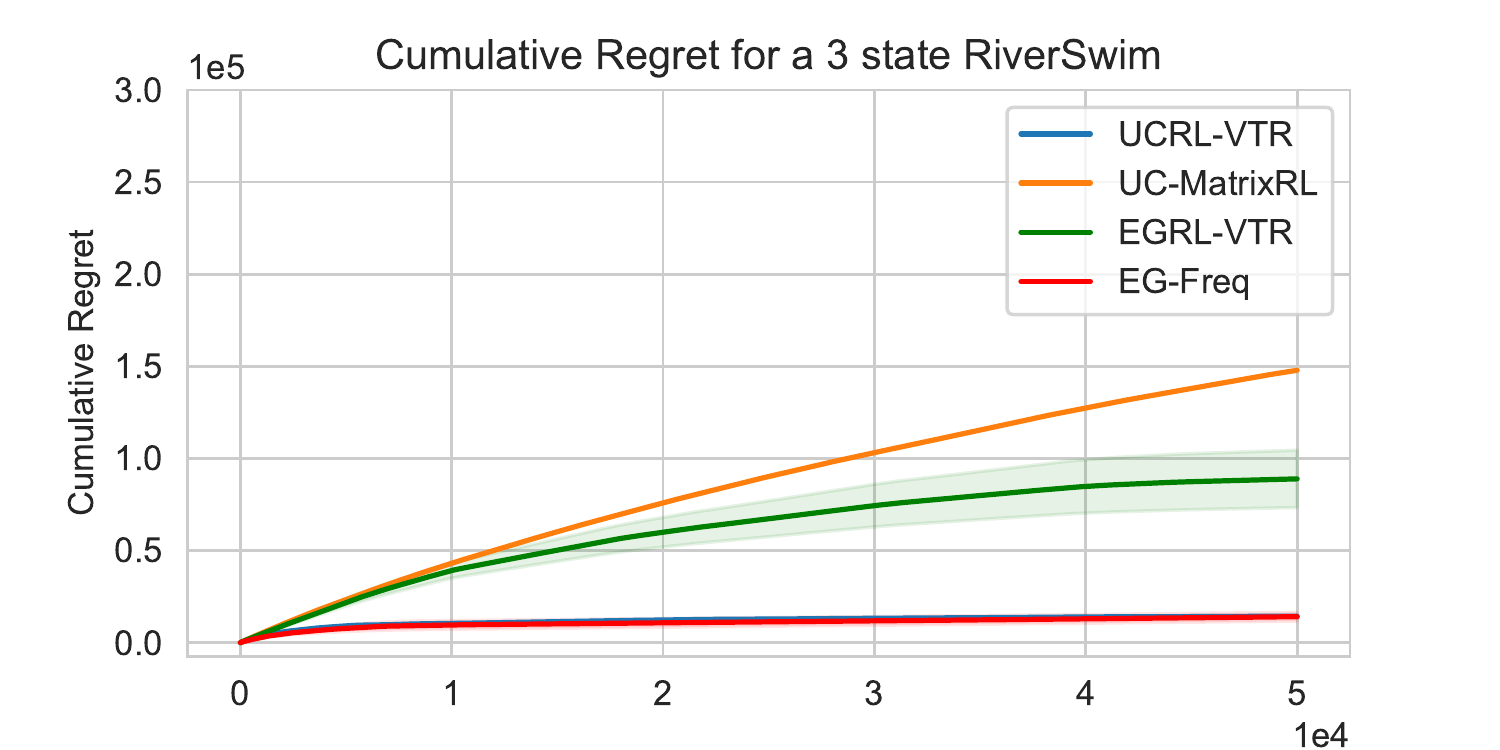}

\end{subfigure}%
\begin{subfigure}{.32\textwidth}
  \centering
  \includegraphics[width=1.1\linewidth]{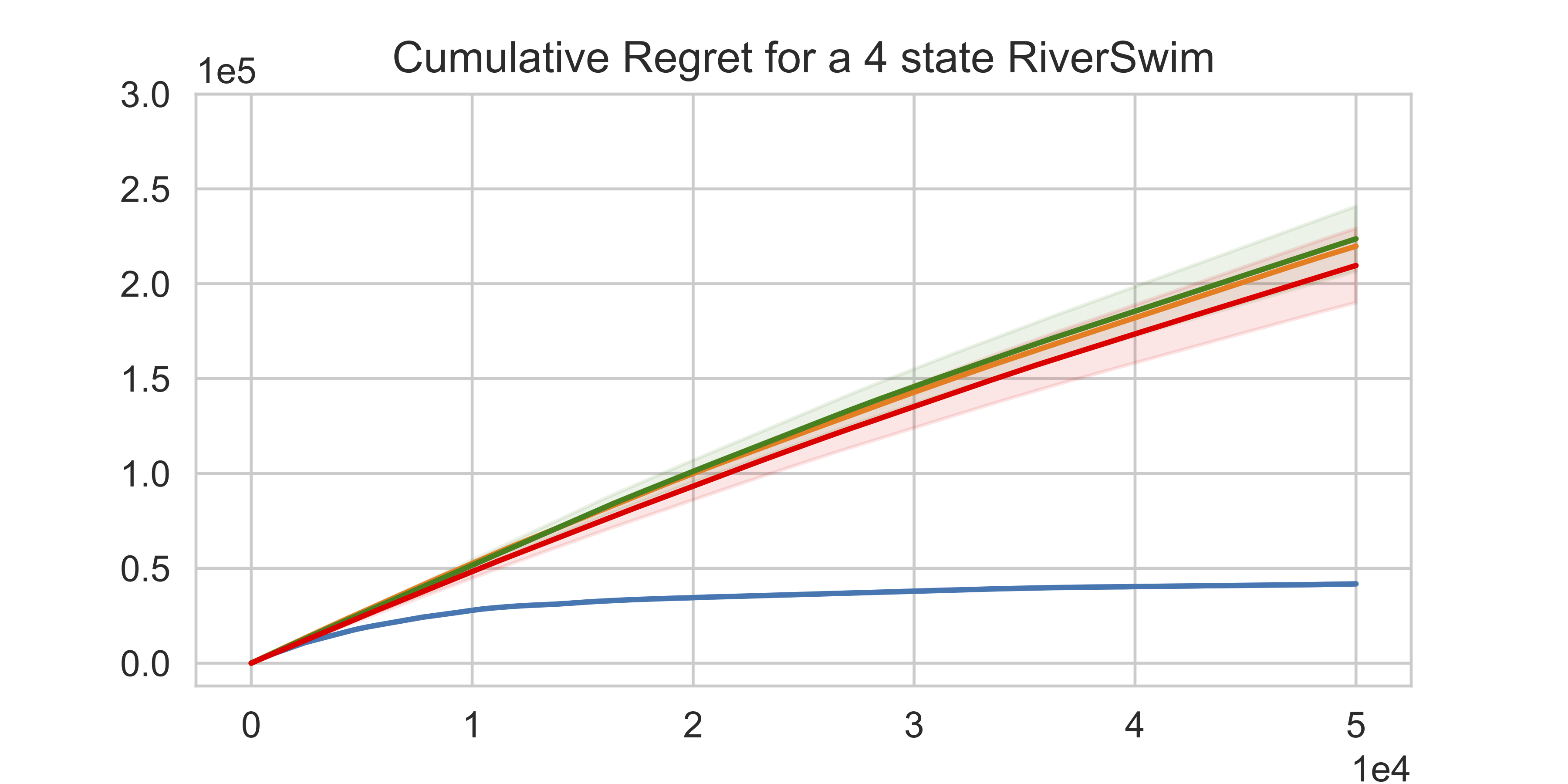}
\end{subfigure}
\begin{subfigure}{.32\textwidth}
  \centering
  \includegraphics[width=1.1\linewidth]{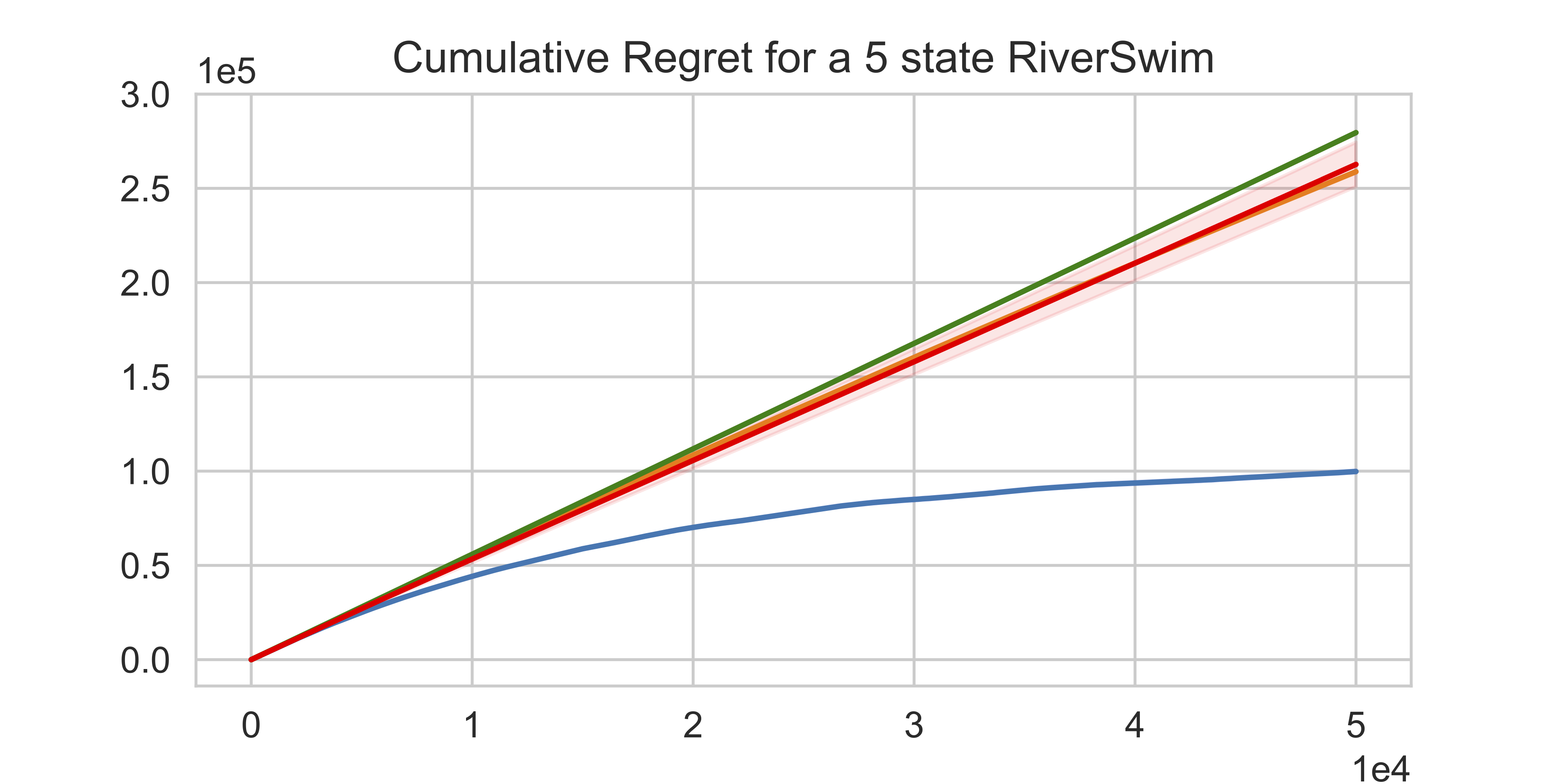}
\end{subfigure}

\begin{subfigure}{.32\textwidth}
  \centering
  \includegraphics[width=1.1\linewidth]{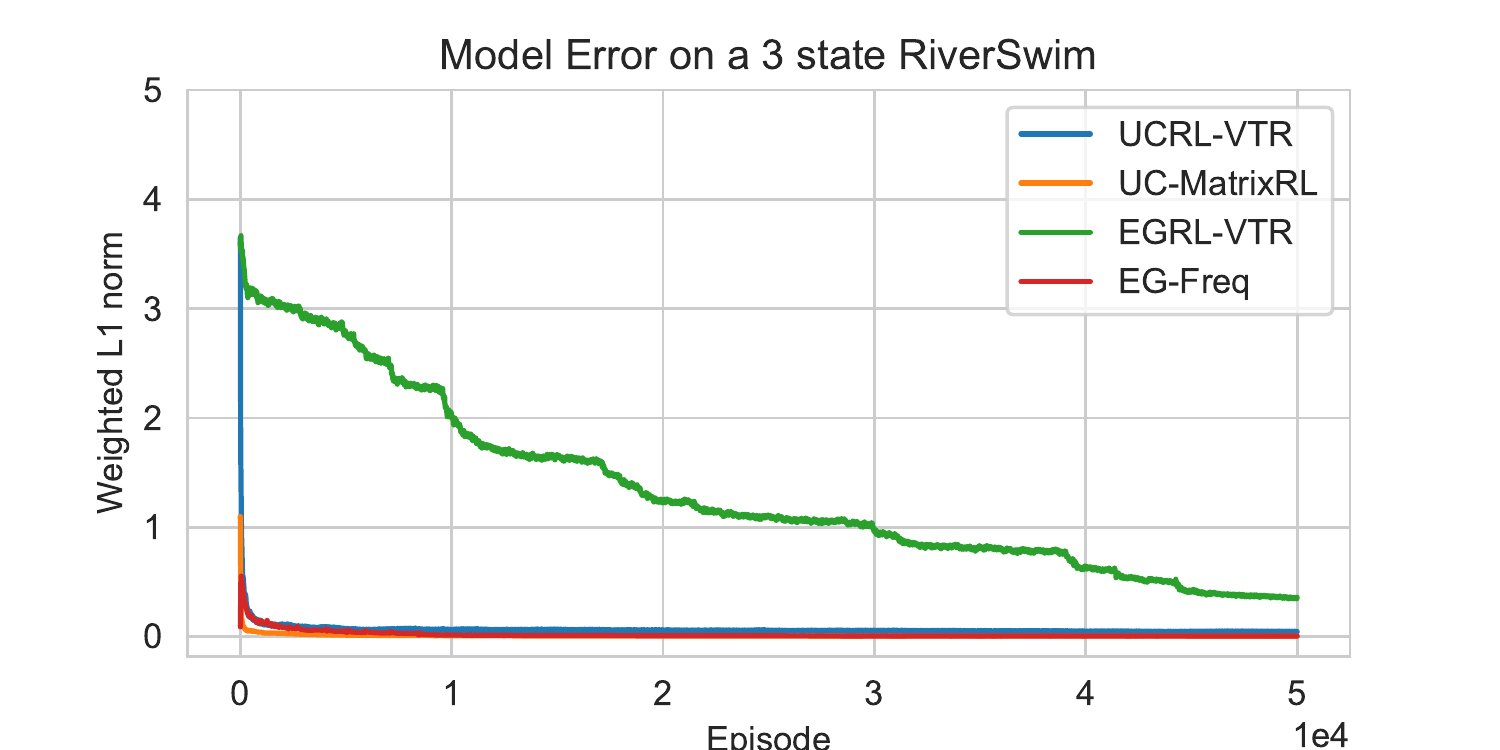}
\end{subfigure}%
\begin{subfigure}{.32\textwidth}
  \centering
  \includegraphics[width=1.1\linewidth]{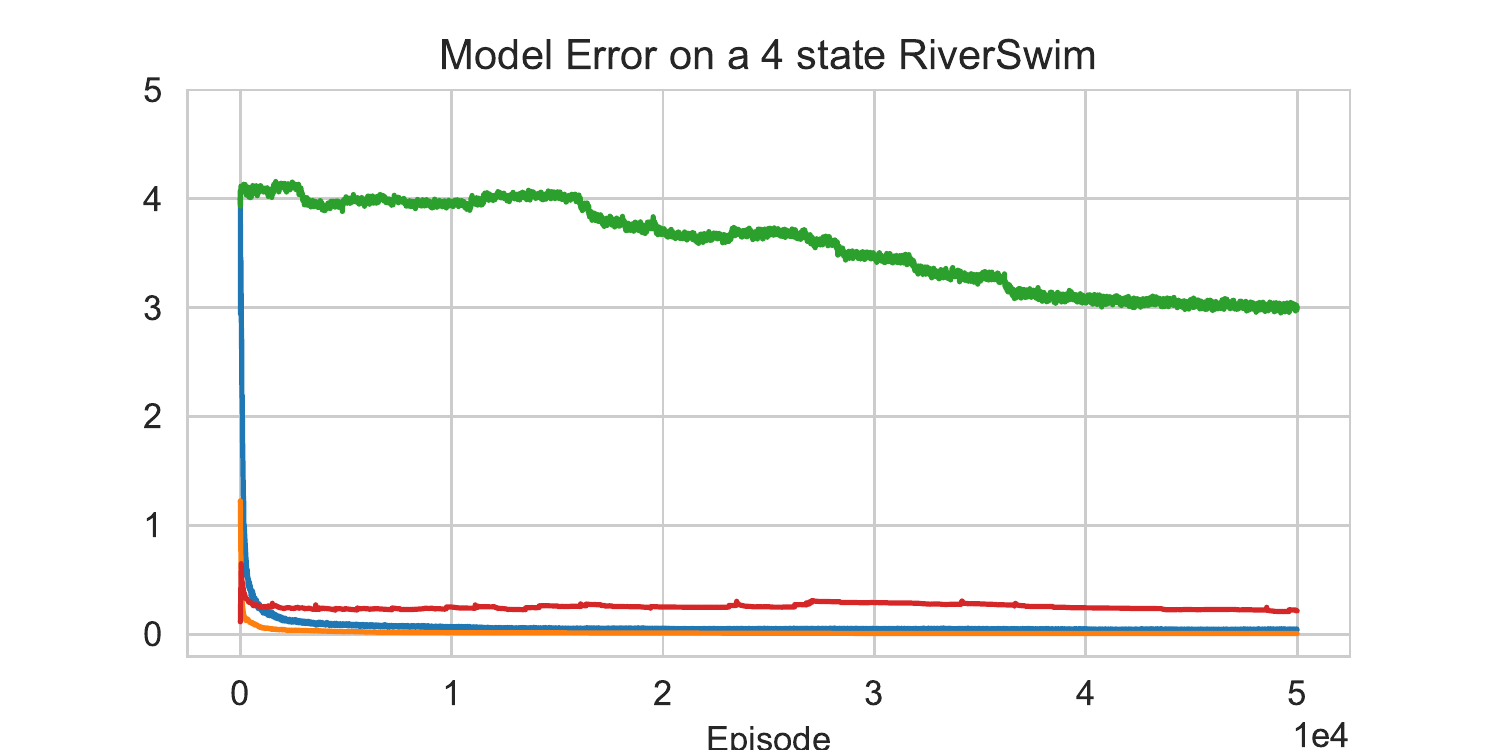}
\end{subfigure}
\begin{subfigure}{.32\textwidth}
  \centering
  \includegraphics[width=1.1\linewidth]{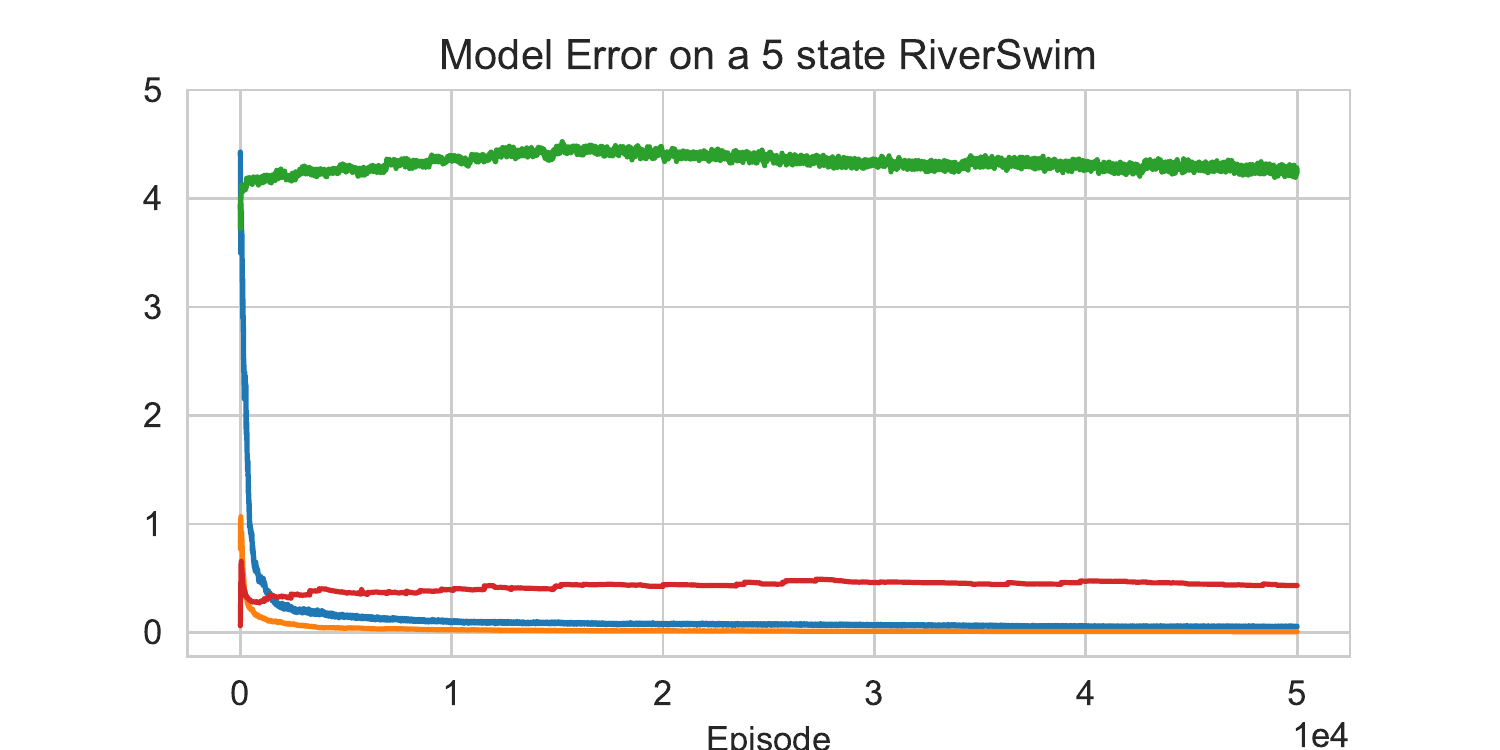}
\end{subfigure}
\caption{The results for the $\epsilon$-greedy algorithms were averaged over thirty runs and the results for the UC algorithms were averaged over ten runs. Error bars are only reported for the regret plots.}
\label{fig:RiverSwim}
\end{figure}

\subsection{WideTree} 
\label{sec:WideTree}
We introduce a novel tabular MDP we call WideTree. The WideTree environment has a fixed horizon $H=2$ but can vary in the number of states. A visualization of an eleven state WideTree environment is shown in Figure \ref{fig:WideTreeMDP}.

\begin{figure}[H]
\centering
\includegraphics[width=.65\textwidth]{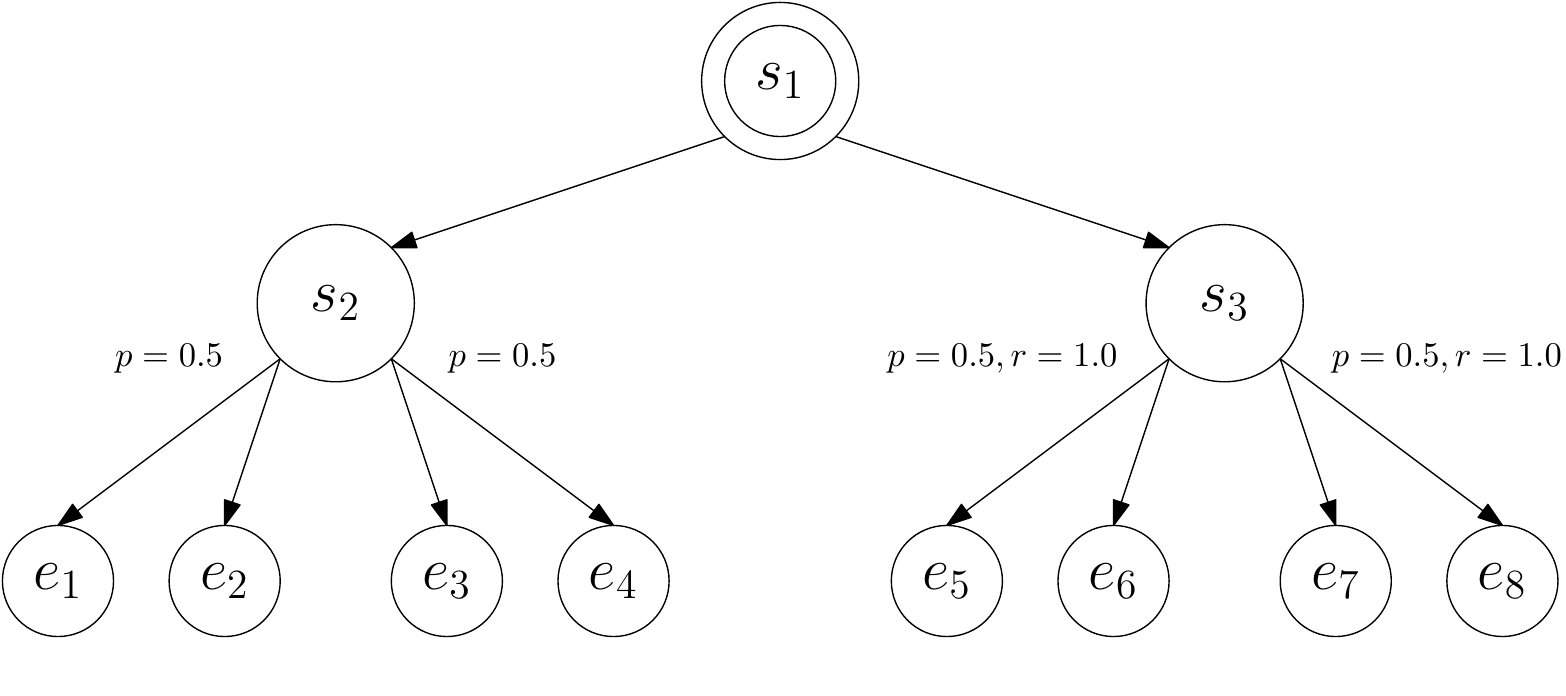}

\caption{An eleven state WideTree MDP. The algorithm starts in the initial state $s_1$. From the initial state $s_1$ the algorithm has a choice of either deterministically transitioning to either state $s_2$ or state $s_3$. Finally from either state $s_2$ or state $s_3$ the algorithm picks one of two possible actions and transitions to one of the terminal states $e_i$. Depending on which state the algorithm transitioned to from the initial state $s_1$, determines which delayed reward the algorithm will observe. The delayed reward is observed at the final stage $h=2$ of this MDP.}

\label{fig:WideTreeMDP}

\end{figure}

In WideTree, an agent starts at the initial state $s_1$. The agent then progresses to one of the many bottom terminal states and collects a reward of either 0 or 1 depending on the action selected at state $s_1$. 
The only significant action is whether to transition from $s_1$ to either $s_2$ or $s_3$. 
Note that the model in the second layer is irrelevant for making a good decision: Once in $s_3$, all actions lead to a reward of one, and once in $s_2$, all actions lead to a reward of zero.
We vary the number of bottom states reachable from states $s_2$ and $s_3$ while still maintaining a reward structure depending on whether the algorithm choose to transition to either $s_2$ or $s_3$ from the initial state $s_1$. 

We set $\epsilon = 0.1$ in this environment, though choosing smaller $\epsilon$ but as long as $\epsilon > 0$ then both EGRL-VTR and EG-Freq will incur linear regret dependent on the choice of $\epsilon$. One could also change the reward function in order to make learning for a given $\epsilon$ hard. 

\begin{figure}[H]
\centering
\begin{subfigure}{.32\textwidth}
  \centering
  \includegraphics[width=1.1\linewidth]{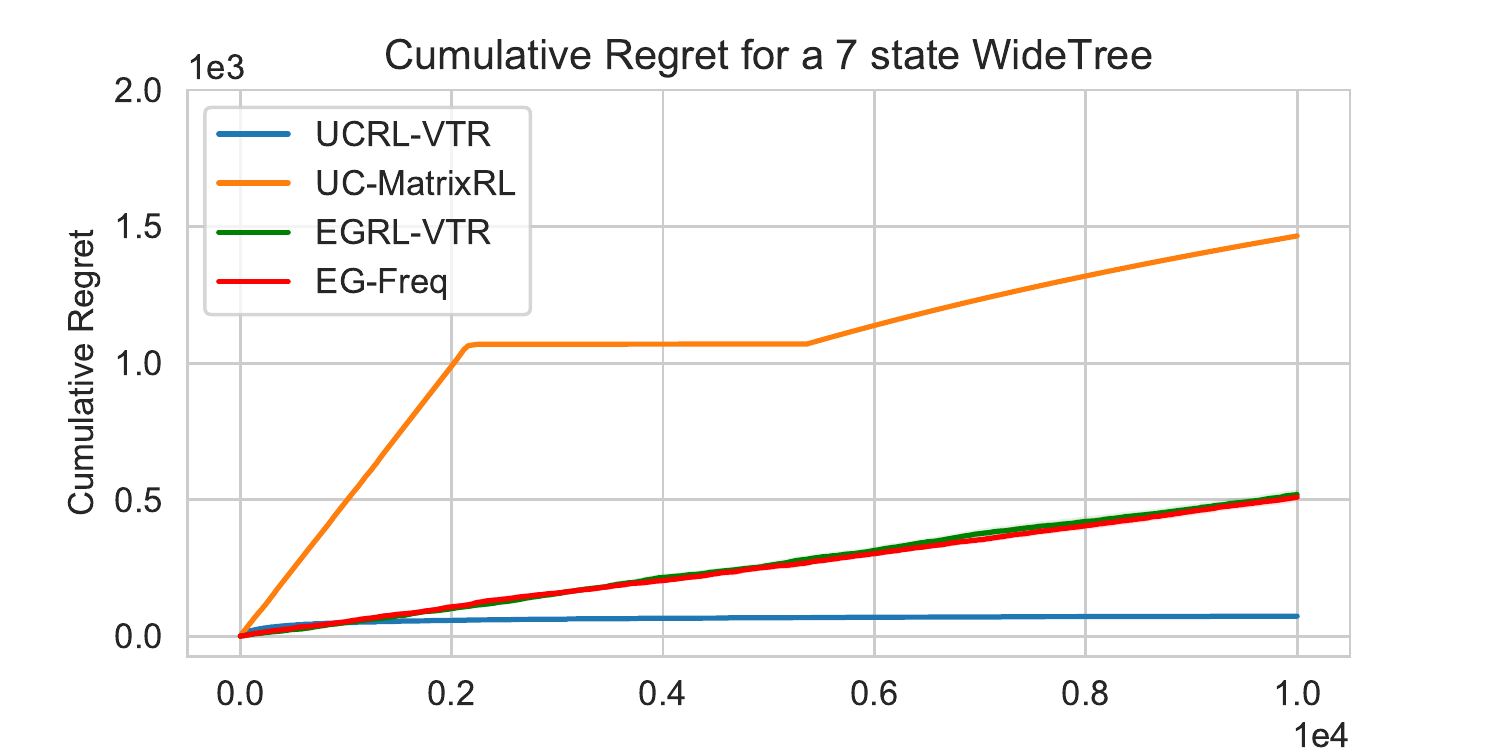}

\end{subfigure}%
\begin{subfigure}{.32\textwidth}
  \centering
  \includegraphics[width=1.1\linewidth]{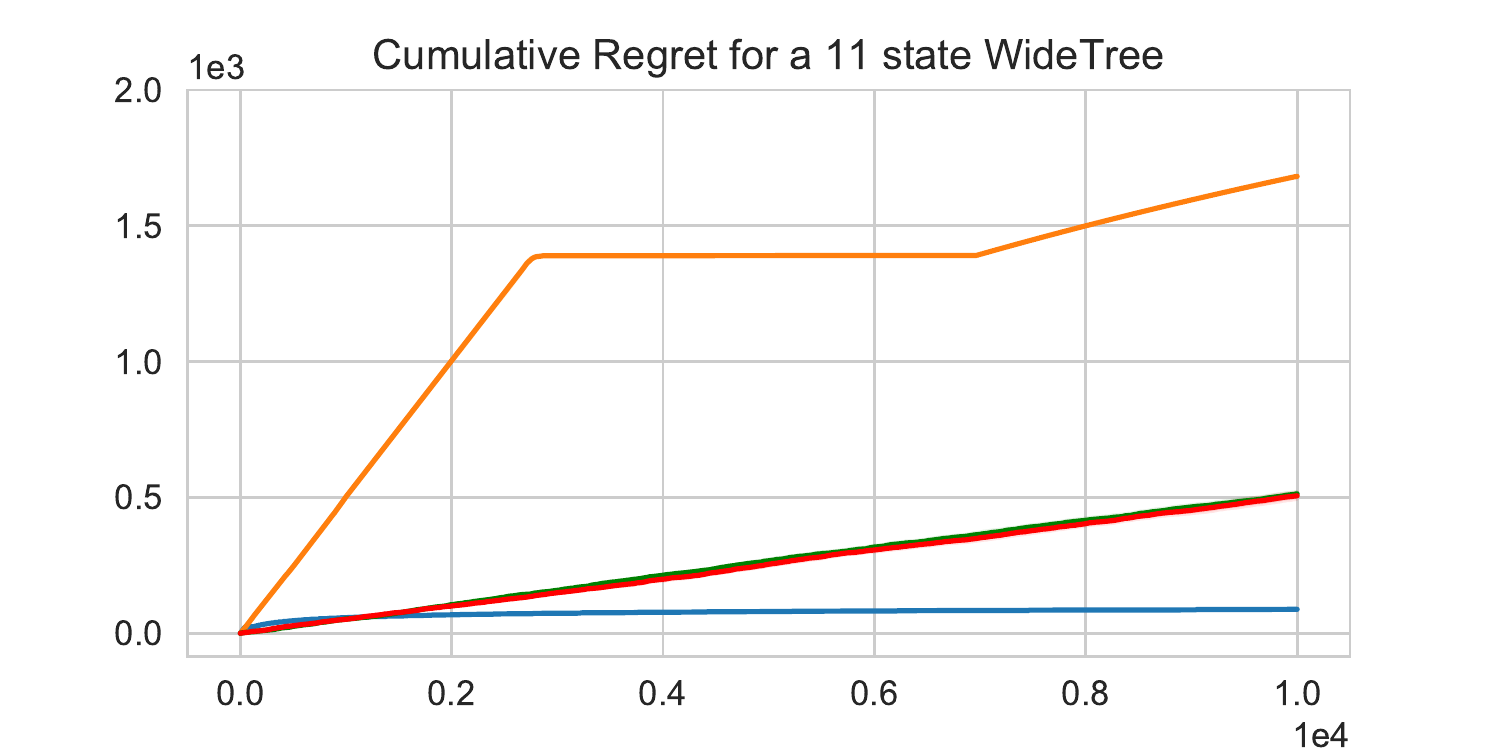}
\end{subfigure}
\begin{subfigure}{.32\textwidth}
  \centering
  \includegraphics[width=1.1\linewidth]{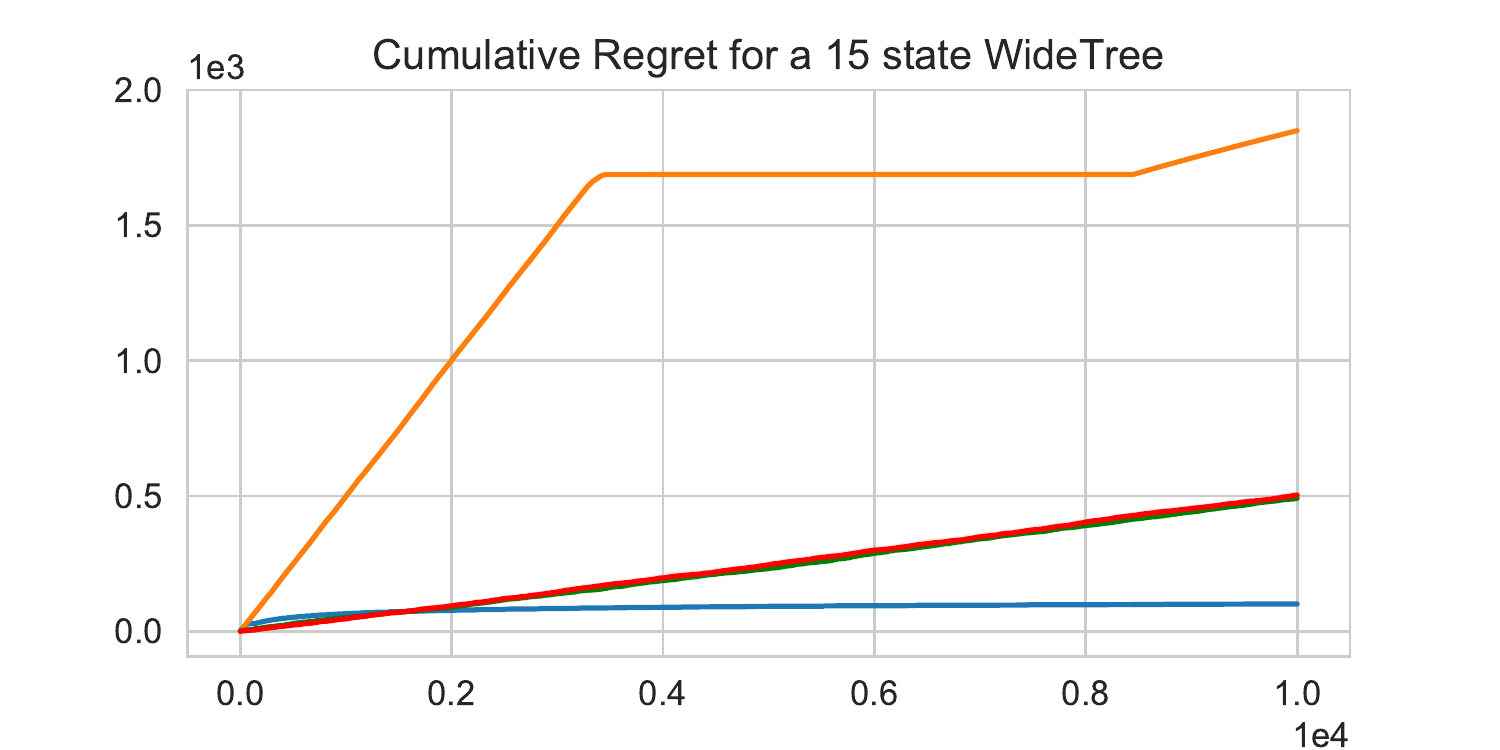}
\end{subfigure}
\begin{subfigure}{.32\textwidth}
  \centering
  \includegraphics[width=1.1\linewidth]{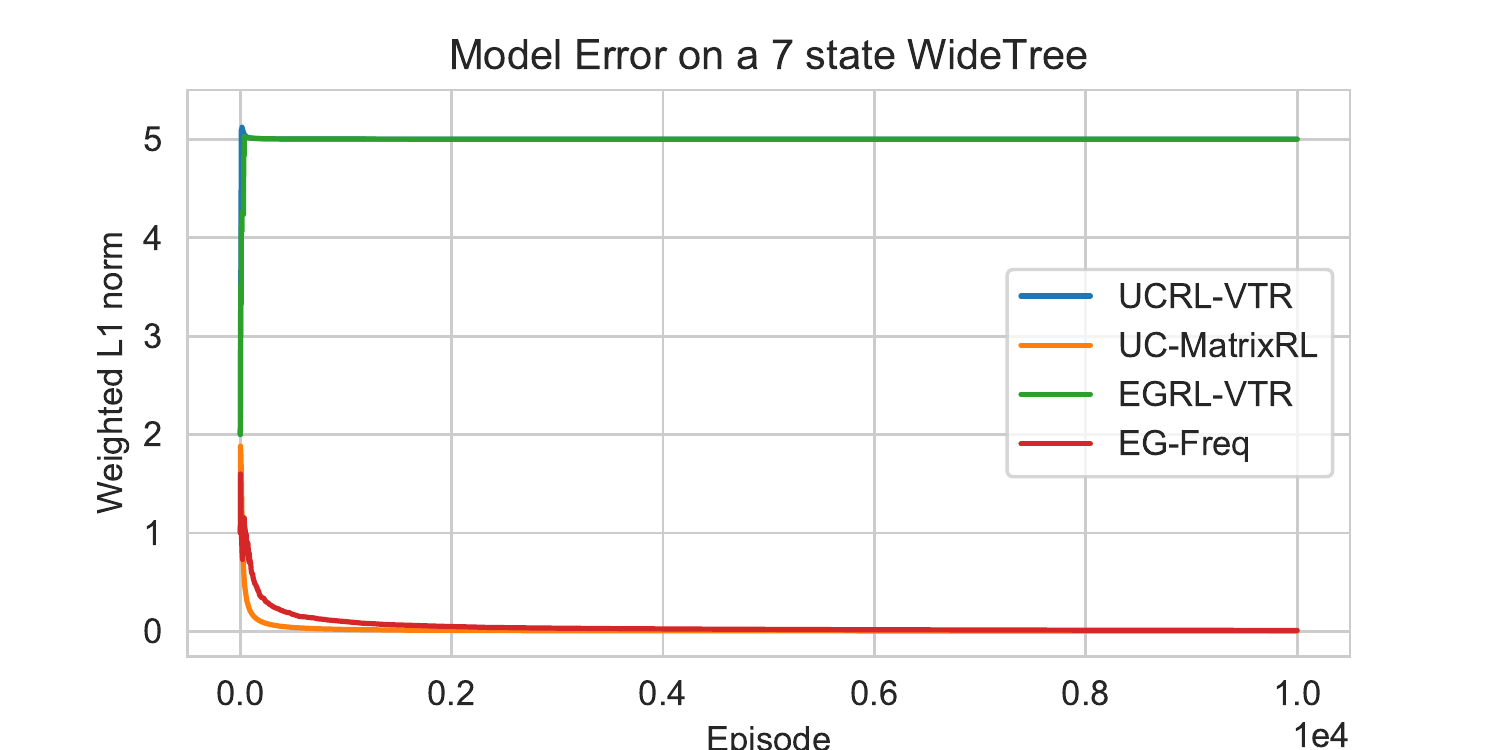}
\end{subfigure}%
\begin{subfigure}{.32\textwidth}
  \centering
  \includegraphics[width=1.1\linewidth]{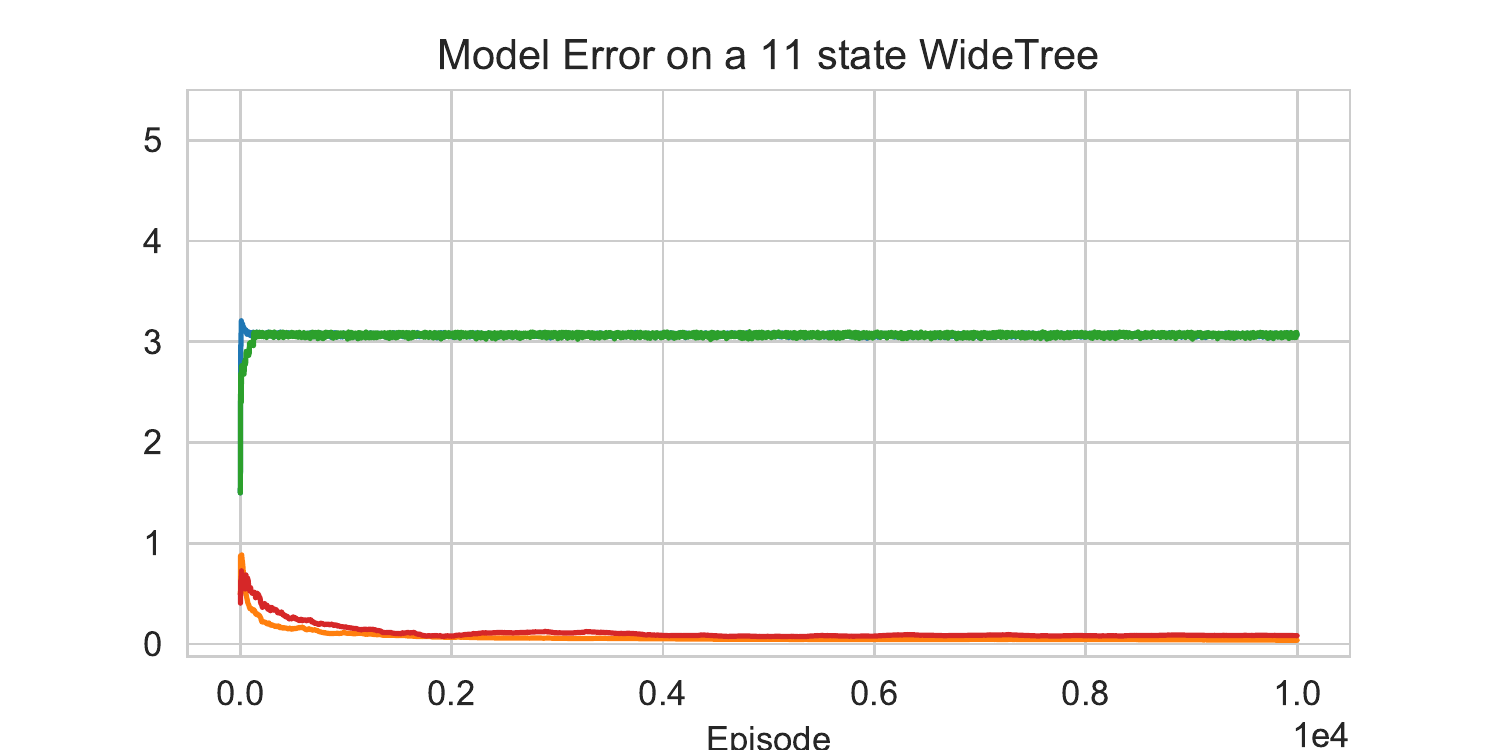}
\end{subfigure}
\begin{subfigure}{.32\textwidth}
  \centering
  \includegraphics[width=1.1\linewidth]{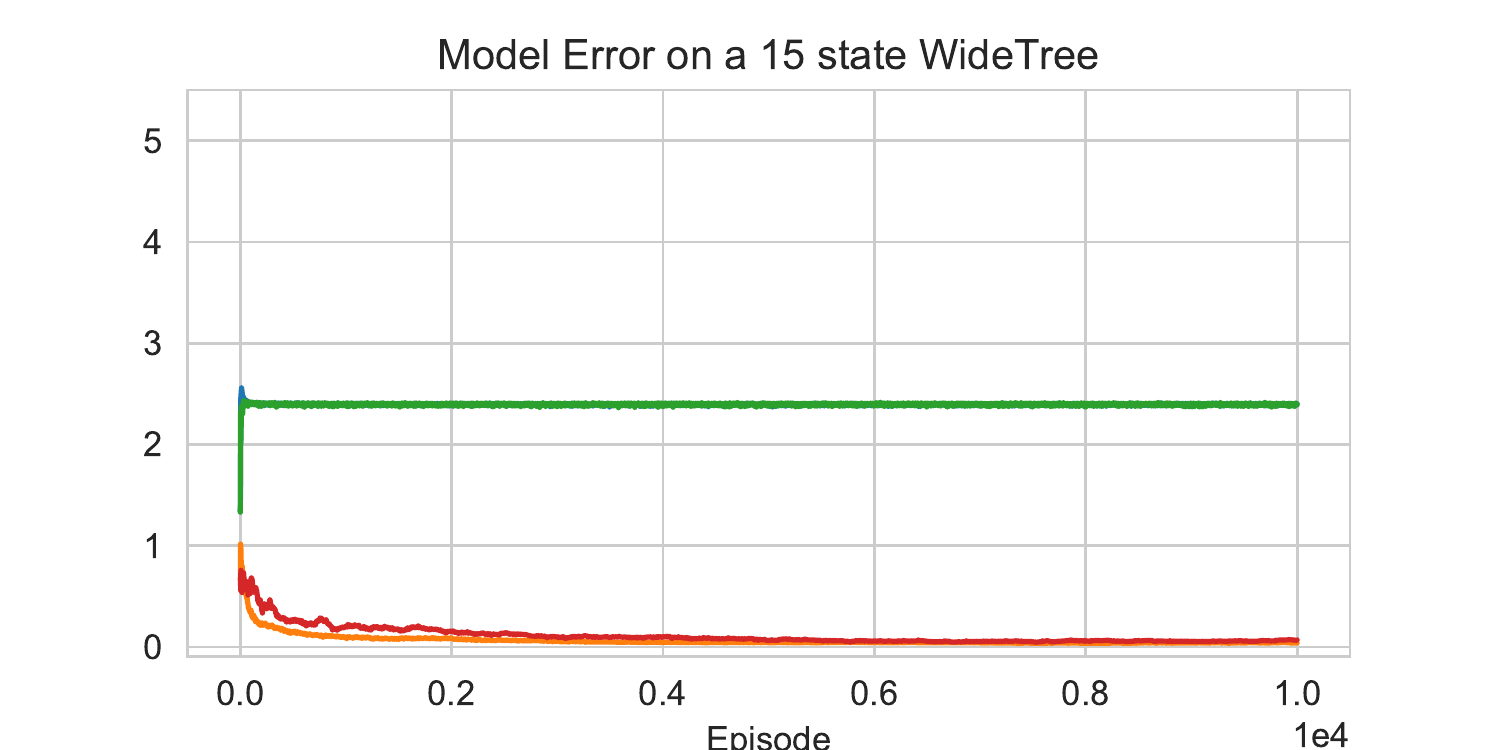}
\end{subfigure}
\caption{As with the RiverSwim experiments, the results for the $\epsilon$-greedy algorithms were averaged over thirty runs and the results for the UC algorithms were averaged over ten runs. Error bars are reported for the regret plots.}
\label{fig:WideTree}
\end{figure}
The results are shown in Figure~\ref{fig:WideTree}, except for UCRL-Mixed and EG-Mixed, whose results are given in Appendix~\ref{sec:mixture_model}. 
Both UCRL-VTR and EG-VTR learn equally poor models (their graphs are `on the top of each other').
Yet, UCRL-VTR manages to quickly learn a good policy, as attested by its low regret.

EG-Freq and EG-VTR perform equally poorly and UC-MatrixRL is even slower as it keeps exploring the environment. 
These experiments clearly illustrate that UCRL-VTR is able to achieve good results without learning a good model -- its focus on values makes pays off swiftly in this well-chosen environment.

\section{Conclusions}
\label{sec:conc}
We considered online learning in episodic MDPs and
proposed an optimistic model-based reinforcement learning method (UCRL-VTR)
with the unique characteristic to evaluate and select models based on their ability to predict value functions that the algorithm constructs during learning.
The regret of the algorithm was shown to be bounded by a quantity that relates to the richness of the model class through the Eluder dimension and the metric entropy of an appropriately construction function space. 
For the case of linear mixture models, the regret bound simplifies to  $\tilde O(H^{3/2} d\sqrt{T})$ where $d$ is the number of model parameters, $H$ is the horizon, and $T$ is the total number of interaction steps. 
Our experiments confirmed that the value-targeted regression objective is not only theoretically sound, but also yields a competitive method which allows task-focused model-tuning: In a carefully chosen environment we demonstrated that the algorithm achieves low regret despite that it ignores modeling a major part of the environment.

\section{Acknowledgements}
Csaba Szepesv\'ari gratefully acknowledges funding  from 
the Canada CIFAR AI Chairs Program, Amii and NSERC.

\bibliography{all,reference}
\bibliographystyle{icml2020}

\newpage

\onecolumn
\appendix

\section{Proof of Theorem \ref{thm:mainbound}}

\def\mw#1{}
In this section, we provide the regret analysis of the UCRL-VTR Algorithm (Algorithm~\ref{alg}). 
We will explain the motivation for our construction of confidence sets for general nonlinear squared estimation, and establish the regret bound for a general class of transition models, $\mathcal{P}$.

\subsection{Preliminaries}
Recall that a finite horizon MDP is $M=(\cS,\cA,P,r,H,s_\circ)$ where
$\cS$ is the state space, $\cA$ is the action space,
$P = (P_a)_{a\in \cA}$ is a collection of $P_a: \cS \to M_1(\cS)$ Markov kernels,
$r: \cS \times \cA \to [0,1]$ is the reward function, $H>0$ is the horizon and $s_\circ\in \cS$ is the initial state.
For a state $s\in \cS$ and an action $a\in \cA$, $P_a(s)$ gives the distribution of the next state that is obtained when action $a$ is executed in state $s$.
For a  bounded (measurable) function $V: \cS \to \R$, we will use $\ip{ P_a(s), V }$ as the shorthand for the expected value of $V$ at a random next state $s'$ whose distribution is $P_a(s)$.

Given any policy $\pi$ (which may or may not use the history), its value function is 
\begin{align*}
    V^\pi(s) = \E_{\pi,\delta_s}\left[\sum_{i=1}^H r(s_i,a_i)\right]\,,
\end{align*}
where $E_{\pi,\delta_s}$ is the expectation operator underlying the probability measure $P_{\pi,\delta_s}$ induced over sequences of state-action pairs of length $H$ by 
executing policy $\pi$ starting at state $s$ in the MDP $M$ and $s_h$ is the state visited in stage $h$ and action $a_h$ is the action taken in that stage after visiting $s_h$.
For a nonstationary Markov policy $\pi = (\pi_1,\dots,\pi_H)$, we also let
\begin{align*}
    V_h^\pi(s) = \E_{\pi_{h:H},\delta_s} \left[\sum_{i=1}^{H-h+1} r(s_i,a_i)\right]
\end{align*}
be the value function of $\pi$ from stage $h$ to $H$. Here, $\pi_{h:H}$ denotes the policy $(\pi_h,\dots,\pi_{H})$.
The optimal value function $V^*=(V^*_1,\dots,V^*_H)$ is defined via $V^*_h(s) = \max_{\pi} V^\pi_h(s)$, $s\in \cS$. 

For simplicity assume that $r$ is known. To indicate the dependence of $V^*$ on the transition model $P$, we will write $V^*_P = (V^*_{P,1},\dots,V^*_{P,H})$. For convenience, we define $V^*_{P,H+1}=0$.

\def\Comment#1{}

Algorithm~\ref{alg} is an instance of the following general model-based optimistic algorithm:
\begin{algorithm}
\begin{algorithmic}[1]
\STATE \textbf{Input: } $\cP$ -- a set of transition models, $K$ -- number of episodes, $s_0$ -- initial state
\STATE Set $\cB_1 = \cP$ \Comment{Initial confidence set for transition models}
\FOR{$k = 1,\dots,K$} \Comment{episodes}
	\STATE $P^k = \hbox{argmax}\{ V^*_{\tilde P}(s_0)\,:\, \tilde P\in \cB_k \}$ \Comment{Optimistic model}
	\STATE $V_k = V^*_{P^k}$ \Comment{Optimistic $H$-stage value function}
	\STATE $s_1^k = s_0$
	\FOR{$h=1,\dots,H$} \Comment{Acting}
	   \STATE Choose $a_h^k = \hbox{argmax}_{a\in \cA} r(s_h^k,a) + \ip{ P_a^k(s_h^k), V_{h+1,k} }$
	    \STATE Observe transition to $s_{h+1}^k$
	\ENDFOR
	\STATE 
	\label{alg:line:bk}
	Construct $\cB_{k+1}$ based on $(s_1^k,a_1^k,\dots,s_H^k,a_H^k)$ 
\ENDFOR
\end{algorithmic}
\caption{Generic Algorithm \ref{alg}-Schema for finite horizon problems}
\label{alg:ucrl}
\end{algorithm}

Specific instances of Algorithm~\ref{alg:ucrl} differ in terms of how $\cB_{k+1}$ is constructed.
In particular, UCRL-VTR uses the construction described in Section~\ref{eq-Bk}.

Recall that $V_k = (V_{1,k},\dots,V_{H,k}, V_{H+1,k})$ (with $V_{H+1,k}=0$) in Algorithm \ref{alg:ucrl}. 
Let $\pi_k$ be the nonstationary Markov policy chosen in episode $k$ by Algorithm \ref{alg:ucrl}.
Let 
\begin{align*}
    R_K = \sum_{k=1}^K V_1^*(s_1^k)-V_1^{\pi_k}(s_1^k)
\end{align*}
be the pseudo-regret of Algorithm \ref{alg} for $K$ episodes. The following standard lemma bounds the $k$th term of the expression on the right-hand side.
\begin{lemma}\label{lem:optbound}
Assuming that $P\in \cB_k$, we have
\begin{align*}
    V_1^*(s_1^k) - V_1^{\pi_k}(s_1^k) 
    \le
    \sup_{\tilde P\in \cB_k} \sum_{h=1}^{H-1} \ip{\tilde P_{a_h^k}(s_h^k)-P_{a_h^k}(s_h^k),V_{h,k}}
    +
    \sum_{h=1}^{H-1} \xi_{h+1,k}\,,
\end{align*}
where 
\begin{align*}
    \xi_{h+1,k} = \ip{ P_{a_h^k}(s_h^k), V_{h+1,k}-V_{h+1}^{\pi_k}} - \left(V_{h+1,k}(s_{h+1}^k) - V_{h+1}^{\pi_k}(s_{h+1}^k)\right)\,.
\end{align*}
\end{lemma}
Note that 
$(\xi_{2,1},\xi_{3,1},\dots,\xi_{H,1},\xi_{2,2},\xi_{3,2},\dots,\xi_{H,2},\xi_{2,3},\dots)$ 
is a sequence of martingale differences.
\begin{proof}
Because $P\in \cB_k$, $V_1^*(s_1^k) \le V_{1,k}(s_1^k)$ by the definition of the algorithm.
Hence,
\begin{align*}
  V_1^*(s_1^k) - V_1^{\pi_k}(s_1^k)  
  & \le
  V_{1,k}(s_1^k) - V_1^{\pi_k}(s_1^k)\,.
\end{align*}
Fix $h\in [H]$. 
In what follows we bound $V_{h,k}(s_h^k) - V_h^{\pi_k}(s_h^k)$.
By the definition of $\pi_k$, $P^k$ and $a_h^k$, we have
\begin{align*}
    V_{h,k}(s_h^k) & = r(s_h^k,a_h^k) + \ip{ P^k_{a_h^k}(s_h^k), V_{h+1,k} } \text{ and } \\
    V_h^{\pi_k}(s_h^k) & = r(s_h^k,a_h^k) + \ip{ P_{a_h^k}(s_h^k), V_{h+1}^{\pi_k} } \,.
\end{align*}
Hence,
\begin{align*}
    V_{h,k}(s_h^k) - V_h^{\pi_k}(s_h^k)
    & = 
    \ip{ P^k_{a_h^k}(s_h^k), V_{h+1,k} } - \ip{ P_{a_h^k}(s_h^k), V_{h+1}^{\pi_k} } \\
    & = 
    \ip{ P^k_{a_h^k}(s_h^k)-P_{a_h^k}(s_h^k), V_{h+1,k} } 
    + \ip{ P_{a_h^k}(s_h^k), V_{h+1,k}-V_{h+1}^{\pi_k} }\,.
\end{align*}
Therefore, by induction, noting that $V_{H+1,k}=0$, we get that 
\begin{align*}
  V_1^*(s_1^k) - V_1^{\pi_k}(s_1^k) 
    & \le
    \sum_{h=1}^{H-1} \ip{ P^k_{a_h^k}(s_h^k)-P_{a_h^k}(s_h^k), V_{h+1,k} } 
    +\sum_{h=1}^{H-1} \xi_{h+1,k} \\
    & \le
    \sup_{\tilde P\in \cB_k}
    \sum_{h=1}^{H-1} \ip{ \tilde P_{a_h^k}(s_h^k)-P_{a_h^k}(s_h^k), V_{h+1,k} } 
    +\sum_{h=1}^{H-1} \xi_{h+1,k}\,.
\end{align*}
\end{proof}

\subsection{The confidence sets for Algorithm~\ref{alg}}

The previous lemma suggests that at the end of the $k$th episode, 
the model could be estimated using
\begin{align}
    \hat P_k = \argmin_{\tilde P \in \cP} \sum_{k'=1}^k \sum_{h=1}^{H-1} \left(
    \ip{\tilde P_{a_h^{k'}}(s_h^{k'}),V_{h+1,{k'}}} - V_{h+1,k'}(s_{h+1}^{k'})
    \right)^2
    \label{eq:lsmodel}
\end{align}
For a confidence set construction, we get inspiration from  Proposition 5 in the paper of \citet{osband2014model}.
The set is centered at $\hat P_k$:
\begin{align}
    \cB_k = \{ \tilde P \in \cP\,:\, L_k(\hat P_k,\tilde P) \le \beta_k \}\,,
    \label{eq:probconfset}
\end{align}
where
\begin{align*}
    L_k(\hat P,\tilde P) = 
    \sum_{k'=1}^k \sum_{h=1}^{H-1} \left(
    \ip{\tilde P_{a_h^{k'}}(s_h^{k'})- \hat P_{a_h^{k'}}(s_h^{k'}),V_{h+1,k'}}\right)^2\,.
\end{align*}
Note that this is the same confidence set as described in Section~\ref{eq-Bk}.
To obtain the value of $\beta_k$, 
we now consider the nonlinear least-squares confidence set construction from \citet{RuVR14}.
The next section is devoted to this construction.

\subsection{Confidence sets for general nonlinear least-squares}

Let $(X_p,Y_p)_{p=1,2,\dots}$ be a sequence of random elements, $X_p\in \cX$ for some measurable set $\cX$ and $Y_p\in \R$.
Let $\cF$ be a subset of the set of real-valued measurable functions with domain $\cX$.
Let $\FF = (\FF_p)_{p=0,1,\dots}$ be a filtration such that for all $p\ge 1$, 
$(X_1,Y_1,\dots,X_{p-1},Y_{p-1},X_p)$ is $\FF_{p-1}$ measurable and 
such that there exists some function $f_*\in \cF$ 
such that $\EE[ Y_p\mid \FF_{p-1} ] = f_*(X_p)$ holds for all $p\ge 1$.
The (nonlinear) least-squares predictor given $(X_1,Y_1,\dots,X_t,Y_t)$ is $\hat f_t = \argmin_{f\in \cF} \sum_{p=1}^t (f(X_p) - Y_p)^2$. We say that $Z$ is conditionally $\rho$-subgaussian given the $\sigma$-algebra $\FF$ if 
for all $\lambda\in \R$,
$\log \EE[ \exp(\lambda Z)|\FF ] \le \frac12 \lambda^2 \rho^2$. 
For $\alpha>0$, let $N_\alpha$
be the $\norm{\cdot}_\infty$-covering number of $\cF$ at scale $\alpha$. 
That is, $N_\alpha$ is the smallest integer for which there exist
$\cG\subset \cF$ with $N_{\alpha}$ elements such that for any $f\in \cF$, $\min_{g\in \cG} \norm{f-g}_\infty\le \alpha$.
For $\beta>0$, define
\begin{align*}
\cF_t(\beta) = \{ f\in \cF\,:\, \sum_{p=1}^t (f(X_p) - \hat f_t(X_p))^2 \le \beta \}\,.
\end{align*}
We have the following theorem, the proof of which is given in Section~\ref{sec:proofofnonlscs}.
\begin{theorem}
\label{thm:nonlscs}
Let $\FF$ be the filtration defined above and assume that the functions in $\cF$ are bounded by the positive constant $C>0$.
Assume that for each $s\ge 1$,
 $(Y_p - f_*(X_p))_p$ is conditionally $\sigma$-subgaussian given $\FF_{p-1}$.
Then, for any $\alpha>0$,
with probability $1-\delta$, for all $t\ge 1$, $f_* \in \cF_t(\beta_t(\delta,\alpha))$, where
\begin{align*}
\beta_t(\delta,\alpha) = 8 \sigma^2 \log(2N_\alpha/\delta) +  4 t \alpha \left(C+\sqrt{\sigma^2 \log(4 t(t+1)/\delta)}\right)\,.
\end{align*}
\end{theorem}
The proof follows that of Proposition~6, \citet{RuVR14}, with minor improvements, which lead to a slightly better bound.
In particular, with our notation, \citeauthor{RuVR14} stated their result with
\begin{align*}
\beta_t^{\text{RvR}}(\delta,\alpha) = 8 \sigma^2 \log( 2N_\alpha/\delta) +  2 t \alpha \left(8C+\sqrt{8\sigma^2 \log(8 t^2 /\delta)}\right)\,.
\end{align*}
While $\beta_t(\delta,\alpha)\le \beta_t^{\text{RvR}}(\delta,\alpha) $, the improvement is only in terms of smaller constants.

\subsection{The choice of $\beta_k$ in Algorithm~\ref{alg}}
\label{sec:confapp}
To use this result in our RL problem recall
 that $\cP$ is the set of transition probabilities parameterized by $\theta\in\Theta$. 
We index time $t=1,2,\dots$ in a continuous fashion. 
Episode $k=1,2,\dots$ and stage $h=1,\dots,H-1$ corresponds to time $t=(k-1)(H-1)+h$:
\begin{center}
\begin{tabular}{|l||c|c|c|c|c|c|c|c|c|c|}
\hline
episode ($k$)  &  $1$  &  $1$ & $\dots$  & $1$       & $2$  &  $2$ & $\dots$  & $2$         & $3$ & $\dots$  \\ \hline
stage ($h$)      &  $1$  &  $2$ & $\dots$  & $H-1$  & $1$  &  $2$ & $\dots$  & $H-1$     & $1$ & $\dots$  \\ \hline
time step ($t$) &  $1$  &  $2$ & $\dots$  & $H-1$ & $H$ & $H+1$ & $\dots$& $2H-2$ & $2H-1$ & $\dots$ \\ \hline
\end{tabular}
\end{center}
Note that the transitions at stage $h=H$ are skipped and the time index at the end of episode $k\ge 1$ is $k(H-1)$.

Let $V_{(t)}$ be the value function used by Algorithm \ref{alg} at time $t$ ($V_{(t)}$ is constant in periods of length $H-1$), while let $(s_{(t)},a_{(t)})$ be the state-action pair visited at time $t$. 

Let $\cV$ be the set of optimal value functions under some model in $\cP$: $\cV = \{ V^*_{P'}\,:\, P' \in \cP \}$. 
Note that $\cV \subset \cB(\cS,H)$, where 
$\cB(\cS,H)$ denotes the set of real-valued measurable functions with domain $\cS$ that are bounded by $H$. 
Note also that for all $t$, $V_{(t)}\in \cV$.
Define $\cX = \cS \times \cA \times \cV$.
We also let $X_t = (s_{(t)},a_{(t)}, V_{(t)})$, 
$Y_t = V_{(t)}(s_{(t+1)})$ 
when $t+1\not\in\{H+1,2H+1,\dots\}$ and $Y_t = V_{(t)}(s_{H+1}^k)$,
and choose 
\begin{align}
\cF = \left\{ f: \cX \to \R \,:\, \exists \tilde P\in \cP \text{ s.t. } f(s,a,v) = \int \tilde P_a(ds'|s) v(s') \right\}\,.
\label{eq:inducedF2}
\end{align}
Note that $\cF \subset \cB_\infty(\cX,H)$. 

Let $\phi:\cP \to \cF$ be the natural surjection to $\cF$: $\phi(P) = f$ where $f(s,a,v) = \int P_a(ds'|s) v(s')$ for $(s,a,v)\in \cX$. 
We know show that $\phi$ is in fact a bijection.
If $P\ne P'$, this means that for some $(s,a)\in \cS\times \cA$ and $U\subset \cS$ measurable, $P_a(U|s)\ne P_a'(U|s)$. Choosing $v$ to be the indicator of $U$, note that $(s,a,v)\in \cX$. Hence, $\phi(P)(s,a,v) = P_a(U|s)\ne P_a'(U|s)=\phi(P')(s,a,v)$, and hence $\phi(P)\ne \phi(P')$: $\phi$ is indeed a bijection. For convenience and to reduce clutter, we will write $f_P = \phi(P)$.

Choose $\FF = (\FF_t)_{t\ge 0}$ so that $\FF_{t-1}$ is generated by $(s_{(1)},a_{(1)},V_{(1)},\dots,s_{(t)},a_{(t)},V_{(t)})$. 
Then $\EE[Y_t|\FF_{t-1}] = \int P_{a_{(t)}}(ds'|s_{(t)}) V_{(t)}(s') = f_P(X_t)$ and by definition $f_P\in \cF$. 
Now, $Y_t\in [0,H]$, hence, $Z_t = Y_t - f_P(X_t)$ is conditionally $H/2$-subgaussian given $\FF_{t-1}$.

Let $t =  k(H - 1)$ for some $k\ge 1$. Thus, this time step corresponds to finishing episode $k$ and thus $V_{(t)} = V_k$.
Furthermore, letting $\hat f_t = \argmin_{f\in \cF} \sum_{p=1}^t (f(X_p)-Y_p)^2$, since $\phi$ is an injection,
we see that $\hat f_t = f_{\hat P_k}$ where $\hat P_k$ is defined using \eqref{eq:lsmodel}. 
For $P',P''\in \cP$, we have
$L_k(P', P'') = \sum_{p=1}^t ( f_{P'}(X_p) - f_{P''}(X_p))^2$ and thus
\begin{align*}
\cB_k 
& = \{ \tilde P \in \cP \,:\, L_k(\hat P_k,\tilde P) \le \beta_k \} 
 = \{ \tilde P \in \cP \,:\, \sum_{p=1}^t ( \hat f_t(X_p) - f_{\tilde P}(X_p))^2 \le \beta_k \} \\
& = \{ \phi^{-1}(f) \,:\, f\in \cF \text{ and } \sum_{p=1}^t ( \hat f_t(X_p) - f(X_p))^2 \le \beta_k \} 
 = \phi^{-1}(\cF_t(\beta_k))\,.
\end{align*}

\begin{corollary}
\label{cor:betakdef}
For $\alpha>0$ and $k\ge 1$ let 
\begin{align*}
\beta_k = 2 H^2 \log\left(\frac{2  \cN(\cF,\alpha, \norm{\cdot}_{\infty})  }{\delta}\right) 
+  2H (kH-1) \alpha \left\{2+\sqrt{ \log\left(\frac{4 kH(kH-1)}{\delta}\right)}\right\}\,.
\end{align*}
Then, with probability $1-\delta$, for any $k\ge 1$, $P \in \cB_k$ where $\cB_k$ is defined by \eqref{eq:probconfset}.
\end{corollary}

\subsection{Regret of Algorithm \ref{alg}}

Recall that $\cX = \cS \times \cA \times \cV$ where $\cV\subset \cB_\infty(\cS,H)$ is the set of value functions that are optimal under some model in $\cP$.
We will abbreviate $(x_1,\dots,x_t)\in \cX^t$ as $x_{1:t}$.
Further, we
 let $\cF|_{x_{1:t}} = \{ (f(x_1),\dots,f(x_t))\,:\, f\in \cF\} (\subset \R^t)$
 and for $S\subset \R^t$, let $\diam( S ) = \sup_{u,v\in S} \norm{u-v}_2$ be the diameter of $S$.
We will need the following lemma, extracted from \citet{RuVR14}:
\begin{lemma}[Lemma~5 of \citet{RuVR14} ]
\label{lem:eluder}
Let $\cF \subset B_{\infty}(\cX,C)$ be a set of functions bounded by $C>0$,
$(\cF_t)_{t\ge 1}$ and $(x_t)_{t\ge 1}$ be sequences
such that $\cF_t \subset \cF$ and $x_t\in \cX$ hold for $t\ge 1$.
Then, for any $T\ge 1$ and $\alpha>0$ it holds that
\begin{align*}
\sum_{t=1}^T \diam(\cF_t|_{x_t}) 
	\le \alpha + C(d \wedge T) + 2 \delta_T \sqrt{d T} \,,
\end{align*}

where $\delta_T = \max_{1\le t \le T}\diam(\cF_t|_{x_{1:t}})$ and $d = \dimE(\cF,\alpha)$. 
\end{lemma}

Let
\begin{align*}
W_k = \sup_{\tilde P\in \cB_k} \sum_{h=1}^{H-1} \ip{\tilde P_{a_h^k}(s_h^k)-P_{a_h^k}(s_h^k),V_{h,k}}\,.
\end{align*}
From Lemma \ref{lem:optbound}, we get
\begin{align}
R_K \le 
\sum_{k=1}^K  W_k
    +
\sum_{k=1}^K    \sum_{h=1}^{H-1} \xi_{h+1,k}\,.
\label{eq:rkwkxik}
\end{align}
\begin{lemma}\label{lem:wkbound}
Let $\alpha>0$ and $d = \dimE(\cF,\alpha)$
where $\cF$ is given by \eqref{eq:inducedF2}.
Then, for any nondecreasing sequence $(\beta_k^2)_{k=1}^K$,
on the event when $P\in \cap_{k\in [K]}\cB_k$,
\begin{align*}
\sum_{k=1}^K  W_k \le 
\alpha + H(d \wedge K(H-1)) + 4 \sqrt{d \beta_K K(H-1)}\,.
\end{align*}
\end{lemma}
\begin{proof}
Let $P\in \cap_{k\in [K]}\cB_k$ holds.
Using the notation of the previous section, letting $\tilde\cF_t = \cF_t(\beta_k)$ for $(k-1)(H-1)+1 \le t \le k(H-1)$,
we have
\begin{align*}
\sum_{k=1}^K W_k 
& \le
\sum_{k=1}^K  \sup_{\tilde P\in \cB_k} 
	\sum_{h=1}^{H-1} \left( f_{\tilde P}(s_h^k,a_h^k,V_{h+1,k}) - f_P(s_h^k,a_h^k,V_{h+1,k}) \right) \\
&  \le
\sum_{t=1}^{K(H-1)} \diam(\tilde\cF_t|_{X_t})
& \tag{because $P\in \cap_{k\in [K]}\cB_k$}
\\
& \le 
 \alpha + H(d \wedge K(H-1)) + 2 \delta_{K(H-1)} \sqrt{d K(H-1)}\,,
\end{align*}
where $X_t$ is defined in Section~\ref{sec:confapp} and
where the last inequality is by
Lemma \ref{lem:eluder}, which is applicable because $\cF \subset \cB_{\infty}(\cX,H)$ holds by choice,
and
$\delta_{K(H-1)} = \max_{1\le t \le K(H-1)} \diam(\tilde \cF_t|_{X_{1:t}})$.
Thanks to the definition of $\tilde \cF_t$, 
$\delta_{K(H-1)} \le 2 \sqrt{\beta_K}$. Plugging this into the previous display finishes the proof.
\end{proof}

\subsubsection{Proof of Theorem \ref{thm:mainbound}}
\begin{proof}
Note that for any $k\in [K]$ and $h\in [H-1]$,  $\xi_{h+1},k\in [-H,H]$. 
As noted beforehand, $\xi_{2,1},\xi_{3,1}$, $\dots,\xi_{H,1},\xi_{2,2},\xi_{3,2},\dots,\xi_{H,2},\xi_{2,3},\dots$ is a martingale difference sequence.
 Thus,
with probability $1-\delta$, 
$\sum_{k=1}^K    \sum_{h=1}^{H-1} \xi_{h+1,k} \le H \sqrt{2K(H-1)\log(1/\delta)}$.
Consider the event when this inequality holds and when
 $P\in \cap_{k\in [K]}\cB_k$.
By using Corollary \ref{cor:betakdef} and a union bound, this event holds with probability at least $1-2\delta$.
On this event, by~\eqref{eq:rkwkxik} and Lemma \ref{lem:wkbound}, we obtain
\begin{align*}
R_K \le 
\alpha + H(d \wedge K(H-1)) + 4 \sqrt{d \beta_K K(H-1)}
+H \sqrt{2K(H-1)\log(1/\delta)}\,.
\end{align*}
Using $\alpha\le 1$, which holds by assumption, 
finishes the proof. 
\end{proof}

\subsubsection{Proof of Corollary \ref{cor:linmix}}
\begin{proof}
Note that
\begin{align*}
\norm{f_{P'} - f_{P''}}_\infty
& = \sup_{s,a,v} | \int (P_a'(ds'|s) - P_a''(ds'|s)) v(s') |
\le H \sup_{s,a} \int |P_a'(ds'|s) - P_a''(ds'|s)| \\
& = H \sup_{s,a} \norm{P_a'(s) - P_a''(s)}_1 =: H \norm{P'-P''}_{\infty,1}\,.
\end{align*}
For $\alpha>0$ let $\cN(\cP,\alpha, \norm{\cdot}_{\infty,1})$ denote the $(\alpha, \norm{\cdot}_{\infty,1})$-covering number of $\cP$. Then we have
\[
\cN(\cF,\alpha, \norm{\cdot}_{\infty})\leq \cN(\cP,\alpha/H, \norm{\cdot}_{\infty,1}).
\]
Then, by Corollary \ref{cor:betakdef},
\begin{align*}
\beta_K =2H^2 \log( 2 \cN(\cF,\alpha,\norm{\cdot}_{\infty})/\delta) + C \leq 2H^2 \log( 2 \cN(\cP,\alpha/H,\norm{\cdot}_{\infty,1})/\delta) + C
\end{align*}
with some universal constant $C>0$.
Let $f: (\Theta,\norm{\cdot}) \to (\cP, \norm{\cdot}_{\infty,1})$  be defined by $\theta \mapsto \sum_j \theta_j P_j$.
Note that $\norm{f(\theta)-f(\theta')}_{\infty,1} 
\le
\sup_{s,a}
\sum_j \norm{ (\theta_j-\theta_j') P_{j,a}(s)}_1
=
\sum_j |\theta_j -\theta_j'|  = \norm{\theta-\theta'}_1$.
Hence, any $(\epsilon,\norm{\cdot}_1)$ covering of $\Theta$ induces an $(\epsilon,\norm{\cdot}_{\infty,1})$-covering of $\cP$ and so 
$\cN(\cP,\alpha/H,\norm{\cdot}_{\infty,1})
\le
\cN(\Theta,\alpha/H,\norm{\cdot}_1) \le C' (R H/\alpha)^d
$
with some universal constant $C'>0$. 

Now, choose $1/\alpha = K \sqrt{\log(KH/\delta)}$.
Hence, 
\begin{align*}
\beta_K \leq 2H^2 (\log(2C'/\delta)+ d \log( RH/\alpha) ) + C\,.
\end{align*}
Suppressing $\log$ factors (e.g., $\log(RH)$), $\log\log$ terms and constants, 
we have $\beta_K = H^2 (d + \log(1/\delta))$.

Let $\cF$ be given by \eqref{eq:inducedF2}.
We now bound $\dimE(\cF,\alpha)$.
Let $\cX = \cS \times \cA \times B(\cS)$ as before.
Define $z:\cS \times \cA \times B(\cS) \to \R^d$ using 
$z(s,a,v)_j = \ip{P_{j,a}(s),v}$
and note that if $x\in \cX$  
is $(\epsilon,\cF)$-independent of 
$x_1,\dots,x_k\in \cX$ 
then 
$z(x)\in \R^d$
is $(\epsilon,\Theta)$-independent of 
$z(x_1),\dots,z(x_k)\in \R^d$.
This holds because if $P = \sum_j \theta_j P_j\in \cP$ then $f_P(s,a,v) = \ip{\theta,z(s,a,v)}$ for any $(s,a,v)\in \cX$.
Hence, $\dimE(\cF,\alpha) \le \dimE(\mathrm{Lin}(\cZ,\Theta),\alpha)$, where
 $\mathrm{Lin}(\cZ,\Theta)$ is the set of linear maps
with domain $\cZ = \{ z(x)\,:\, x\in \cX\} \subset \R^d$ 
and parameter from $\Theta$:
$ \mathrm{Lin}(\cZ,\Theta)
 = \{ h \,:\, h: \cZ \to \R \text{ s.t. } \exists \theta\in \Theta: 
h( z ) = \ip{\theta,z}, z\in \cZ \}$.
Now, by Proposition~11 of \citet{RuVR14}, $\dimE(\mathrm{Lin}(\cZ,\Theta),\alpha) = O(d \log(1+(S\gamma/\alpha)^2 )$ where $S$ is the $\norm{\cdot}_2$ diameter of $\Theta$ and $\gamma = \sup_{z\in \cZ} \norm{z}_2$.
We have
\begin{align*}
\norm{z}_2^2 = \sum_j (\ip{P_{j,a}(s),v})^2 \le H^2 d\,,
\end{align*}
hence $\gamma \le H \sqrt{d}$. By the relation between the $1$ and $2$ norms, the $2$-norm
diameter of $\Theta$ is at most $\sqrt{d}R$. Dropping $\log$ terms, $\dimE(\cF,\alpha) = \tilde O(d)$.

Plugging into Theorem \ref{thm:mainbound} 
gives the desired result.
\end{proof}

\subsection{Proof of Theorem \ref{thm:nonlscs}}
\label{sec:proofofnonlscs}

Recall the following:
\begin{definition}
A random variable $X$ is $\sigma$-subgaussian if for all $\lambda \in \R$, it holds that $\EE[\exp(\lambda X)] \leq \exp\left(\lambda^2 \sigma^2 / 2\right)$. 
\end{definition}

The proof of the next couple of statements is standard and is included only for completeness.
\begin{theorem}
\label{thm:subgaussian-tail}
If $X$ is $\sigma$-subgaussian, then for any $\lambda > 0$, with probability at least $1-\delta$,
\begin{align}
X < \frac{1}{\lambda}\, \log\left(\frac1\delta\right) + \lambda\, \frac{\sigma^2}2\,.
\label{eq:subgausstail}
\end{align}
\end{theorem} 
\begin{proof}
Let $\lambda>0$.
We have, $\{X\ge \epsilon\} = \{ \exp(\lambda (X-\epsilon) ) \ge 0 \}$. Hence, Markov's inequality gives $\Prob{X \ge \epsilon} \le \exp(-\lambda \epsilon) \EE[ \exp(\lambda X) ]
\le \exp(-\lambda \epsilon+\frac12 \lambda^2 \sigma^2 )$.
Equating the right-hand side with $\delta$ and solving for $\epsilon$, we get that $\log(\delta) = -\lambda \epsilon + \frac12 \lambda^2 \sigma^2$. Solving for $\epsilon$ gives
$\epsilon = \log(1/\delta)/\lambda + \frac{\sigma^2}2 \lambda$, finishing the proof.
\end{proof}
Choosing the $\lambda$ that minimizes the right-hand side of the bound gives the usual form:
\begin{align}
\Prob{ X \ge \sqrt{2\sigma^2 \log(1/\delta)} } \le \delta\,.
\label{eq:basicsg}
\end{align}

\begin{lemma}[Lemma~5.4 of \citet{lattimore2018bandit}]\label{lem:subgaussian-properties}
Suppose that $X$ is $\sigma$-subgaussian and 
$X_1$ and $X_2$ are independent and $\sigma_1$ and $\sigma_2$-subgaussian, respectively, then:
\begin{enumerate}
\item $\E[X] = 0$. 
\item $cX$ is $|c| \sigma$-subgaussian for all $c \in \R$.
\item $X_1 + X_2$ is $\sqrt{\sigma_1^2 + \sigma_2^2}$-subgaussian. \label{lem:sg:3}
\end{enumerate}
\end{lemma} 

Let $(Z_p)_p$ be an $\FF = (\FF_p)_p$-adapted process.
Recall that $(Z_p)_p$ is conditionally $\sigma$-subgaussian given $\FF$ if for all $p\ge 1$,
\begin{align*}
\log \EE[ \exp(\lambda Z_p )|\FF_{p-1}] \le \frac{1}{2} \lambda^2 \sigma^2\,, \quad \text{ for all } \lambda\in \R\,.
\end{align*}
A standard calculation gives that $S_t = \sum_{p=1}^t Z_p$ is $\sqrt{t} \sigma$-subgaussian 
(essentially,  a refinement of the calculation that is need to show Part~\eqref{lem:sg:3} of  Lemma~\ref{lem:subgaussian-properties})
and thus, in particular, for any $t\ge 1$ and $\lambda>0$, with probability $1-\delta$,
\begin{align*}
S_t < \frac{1}{\lambda} \, \log\left(\frac1{\delta}\right) + \lambda\, \frac{t \sigma^2}{2}\,.
\end{align*}
In fact, by slightly strengthening the argument, 
one can show that the above inequality holds simultaneously for all $t\ge 1$:
\begin{theorem}[E.g., Lemma~7 of \citet{RuVR14}]
\label{thm:subgaussian-tail2}
Let $\FF$ be a filtration and let 
$(Z_p)_p$ be an $\FF$-adapted, conditionally $\sigma$-subgaussian process. Then
for any $\lambda > 0$, with probability at least $1-\delta$, for all $t\ge 1$,
\begin{align}
S_t < \frac{1}{\lambda}\, \log\left(\frac1\delta\right) + \lambda\, \frac{t \sigma^2}2\,,
\label{eq:subgausstail2}
\end{align}
where $S_t = \sum_{p=1}^t Z_p$.
\end{theorem} 

\paragraph{Proof of Theorem \ref{thm:nonlscs}}
Let us introduce the following helpful notation: For vectors $x,y\in \R^t$, let
$\ip{x,y}_t = \sum_{p=1}^t x_p y_p$, 
$\norm{x}_t^2 = \ip{x,x}_t$, and for $f:\cX \to \R$,
$\norm{f}_t^2= \sum_{p=1}^t f^2(X_p)$. More generally, we will overload addition and subtraction such that for $x\in \R^t$, $x+f\in \R^t$ is the vector whose $p$th coordinate is $x_p + f(X_p)$ ($x_p$ and $X_p$ both appear on purpose here). We also overload $\ip{\cdot,\cdot}_t$ such that $\ip{x,f}_t= \ip{f,x}_t = \sum_{p=1}^t x_p f(X_p)$.

Define $Z_p$ using $Y_p = f_*(X_p)+Z_p$ and collect $(Y_p)_{p=1}^t$ and $(Z_p)_{p=1}^t$ into the vectors $Y$ and $Z$.
As in the statement of the theorem,
let $\FF = (\FF_p)_{p=0,1,\dots}$ be such that for any $s\ge 1$, $(X_1,Y_1,\dots,X_{p-1},Y_{p-1},X_p)$ is $\FF_{p-1}$-measurable. Note that for any $p\ge 1$, $Z_p = Y_p-f_*(X_p)$ is $\FF_p$-measurable, hence
$(Z_p)_{p\ge 1}$ is $\FF$-adapted.

With this, elementary calculation  gives
\begin{align*}
\norm{Y - f}_t^2 - \norm{Y-f_*}_t^2 = \norm{f_*-f}_t^2 + 2 \ip{Z,f_*-f}_t\,.
\end{align*}
Splitting $\norm{f_*-f}_t^2$ and rearranging gives
\begin{align}
\label{eq:corenlls}
\frac12 \norm{f_*-f}_t^2
= 
\norm{Y - f}_t^2 - \norm{Y-f_*}_t^2 + E(f)
\end{align}
where
\begin{align*}
E(f) =  - \frac12 \norm{f_*-f}_t^2 +  2 \ip{Z,f-f_*}_t \,.
\end{align*}

Recall that $\hat f_t =\argmin_{f\in \cF} \norm{Y-f}_t^2$.
Plugging $\hat f_t$ into \ref{eq:corenlls} in place of $f$ and using that 
thanks to $f_*\in \cF$, $\norm{Y-\hat f_t}_t^2 \le \norm{Y-f_*}_t^2$, we get
\begin{align}
\frac12 \norm{f_*-\hat f_t}_t^2 \le E(\hat f_t)\,.
\label{eq:halfnormbound}
\end{align}
Thus, it remains to bound $E(\hat f_t)$.
For this fix some $\alpha>0$ to be chosen later and let $\cG(\alpha) \subset \cF$ be an $\alpha$-cover of $\cF$ in $\norm{\cdot}_\infty$. Let $g\in \cG(\alpha)$ be a random function, also to be chosen later.
We have
\begin{align}
E(\hat f_t) =
E(\hat f_t) - E(g)+
 E(g) 
\le 
E(\hat f_t) - E(g)+
\max_{\tilde g\in \cG(\alpha)} E(\tilde g)  \label{eq:Ehatfbound}
\end{align}

We start by bounding the last term above.
A simple calculation
gives that for any fixed $f\in \cF$, w.p. $1-\delta$, $2\ip{Z,f-f_*}_t$ is $2 \sigma \norm{f-f_*}_t$-subgaussian.
Hence, with probability $1-\delta$, simultaneously for all $t\ge 1$,
\begin{align*}
E(f) \le - \frac12 \norm{f_*-f}_t^2 +  
\frac{1}{\lambda}\, \log\left(\frac1\delta\right) + \lambda\, \frac{4\sigma^2\norm{f-f_*}_t^2 }2
= 4 \sigma^2\log\left(\frac1\delta\right) \,,
\end{align*}
where the equality follows by choosing $\lambda = 1/(4\sigma^2)$ (which makes the first and last terms cancel).
(Note how splitting $\norm{f-f_*}_t^2$ into two halves allowed us to bound the ``error term'' $E(f)$ independently of $t$.)
Now, by a union bound, it follows that with probability at least $1-\delta$, 
the second term is bounded by $4 \sigma^2 \log(|\cG(\alpha)|/\delta)$.

Let us now turn to bounding the first term.
We calculate
\begin{align*}
E(\hat f_t) - E(g)
& =\frac12 \norm{g - f_*}_t^2 -\frac12 \norm{\hat f_t - f_*}_t^2  + 2 \ip{Z, \hat f_t- g}_t\\
& \le \frac12 \left(\ip{g-\hat f_t,g+\hat f_t+2 f_*}_t  \right)
+ 2 \norm{Z}_t \norm{\hat f_t-g}_t \\
& \le \frac12 4C \alpha \,t + 2 \norm{Z}_t \alpha \sqrt{t} \,,
\end{align*}
where for the last inequality we chose
 $g = \argmin_{\tilde g\in \cG(\alpha)} \norm{\hat f_t- \tilde g}_\infty$ so that $\norm{\hat f_t - g}_t\le \alpha \sqrt{t}$
 and used Cauchy-Schwartz, together with that
 $\norm{g}_t,\norm{\hat f_t}_t, \norm{f_*}_t\le C\sqrt{t}$, which follows from $g,\hat f_t,f_*\in \cF$ and that by assumption all functions in $\cF$ are bounded by $C$.

It remains to bound $\norm{Z}_t$.
For this, we observe that with probability $1-\delta$, simultaneously for all $t\ge 1$,
\begin{align*}
\norm{Z}_t \le \sigma \sqrt{ 2 t \log(2t(t+1)/\delta)}\,.
\end{align*}
Indeed,
this follows because with probability $1-\delta$,
simultaneously for any $s\ge 1$, $|Z_p| ^2\le 2 \sigma^2 \log(2s(s+1)/\delta)$ holds because of a union bound and
Eq.~\eqref{eq:basicsg}.
Therefore, for the above choice $g$, with probability $1-\delta$,
simultaneously for all $t\ge 1$, it holds that
\begin{align*}
E(\hat f_t) - E(g) \le
2 C \alpha \,t + 2t \alpha \sqrt{\sigma^2 \log(2 t(t+1)/\delta)}\,.
\end{align*}

Merging this with Eqs.~\eqref{eq:halfnormbound} and \eqref{eq:Ehatfbound} and with another union bound, we get that with probability $1-\delta$, for any $t\ge 1$,
\begin{align*}
\norm{f_*-\hat f_t}_t^2 \le 8 \sigma^2 \log(2N_\alpha/\delta) +  4 t \alpha \left(C+\sqrt{\sigma^2 \log(4 t(t+1)/\delta)}\right)\,,
\end{align*}
where $N_\alpha$ is the $(\alpha,\norm{\cdot}_\infty)$-covering number of $\cF$.
\qed

\section{Proof of Theorem \ref{thm2}}\label{proof-thm2}
In this section we establish a regret lower bound by reduction to a known result for tabular MDP.
\begin{proof}
	We assume without loss of generality that $d$ is a multiple of 4 and $d\ge 8$. We set $S = 2$ and $A = d/4\ge 2$. According to \cite{azar2017minimax}, \cite{osband2016lower}, there exists an MDP $\mathcal{M}(\mS, \mA, P, r, H)$ with $S$ states, $A$ actions and horizon $H$ such that any algorithm has regret at least $\Omega(\sqrt{HSAT})$. In this case, we have $|\mS\times\mA\times\mS|=d$. We use $\sigma(s, a, s')$ to denote the index of $(s, a, s')$ in $\mS\times\mA\times\mS$. Letting
	\begin{equation*}
		P_{i}(s'|s, a) = \begin{cases}
			1 & \quad \text{if } \sigma(s, a, s') = i,\\
			0 & \quad \text{otherwise},
		\end{cases}
	\end{equation*}
	and $\theta^{i} = P(s'|s, a)$ if $\sigma(s, a, s') = i$, we will have
$
		P(s'|s, a) = \sum_{i=1}^{d}\theta^{i}P_{i}(s'|s, a).
$
	Therefore $P$ can be parametrized using \eqref{eq-model-linear}. Therefore, the known lower bound $\Omega(\sqrt{HSAT})$ implies a worst-case lower bound of $\Omega(\sqrt{H\cdot d/2\cdot T}) = \Omega(\sqrt{HdT})$ for our model.

\end{proof}

\section{Implementation} \label{sec:implement}

\subsection{Analysis of Implemented Confidence Bounds}\label{YasinAnalysis} In the implementation of UCRL-VTR used in Section \ref{sec:experiment}, we used different confidence intervals then the ones stated in the paper. The confidence intervals used in our implementation are the ones introduced in \cite{abbasi2011improved}. These confidence intervals are much tighter in the linear setting than the ones introduced in Section \ref{section 3} and thus have better practical performance. The purpose of this section is to formally introduce the confidence intervals used in our implementation of UCRL-VTR as well as show how these confidence intervals were adapted from the linear bandit setting to the linear MDP setting.
\subsubsection{Linear MDP Assumptions}
For our implementation of UCRL-VTR we used different confidence then was introduced in the paper. These are the tighter confidence bounds from the seminal work done by \cite{abbasi2011improved} and further expanded upon in Chapter 20 of \cite{lattimore2018bandit}. Now we will state some assumptions in the MDP setting, then we will state the equivalent assumptions from the linear bandit setting, and lastly we will make the connections between the two that allow us to use the confidence bounds from the linear bandit setting in the RL setting.
\begin{enumerate}\label{MDP_assumptins}
    \item $P^*(s' \mid s,a) = \sum_{i=1}^d (\theta_*^{\textit{MDP}})_{i} P_i(s' \mid s,a)$
    \item $s_{h+1}^k \sim P^*(\cdot \mid s_h^k,a_h^k)$
    \item $\mathcal{C}_t^{\textit{MDP}} = \{\theta^{\textit{MDP}} \in \mathbb{R}^d :\norm{\theta^{\textit{MDP}} - \hat{\theta}_t^{\textit{MDP}}}_{M_k} \leq \beta_t\}$
\end{enumerate}
where $t$ is defined in the table of \ref{sec:confapp}. Also note that in this section $(\cdot)_*$ denotes the true parameter or model, $(\cdot)^{\textit{MDP}}$ denotes something derived or used in the linear MDP setting, and $(\cdot)^{\textit{LIN}}$ denotes something derived or used in the linear bandit setting. Now, under 1-3 of \ref{MDP_assumptins} we hope to construct a confidence set $\mathcal{C}_t^{\textit{MDP}}$ such that
\begin{align*}
    \theta^{\textit{MDP}} \in \bigcap_{t=1}^\infty \mathcal{C}_t^{\textit{MDP}}
\end{align*}
with high probability.
Now the choice of how to choose both $\mathcal{C}_t^{\textit{MDP}}$ and $\beta_t$ comes from the linear bandit literature. We will introduce the necessary theorems and assumptions to derive both $\mathcal{C}_t^{\textit{LIN}}$ and $\beta_t$ in the linear bandit setting and then adapt the results from the linear bandit setting to the linear MDP setting.
\subsubsection{Tighter Confidence Bounds for Linear Bandits}
The following results are introduced in the paper by \cite{abbasi2011improved} and are further explained in Chapter 20 of the book by \cite{lattimore2018bandit}. In this section, we will introduce the theorems and lemmas that allows us to derive tighter confidence intervals for the linear bandit setting. Then we will carefully adapt the confidence intervals to the linear bandit setting. Now supposed a bandit algorithm has chosen actions $A_1,...,A_t \in \mathbb{R}^d$ and received rewards $X_1^{\textit{LIN}},...,X_t^{\textit{LIN}}$ with $X_s^{\textit{LIN}} = \langle A_t, \theta_*^{\textit{LIN}} \rangle + \eta_s$ where $\eta_s$ is some zero mean noise. The least squares estimator of $\theta_*^{\textit{LIN}}$ is the minimizer of the following loss function
\begin{align*}
    L_t(\theta^{\textit{LIN}}) = \sum_{s=1}^t (X_s^{\textit{LIN}} - \langle A_t, \theta^{\textit{LIN}} \rangle)^2 + \lambda \norm{\theta^{\textit{LIN}}}_2^2
\end{align*}
where $\lambda > 0$ is the regularizer. This loss function is minimized by
\begin{gather*}
    \hat{\theta}_t^{\textit{LIN}} = W_t^{-1} \sum_{s=1}^t X_s^{\textit{LIN}} A_s \text{ with } W_t = \lambda I + \sum_{s=1}^t A_s A_s^{\top}
\end{gather*}
notice how this linear bandit problem is very similar to the linear MDP problem introduced in section 3 of our paper. In our linear MDP setting, it is convenient to think of $M$ and $W$ as serving equivalent purposes (storing rank one updates) thus it is also convenient to think of $A_t$ and $X_t^{\textit{MDP}}$ as serving equivalent purposes (the features by which we use to make our predictions), where $X_t^{\textit{MDP}}$ is defined in section 3 of our paper with some added notation to distinguish it from the $X_t^{\textit{LIN}}$ used here in the linear bandit setting. We will now build up some intuition by making some simplifying assumptions.
\begin{enumerate}
    \item No regularization: $\lambda = 0$ and $W_t$ is invertible.
    \item Independent subgaussian noise: $(\eta_s)_s$ are independent and $\sigma$-subgaussian
    \item Fixed Design: $A_1,...,A_t$ are deterministically chosen without the knowledge of $X_1^{\textit{LIN}},...,X_t^{\textit{LIN}}$
\end{enumerate}
finally it is also convenient to think of $X_t^{\textit{LIN}}$ and $V_{t+1}(s_{t+1})$ as serving equivalent purposes (the target of our predictions). Thus the statements we prove in the linear bandit setting can be easily adapted to the linear MDP setting. While none of the assumptions stated above is plausible in the bandit setting, the simplifications eases the analysis and provides insight.
\\
\\
Comparing $\theta_*^{\textit{LIN}}$ and $\hat{\theta}_t^{\textit{LIN}}$ in the direction $x \in \mathbb{R}^d$, we have
\begin{gather*}
    \langle \hat{\theta}_t^{\textit{LIN}} - \theta_*^{\textit{LIN}}, x \rangle = \left\langle x, W_t^{-1} \sum_{s=1}^t A_s X_s^{\textit{LIN}} - \theta_*^{\textit{LIN}} \right\rangle =  \left\langle x, W_t^{-1} \sum_{s=1}^t A_s (A_s^\top \theta_*^{\textit{LIN}} + \eta_s) - \theta_*^{\textit{LIN}} \right\rangle \\ = \left\langle x, W_t^{-1} \sum_{s=1}^t A_s \eta_s \right\rangle = \sum_{s=1}^t \langle x, W_t^{-1} A_s \rangle \eta_s
\end{gather*}
Since $(\eta_s)_s$ are independent and $\sigma$-subgaussian, by Lemma 5.4 and Theorem 5.3 (need to be stated),
\begin{gather*}
    \mathbb{P}\left( \langle \hat{\theta}_t^{\textit{LIN}} - \theta_*^{\textit{LIN}}, x \rangle \geq \sqrt{2\sigma^2  \sum_{s=1}^t \langle x, W_t^{-1} A_s \rangle^2 \log\left(\frac{1}{\delta}
    \right)} \right) \leq \delta
\end{gather*}
A little linear algebra shows that $ \sum_{s=1}^t \langle x, W_t^{-1} A_s \rangle^2 = \norm{x}_{W_t^{-1}}^2$ and so,
\begin{gather}
    \mathbb{P}\left( \langle \hat{\theta}_t^{\textit{LIN}} - \theta_*^{\textit{LIN}}, x \rangle \geq \sqrt{2\sigma^2 \norm{x}_{W_t^{-1}}^2 \log\left(\frac{1}{\delta}
    \right)} \right) \leq \delta
\end{gather}
We now remove the limiting assumptions we stated above and use the newly stated assumptions for the rest of this section
\begin{enumerate}
    \item There exists a $\theta_*^{\textit{LIN}} \in \mathbb{R}^d$ such that $X_t^{\textit{LIN}} = \langle\theta_*^{\textit{LIN}}, A_t \rangle + \eta_t$ for all $t \geq 1$.
    \item The noise is conditionally $\sigma$-subgaussian:
    \begin{gather*}
        \text{for all $\alpha \in \mathbb{R}$ and $t \geq 1$,  } \mathbb{E}[\exp(\alpha \eta_t) \mid \mathcal{F}_{t-1}] \leq \exp\left(\frac{\alpha \sigma^2}{2}\right) a.s.
    \end{gather*}
    where $\mathcal{F}_{t-1}$ is such that $A_1,X_1^{\textit{LIN}},...,A_{t-1},X_{t-1}^{\textit{LIN}}$ are $\mathcal{F}_{t-1}$-measurable.
    \item In addition, we now assume $\lambda > 0$.
\end{enumerate}
The inclusion of $A_t$ in the definition of $\mathcal{F}_{t-1}$ allows the noise to depend on past choices, including the most recent action.
Since we want exponentially decaying tail probabilities, one is tempted to try the Cramer-Chernoff method:
\begin{gather*}
    \mathbb{P}(\norm{\hat{\theta}_t^{\textit{LIN}} - \theta_*^{\textit{LIN}}}_{W_t}^2 \geq u^2) \leq \inf_{\alpha > 0} \mathbb{E}\left[\exp\left(\alpha \norm{\hat{\theta}_t^{\textit{LIN}} - \theta_*^{\textit{LIN}}}_{W_t}^2 - \alpha u^2\right) \right].
\end{gather*}
Sadly, we do not know how to bound this expectation. Can we still somehow use the Cramer–Chernoff method? We take inspiration from looking at the special case of $\lambda = 0$ one last time, assuming that $W_t = \sum_{s=1}^t A_s A_s^\top$ is invertible. Let
\begin{gather*}
    S_t = \sum_{s=1}^t \eta_s A_s
\end{gather*}
Recall that $\hat{\theta}_t^{\textit{LIN}} = W_t^{-1} \sum_{s=1}^t X_s^{\textit{LIN}} A_s = \theta_*^{\textit{LIN}} + W_t^{-1}S_t$. Hence,
\begin{gather*}
    \frac{1}{2} \norm{\hat{\theta}_t^{\textit{LIN}} - \theta_*^{\textit{LIN}}}_{W_t}^2 = \frac{1}{2} \norm{S_t}_{W_t^{-1}}^2 = \max_{x \in \mathbb{R}^d} \left(\langle x,S_t \rangle - \frac{1}{2}\norm{x}_{W_t}^2\right).
\end{gather*}
The next lemma shows that the exponential of the term inside the maximum is a supermartingale even when $\lambda \geq 0$.
\begin{lemma}\label{lemma202}
For all $x \in \mathbb{R}^d$ the process $D_t(x) = \exp(\langle x,S_t \rangle - \frac{1}{2}\norm{x}_{W_t^2})$ is an $\mathbb{F}$-adapted non-negative supermartingale with $D_0(x) \leq 1$.
\end{lemma}
\noindent The proof for this Lemma can be found in Chapter 20 of the book by \cite{lattimore2018bandit}.
For simplicity, consider now again the case when $\lambda = 0$. Combining the lemma and the linearisation idea almost works. The Cramer–Chernoff method leads to
\begin{gather}\label{205}
    \mathbb{P}\left( \frac{1}{2} \norm{\hat{\theta}_t^{\textit{LIN}} - \theta_*^{\textit{LIN}}}_{W_t}^2 \geq \log(1/\delta) \right) = \mathbb{P}\left( \exp \left(\max_{x \in \mathbb{R}^d} \left(\langle x,S_t \rangle - \frac{1}{2}\norm{x}_{W_t}^2\right) \right) \geq \log(1/\delta) \right) \\ \leq \delta \mathbb{E}\left[ \exp \left(\max_{x \in \mathbb{R}^d} \left(\langle x,S_t \rangle - \frac{1}{2}\norm{x}_{W_t}^2\right) \right) \right] = \delta \mathbb{E} \left[ \max_{x \in \mathbb{R}^d} D_t(x) \right]
\end{gather}
Now Lemma \ref{lemma202} shows that $\mathbb{E}[D_t(x)] \leq 1$. Now using Laplace's approximation we write
\begin{gather*}
    \max_x D_t(x) \approx \int_{\mathbb{R}^d} D_t(x) dh(x),
\end{gather*}
where $h$ is some measure on $\mathbb{R}^d$ chosen so that the integral can be calculated in closed form. This is not a requirement of the method, but it does make the argument shorter. The main benefit of replacing the maximum with an integral is that we obtain the following lemma
\begin{lemma}\label{203}
Let h be a probability measure on $\mathbb{R}^d$; then; $\Bar{D}_t = \int_{\mathbb{R}^d} D_t(x) dh(x)$ is an $\mathbb{F}$-adapted non-negative supermartingale with $\Bar{D}_0 = 1$.
\end{lemma}
\noindent The proof of Lemma \ref{203} can, again, be found in Chapter 20 of the book by \cite{lattimore2018bandit}. Now the following theorem is the key result from which the confidence set will be derived.
\begin{theorem}\label{204}
For all $\lambda > 0$, and $\delta \in (0,1)$
\begin{gather*}
    \mathbb{P}\left(exists \ t \in \mathbb{N} : \norm{S_t}_{W_t^{-1}}^2 \geq 2 \sigma^2 \log \left( \frac{1}{\delta} \right) + \log \left(\frac{\det W_t}{\lambda^d}\right) \right) \leq \delta
\end{gather*}
Furthermore, if $\norm{\theta_*^{\textit{LIN}}}_2 \leq m_2$, then $\mathbb{P}(exists \ t \in \mathbb{N}^+ : \theta_*^{\textit{LIN}} \notin \mathcal{C}_t^{\textit{LIN}}) \leq \delta$ with
\begin{gather*}
    \mathcal{C}_t^{\textit{LIN}} = \left\{\theta \in \mathbb{R}^d : \norm{\hat{\theta}_{t-1}^{\textit{LIN}} - \theta}_{W_{t-1}} < m_2\sqrt{\lambda} + \sqrt{2 \sigma^2 \log \left(\frac{1}{\delta}\right) + \log \left(\frac{W_{t-1}}{\lambda^d}\right)}\right\}.
\end{gather*}
\end{theorem}
\noindent The proof of Theorem \ref{204} can be found in Chapter 20 of the book by \cite{lattimore2018bandit}.
\subsubsection{Adaptation of the Confidence Bounds to our Linear MDP Setting}
Now with the Lemmas and Theorems introduced in the previous section we are ready to derive the confidence bounds used in our implementation of UCRL-VTR. Now using the notation from the linear bandit setting we set
\begin{enumerate}
    \item The target $X_t^{\textit{MDP}} = \int_{j} V_t(s')P_j(ds'\mid s_t,a_t)$
    \item $Y_t = V_t(s_{t+1})$
    \item $\mathcal{F}_{t-1} = \sigma(s_1,a_1,...,s_{t-1},a_{t-1})$, which just means the filtration is set to be the sigma-algebra generated by all past states and actions observed.
    \item $\eta_t = Y_t - \langle X_t^{\textit{MDP}}, \theta_*^{\textit{MDP}} \rangle 
    =  V_t(s_{t+1}) - \int_{j} V_t(s')P_j^*(ds'\mid s_t,a_t)$, 
    since $\theta_*^{\textit{MDP}}$ 
    is the true model of the MDP.
    \item $M_t$ in the linear MDP setting is defined equivalently to $W_t$ in the linear bandit setting, i.e. they are both the sums of a regularizer term and a bunch of rank one updates.
\end{enumerate}
it can be seen that our the noise in our system $\eta_t$ has zero mean $\mathbb{E}[\eta_t \mid \mathcal{F}_{t-1}] = 0$ finally the noise in our system has variance $H/2$ thus our system in $H/2$-subgaussian.
\begin{lemma}\label{rangesg}(Hoeffding's lemma)
Let $Z = Z - \mathbb{E}[Z]$ be a real centered random variable such that  $Z \in [a,b]$ almost surely. Then $\mathbb{E}[\exp(\alpha Z)] \leq \exp(\alpha^2 \frac{(b-a)^2}{8})$ for any $\alpha \in \mathbb{R}$ or $Z$ is subgaussian with variance $\sigma^2 = \frac{(b-a)^2}{4}$.
\end{lemma}
\textit{Proof} 
Define $\psi(\alpha) = \log \mathbb{E}[\exp(\alpha Z)]$ we can then compute
\begin{gather*}
    \psi ' (\alpha) = \frac{\mathbb{E}[Z \exp(\alpha Z)]}{\mathbb{E}[\exp{(\alpha Z)}]}, \ \  \psi '' (\alpha) =  \frac{\mathbb{E}[Z^2 \exp(\alpha Z)]}{\mathbb{E}[\exp{(\alpha Z)}]} - \left( \frac{\mathbb{E}[Z \exp(\alpha Z)]}{\mathbb{E}[\exp{(\alpha Z)}]}\right)^2
\end{gather*}
Thus $\psi '' (\alpha)$ can be interpreted as the variance of the random variable $Z$ under the probability measure $d\mathbb{Q} = \frac{\exp(\alpha Z)}{\mathbb{E}[\exp(\alpha Z)]}d\mathbb{P}$, but since $Z \in [a,b]$ almost surely, we have, under any probability
\begin{gather*}
    \text{var}(Z) = \text{var}(Z - \frac{a+b}{2}) \leq \mathbb{E}\left[\left(Z - \frac{a+b}{2}\right)^2\right] \leq \left(\frac{b-a}{4}\right)^2
\end{gather*}
The fundamental theorem of calculus yields
\begin{gather*}
    \psi(\alpha) = \int_0^s \int_0^\mu \psi(\rho) d\rho d\mu = \frac{s^2 (b-a)^2}{8}
\end{gather*}
using $\psi(0) = \log 1 = 0$ and $\psi'(0) = \mathbb{E}[Z] = 0$. \qed
\\
\\
\noindent Now using Lemma \ref{rangesg} and the fact that $Y_t$ is bounded in the range of $[0,H]$, $\mathbb{E}[Y_t] = \langle X_t^{\textit{MDP}}, \theta_*^{\textit{MDP}} \rangle$, and $\eta_t = Y_t - \langle X_t^{\textit{MDP}}, \theta_*^{\textit{MDP}} \rangle = Y_t - \mathbb{E}[Y_t]$, the noise $\eta_t$ in our linear MDP setting is $H/2$-subgaussian. This result is also stated in a proof from \ref{sec:confapp}.
\\
\\
Putting this all together we can derive the tighter confidence set for UCRL-VTR in the linear setting,
\begin{gather*}
    \mathcal{C}_t^{\textit{MDP}} = \left\{\theta \in \mathbb{R}^d : \norm{\hat{\theta}_{t-1}^{\textit{MDP}} - \theta}_{M_{t-1}} < m_2\sqrt{\lambda} + \frac{H}{2}\sqrt{2 \log \left(\frac{1}{\delta}\right) + \log \left(\frac{M_{t-1}}{\lambda^d}\right)}\right\}.
\end{gather*}
where here in the linear MDP setting $M_t$ replaces $W_t$ from the linear bandit setting and $\norm{\theta_*^{\textit{MDP}}}_2 \leq m_2$. The justification of using these bounds in the linear MDP setting follow exactly from the justification given above for using these bounds in the linear bandit setting.

\subsection{UCRL-VTR}

In the proceeding subsections we discuss the implementation of the algorithms studied in Section \ref{sec:experiment} of the paper. The first algorithm we present is the algorithm used to generate the results for UCRL-VTR.

\begin{algorithm}[H]
	\caption{UCRL-VTR with Tighter Confidence Bounds}
	\label{LinUCRL-VTR}
	\begin{algorithmic}[1]
		\STATE \textbf{Input: } MDP, $d, H, T=KH$;
		\STATE \textbf{Initialize: } $M_{1,1}\leftarrow I$, $\quad w_{1, 1}\leftarrow 0\in\mathbb{R}^{d\times 1}, \quad \theta_{1}\leftarrow M_{1,1}^{-1}w_{1, 1}\qquad$ for $1\le h\le H$, $d_1 = |\mathcal{S}| \times |\mathcal{A}|$;
		\STATE \textbf{Initialize: } $\delta\leftarrow 1/K$, and for $1\le k\le K$,
		\STATE Compute Q-function $Q_{h, 1}$ using $\theta_{1, 1}$ according to \eqref{eq-q};
		\FOR{$k = 1:K$}
			\STATE Obtain initial state $s_{1}^{k}$ for episode $k$;
			\FOR{$h = 1:H$}

				\STATE Choose action greedily by $$a_{h}^{k} = \arg\max_{a\in\mA}Q_{h, k}(s_{h}^{k}, a)$$ and observe the next state $s_{h+1}^{k}$.

				 \STATE Compute the predicted value vector: \hfill{\textcolor{blue!80}{$\triangleright$ Evaluate the expected value of next state}}
 \\
				 \qquad\quad
	 			\begin{equation*}
					\begin{aligned}
						X_{h,k}&\leftarrow \mathbb{E}_{\bigcdot}[V_{h+1, k}(s) |s_{h}^{k}, a_{h}^{k}]
						& = \sum_{s\in\mS}V_{h+1, k}(s)\cdot P_{\bigcdot}(s|s_{h}^{k}, a_{h}^{k}).
					\end{aligned}
				\end{equation*}
	 \STATE $y_{h,k} \leftarrow V_{h+1, k}(s_{h+1}^{k})$		  \hfill{\textcolor{blue!80}{$\triangleright$ Update regression parameters}}
				 \STATE $M_{h+1,k}\leftarrow M_{h,k} + X_{h,k} X_{h,k}^{\top}$
				 \STATE $w_{h+1,k}\leftarrow w_{h,k} + y_{h,k} \cdot X_{h,k}$

			\ENDFOR
			\STATE Update at the end of episode:\hfill{\textcolor{blue!80}{$\triangleright$ Update Model Parameters}}
			\begin{equation*}
				\begin{aligned}
					M_{1,k+1} & \gets M_{H+1,k},\\
					w_{1,k+1} & \gets w_{H+1,k},\\
					\theta_{k+1} &\leftarrow M_{1, k+1}^{-1}w_{1,k+1};
				\end{aligned}
			\end{equation*}
			\STATE Compute $Q_{h, k+1}$ for $h=H,\ldots,1,$ using $\theta_{k+1}$ according to (\ref{UCRL_VTR-Q-update}) using
		    \begin{gather*}
			\sqrt{\beta_{h,k}}\leftarrow \sqrt{d_1} + \frac{H-h+1}{2}\sqrt{2\log\left(\frac{1}{\delta}\right) + \log \det(M_{1,k+1})};
		    \end{gather*}
			\hfill{\textcolor{blue!80}{$\triangleright$ Computing Q functions}}
		\ENDFOR
\end{algorithmic}
\end{algorithm}

The iterative Q-update for Algorithm \ref{LinUCRL-VTR} is
\begin{equation}
\begin{gathered}\label{UCRL_VTR-Q-update}
    V_{h+1,k}(s) = 0 \\
    Q_{h,k}(s,a) = r(s,a) + X_{h,k}^\top \theta_k + \sqrt{\beta_{h,k}} \sqrt{X_{h,k}^\top M_{1,k+1}^{-1}X_{h,k}} \\
    V_{h,k}(s) = \max_a{Q_{h,k}(s,a)}
\end{gathered}
\end{equation}

The choice of the confidence bounds used in Algorithm \ref{LinUCRL-VTR} comes from the tight bounds derived in \cite{abbasi2011improved} for linear bandits and further expanded upon in Chapter 20 of \cite{lattimore2018bandit}. The details of which are shown and stated in \ref{YasinAnalysis}. We slightly tighten the values for the noise at each stage by using the fact that for each stage in the horizon, $h \in [H]$, the value $V_h^k(\cdot)$ is capped as to never be greater than $H-h+1$. The appearance of the $\sqrt{d_1}$ comes from the fact that $\norm{\theta_*}_2 \leq \sqrt{d_1}$ for all $\theta_* \in \mathbb{R}^d$ in the tabular setting since $\theta_*$ in the tabular setting is equal to the true model of the environment.

\subsection{EGRL-VTR} In this section we discuss the algorithm EGRL-VTR. This algorithm is very similar to UCRL-VTR expect it performs $\varepsilon$-greedy value iteration instead of optimistic value iteration and acts $\varepsilon$-greedy with respect to $Q_{h,k}$.
\begin{algorithm}[H]
	\caption{EGRL-VTR}
	\label{EGRL-VTR}
	\begin{algorithmic}[1]
		\STATE \textbf{Input: } MDP, $d, H, T=KH, \varepsilon > 0$;
		\STATE \textbf{Initialize: } $M_{1,1}\leftarrow I$, $\quad w_{1, 1}\leftarrow 0\in\mathbb{R}^{d\times 1}, \quad \theta_{1}\leftarrow M_{1,1}^{-1}w_{1, 1}\qquad$ for $1\le h\le H$;
		\STATE Compute Q-function $Q_{h, 1}$ using $\theta_{1, 1}$ according to (\ref{EGRL-VTR-Q-update});
		\FOR{$k = 1:K$}
			\STATE Obtain initial state $s_{1}^{k}$ for episode $k$;
			\FOR{$h = 1:H$}

				\STATE With probability $1-\varepsilon$ do $$a_{h}^{k} = \arg\max_{a\in\mA}Q_{h, k}(s_{h}^{k}, a) $$ else pick a uniform random action $a_{h}^{k} \in \mathcal{A}$. Observe the next state $s_{h+1}^{k}$.

				 \STATE Compute the predicted value vector: \hfill{\textcolor{blue!80}{$\triangleright$ Evaluate the expected value of next state}}
 \\
				 \qquad\quad
	 			\begin{equation*}
					\begin{aligned}
						X_{h,k}&\leftarrow \mathbb{E}_{\bigcdot}[V_{h+1, k}(s) |s_{h}^{k}, a_{h}^{k}]
						& = \sum_{s\in\mS}V_{h+1, k}(s)\cdot P_{\bigcdot}(s|s_{h}^{k}, a_{h}^{k}).
					\end{aligned}
				\end{equation*}
	 \STATE $y_{h,k} \leftarrow V_{h+1, k}(s_{h+1}^{k})$		  \hfill{\textcolor{blue!80}{$\triangleright$ Update regression parameters}}
				 \STATE $M_{h+1,k}\leftarrow M_{h,k} + X_{h,k} X_{h,k}^{\top}$
				 \STATE $w_{h+1,k}\leftarrow w_{h,k} + y_{h,k} \cdot X_{h,k}$

			\ENDFOR
			\STATE Update at the end of episode:\hfill{\textcolor{blue!80}{$\triangleright$ Update Model Parameters}}
			\begin{equation*}
				\begin{aligned}
					M_{1,k+1} & \gets M_{H+1,k},\\
					w_{1,k+1} & \gets w_{H+1,k},\\
					\theta_{k+1} &\leftarrow M_{1, k+1}^{-1}w_{1,k+1};
				\end{aligned}
			\end{equation*}
			\STATE Compute $Q_{h, k+1}$ for $h=H,\ldots,1,$ using $\theta_{k+1}$ according to (\ref{EGRL-VTR-Q-update})
			\hfill{\textcolor{blue!80}{$\triangleright$ Computing Q functions}}
		\ENDFOR
\end{algorithmic}
\end{algorithm}
The iterative value update for EGRL-VTR is
\begin{equation}
\begin{gathered}\label{EGRL-VTR-Q-update}
    V_{h+1,k}(s) = 0 \\
    Q_{h,k}(s,a) = r(s,a) + X_{h,k}^\top \theta_k \\
    V_{h,k}(s) = (1 - \varepsilon)\Pi_{[0,H]} \max_{a} Q_{h,k}(s,a)  +  \frac{\varepsilon}{|\mathcal{A}|}\sum_{a \in \mathcal{A}}Q_{h,k}(s,a)
\end{gathered}
\end{equation}

\subsection{EG-Frequency} In this section we discuss the algorithm EG-Frequency. This algorithm is the $\varepsilon$-greedy version of UC-MatrixRL \cite{yang2019reinforcement}.\begin{algorithm}[H]
	\caption{EG-Frequency}
	\label{EG-Freq}
	\begin{algorithmic}[1]
		\STATE \textbf{Input: } MDP, Features $\phi: \mathcal{S}\times\mathcal{A} \rightarrow \mathbb{R}^{|\mathcal{S}||\mathcal{A}|}$ and $\psi: \mathcal{S} \rightarrow \mathbb{R}^{|\mathcal{S}|}$, $\varepsilon > 0$, and the total number of episodes $K$;
		\STATE \textbf{Initialize: } $A_1 \leftarrow I \in \mathbb{R}^{|\mathcal{S}||\mathcal{A}|\times |\mathcal{S}||\mathcal{A}|}$, $M_1 \leftarrow 0\in\mathbb{R}^{|\mathcal{S}||\mathcal{A}|\times |\mathcal{S}|}$, and $K_\psi \leftarrow \sum_{s' \in \mathcal{S}} \psi(s')\psi(s')^\top$;
		\FOR{$k = 1:K$}
			\STATE Let {$Q_{h,k}$} be given in (\ref{EG_Freq-Q-update}) using $M_k$;
			\FOR{$h = 1:H$}
		        \STATE Let the current state be $s_h^k$;
		        \STATE With probability (1$-\varepsilon$) play action $a_{h}^{k} = \arg\max_{a\in\mA}Q_{h, k}(s_{h}^{k}, a) $ else pick a uniform random action $a_h^k \in \mathcal{A}$.
		        \STATE Record the next state $s_{h+1}^k$
			\ENDFOR
			\STATE $A_{k+1} \leftarrow A_{k} + \sum_{h\leq H}\phi(s_h^k,a_h^k) \phi(s_h^k,a_h^k)^\top$
			\STATE $M_{k+1} \leftarrow M_k + A_{k+1}^{-1}\sum_{h \leq H}\phi(s_h^k,a_h^k)\psi(s_{h+1}^k)^\top K_\psi^{-1}$
		\ENDFOR
\end{algorithmic}
\end{algorithm}
The iterative Q-update for EG-Frequency is
\begin{equation}
\begin{gathered}\label{EG_Freq-Q-update}
    Q_{h+1,k}(s,a) = 0 \text{ and }\\
    Q_{h,k}(s,a) = r(s,a) +  \phi(s,a)^\top M_k \mathbf{\Psi}^\top V_{h+1,k} \\ V_{h,k} = (1 - \varepsilon)\Pi_{[0,H]} \max_{a} Q_{h,k}(s,a) + \frac{\varepsilon}{|\mathcal{A}|}\sum_{a \in \mathcal{A}}Q_{h,k}(s,a)
\end{gathered}
\end{equation}
Note that $\mathbf{\Psi}$ is a $|\mathcal{S}| \times |\mathcal{S}|$ whose rows are the features $\psi(s')$ and $\mathbf{\Phi}$ is a $|\mathcal{S}||\mathcal{A}| \times |\mathcal{S}||\mathcal{A}|$ whose rows are the features $\phi(s,a)$. In the tabular RL setting both $\mathbf{\Psi}$ and $\mathbf{\Phi}$ are the identity matrix which is what we used in our numerical experiments. In the tabular RL setting, EG-Frequency stores the counts of the number of times it transitioned to next state $s'$ from the state-action pair $(s,a)$ and fits the estimated model $M_k$ accordingly.
\subsubsection{Futher Implementation Notes}
In this section, we include some further details on how we implemented Algorithms \ref{LinUCRL-VTR}, \ref{EGRL-VTR}, and \ref{EG-Freq}. All code was written in Python 3 and used the Numpy and Scipy libraries. All plots were generated using MatPlotLib. In Algorithm \ref{LinUCRL-VTR}, Numpy's logdet function was used to calculate the determinate in step 15 for numerical stability purposes. No matrix inversion was performed in our code, instead a Sherman-Morrison update was performed for each matrix in which a matrix inversion is performed at each $(k,h)$ in order to save on computation. To read more about the Sherman Morrison update in the context of RL, we refer to the reader to Eqn (9.22) of \cite{sutton2018introduction}. When computing the weighted L1-norm, we added a small constant to each summation in the denominator to avoid dividing by zero. Finally, when computing UC-MatrixRL we also used the self-normalize bounds introduced in the beginning of this section. Some pseudocode for using self-normalized bounds with UC-MatrixRL can be found in step 5 of Alg \ref{UCRL-MIX}.

\section{Mixture Model} 
\label{sec:mixture_model} In this section, we introduce, analyze, and evaluate a Linear model-based RL algorithm that used both the canonical model and the VTR model for planning. We call this algorithm UCRL-MIX.
\subsection{UCRL-MIX} Below a meta-algorithm for UCRL-MIX
\begin{algorithm}[H]
	\caption{UCRL-MIX}
	\label{UCRL-MIX}
	\begin{algorithmic}[1]
		\STATE Compute Algorithm \ref{LinUCRL-VTR} and UC-MatrixRL \cite{yang2019reinforcement} simultaneously.
		\STATE At end of episode $k$, perform value iteration and set $V_{H+1,k}(s) = 0$.
		\FOR {$h = H+1:1$}
		\FOR {$s \in |\mathcal{S}|$ and $a \in |\mathcal{A}|$}
		\STATE Compute the confidence set bonuses as follows
		\begin{gather*}
		B_{h,k}^{VTR} \leftarrow \sqrt{d_1} + \frac{H-h+1}{2}\sqrt{2\log\left(\frac{2}{\delta}\right) + \log \det(M_{1,k+1})};
		\\
		B_{h,k}^{MAT}  \leftarrow \sqrt{|\mathcal{S}||\mathcal{A}|} + \frac{H-h+1}{2}\sqrt{2\log\left(\frac{2}{\delta}\right) + \log \det(A_{k+1})};
		\end{gather*}
			\IF  {$B_{h,k}^{VTR}\sqrt{X_{h,k}^\top M_{1,k+1}^{-1} X_{h,k}} \leq B_{h,k}^{MAT}\sqrt{\phi^\top (s,a) A_n^{-1} \phi(s,a)}$}
			\STATE Perform one step of value iteration using the VTR model as follows: $Q_{h,k}(s,a) = r(s,a) + X_{h,k}^\top \theta_k + \sqrt{\beta_{h,k}} \sqrt{X_{h,k}^\top M_{1,k+1}^{-1}X_{h,k}}$
			\ELSE 
			\STATE Update $Q_{h,k}(s,a)$ according to Equation 8 \cite{yang2019reinforcement} using the UC-MatrixRL model $A_k$. Note that in \cite{yang2019reinforcement} they use $n$ to denote the current episode, in our paper we use $k$ to denote the current episode.
			\ENDIF
			\STATE $V_{h,k}(s) = \max_a Q_{h,k}(s,a)$
		\ENDFOR
		\ENDFOR
\end{algorithmic}
\end{algorithm}

We are now using multiple models instead of a single model, we must adjust our confidence sets accordingly. By using a union bound we replace $\delta$ with $\delta/2$ for our confidence parameter. This updated confidence parameter changes the term inside the logarithm. We now have $\log(2/\delta)$ where as before we had $\log(1/\delta)$.

\subsection{Numerical Results}  
We will include the cumulative regret and the weighted L1 norm of UCRL-MIX on the RiverSwim environment as in Section \ref{sec:experiment}. We also include a bar graph of the relative frequency with which the algorithm used the VTR-model for planning and the canonical model for planning. 

\begin{figure}[H]
\centering
\begin{subfigure}{.32\textwidth}
  \centering
  \includegraphics[width=1.1\linewidth]{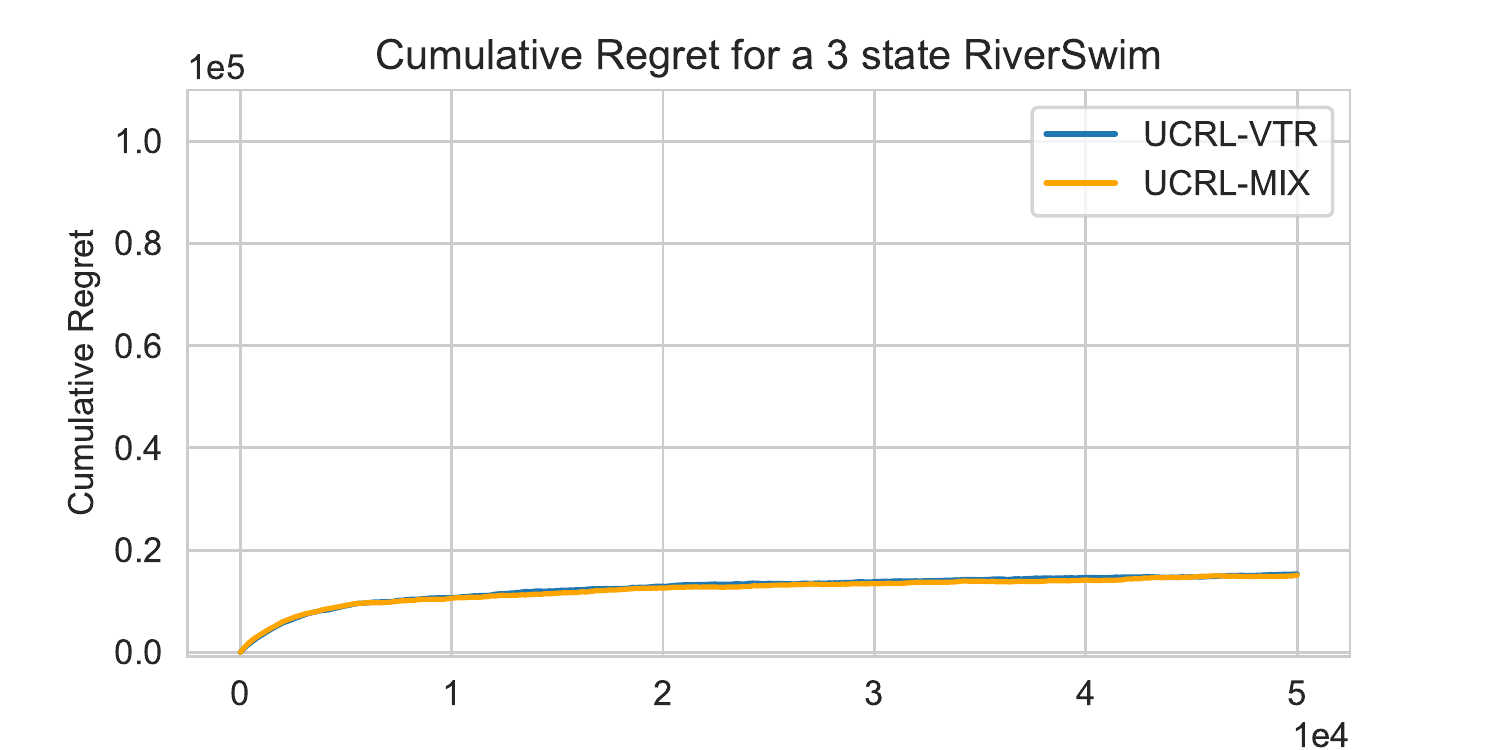}

\end{subfigure}%
\begin{subfigure}{.32\textwidth}
  \centering
  \includegraphics[width=1.1\linewidth]{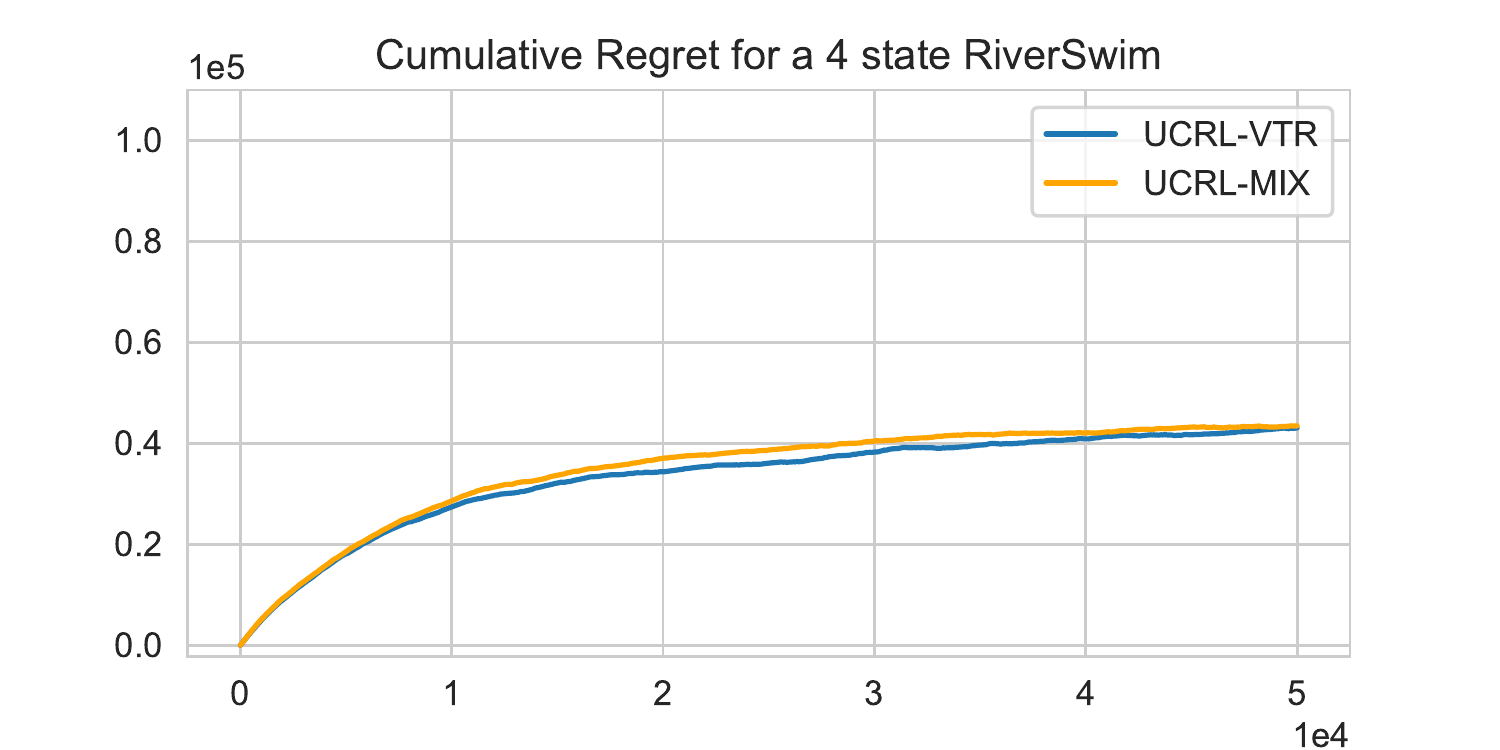}
\end{subfigure}
\begin{subfigure}{.32\textwidth}
  \centering
  \includegraphics[width=1.1\linewidth]{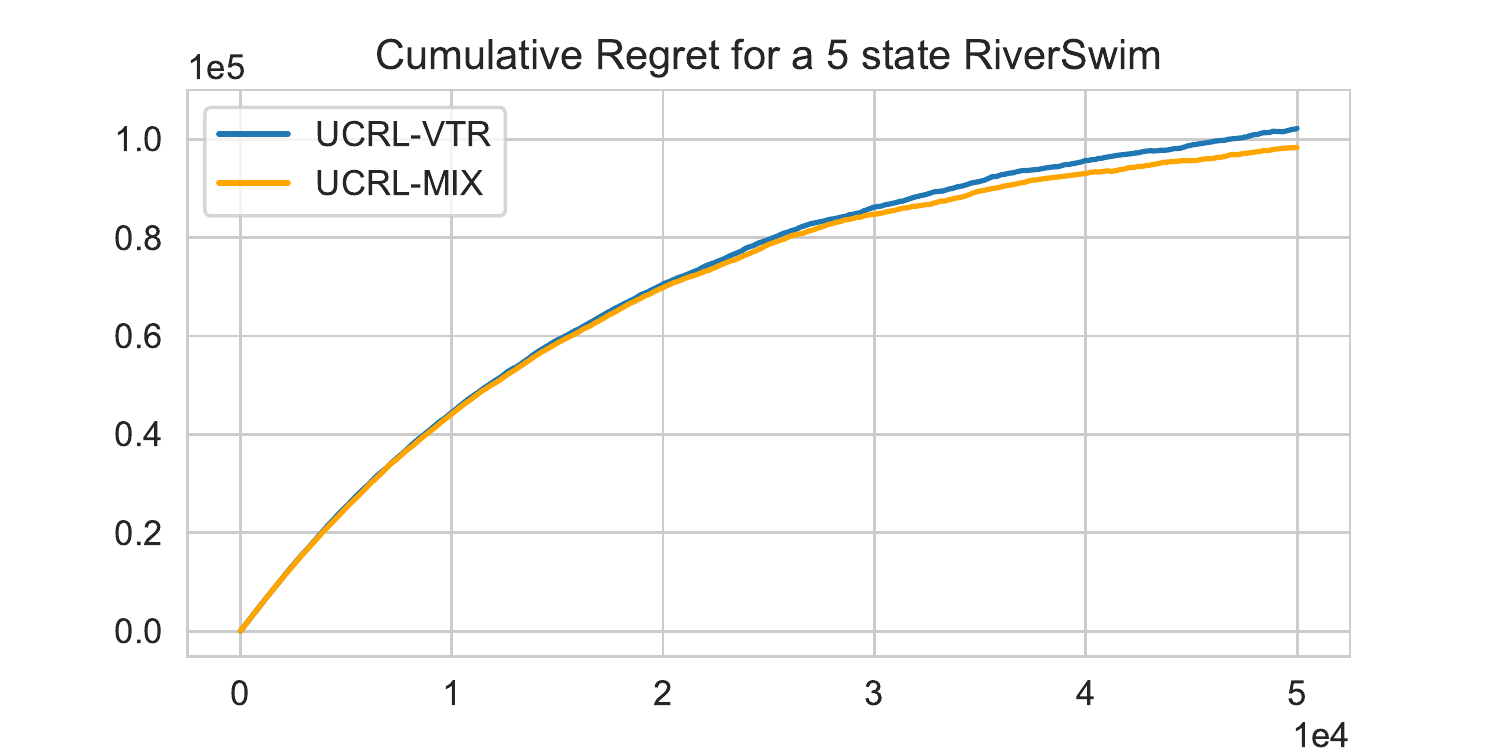}
\end{subfigure}

\begin{subfigure}{.32\textwidth}
  \centering
 \includegraphics[width=1.1\linewidth]{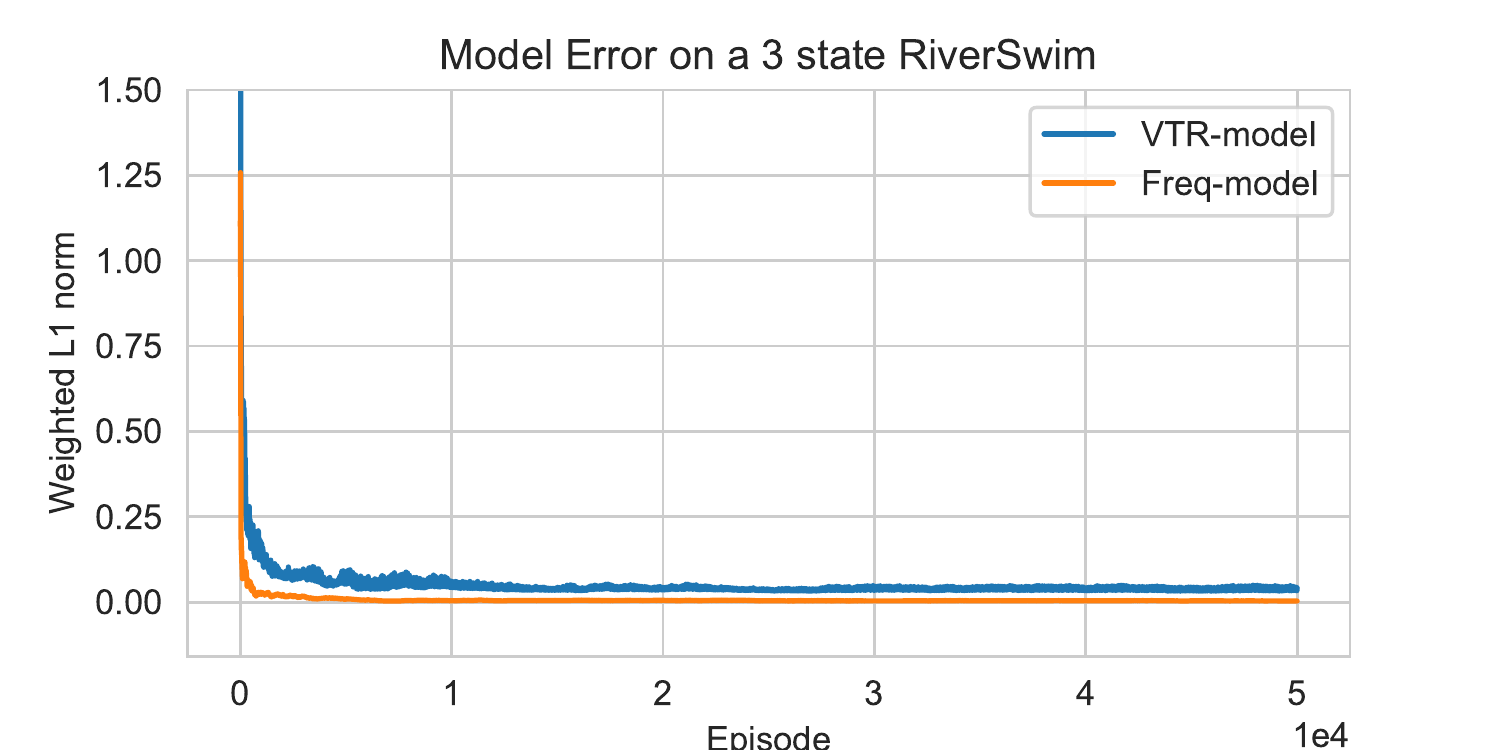}
\end{subfigure}%
\begin{subfigure}{.32\textwidth}
  \centering
  \includegraphics[width=1.1\linewidth]{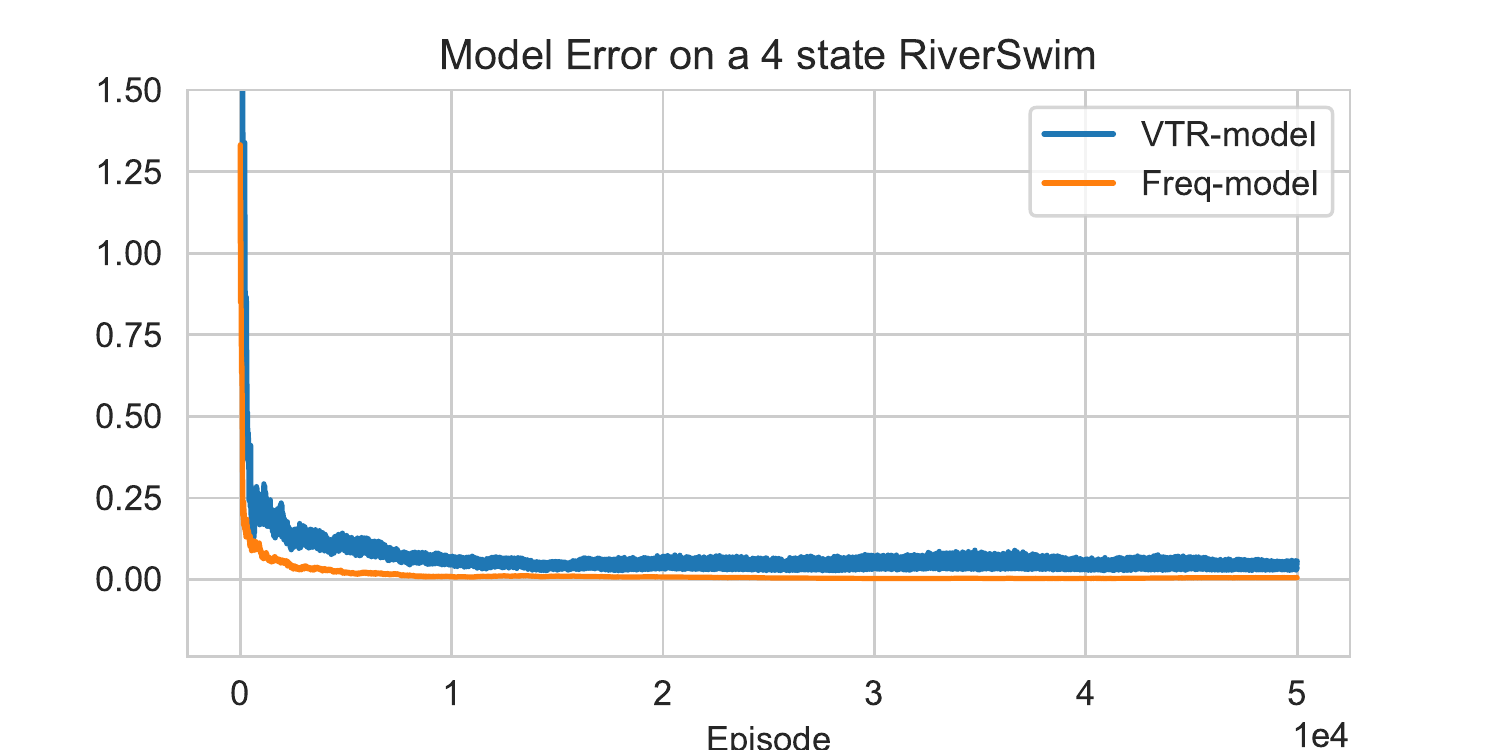}
\end{subfigure}
\begin{subfigure}{.32\textwidth}
  \centering
  \includegraphics[width=1.1\linewidth]{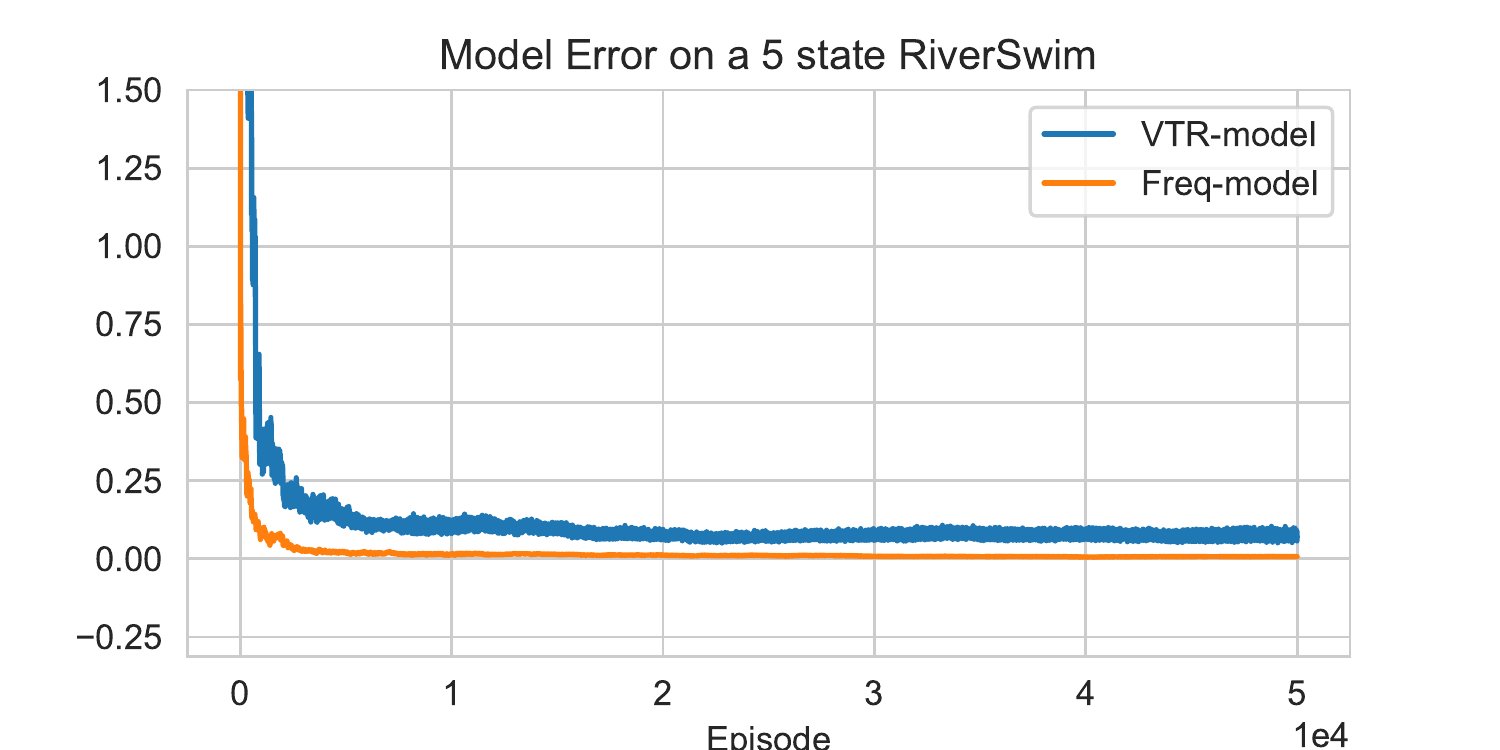}
\end{subfigure}
\caption{In the plots for the model error we include model error for both the VTR-model and the canonical model. Even though only one is used during planning both are updated at the end of each episode.}
\label{fig:riverswimmix}
\end{figure}

If we compare the results of Figure \ref{fig:riverswimmix} with the results of Figure \ref{fig:RiverSwim} from Section \ref{sec:riverswim} we see that the cumulative regret of UCRL-MIX is almost identical to the cumulative regret of UCRL-VTR. The model errors of both the VTR and the canonical models are almost identical to the model errors of UCRL-VTR and UC-MatrixRL respectively. 

\begin{figure}[H]
\centering
\begin{subfigure}{.32\textwidth}
  \centering
  \includegraphics[width=1.1\linewidth]{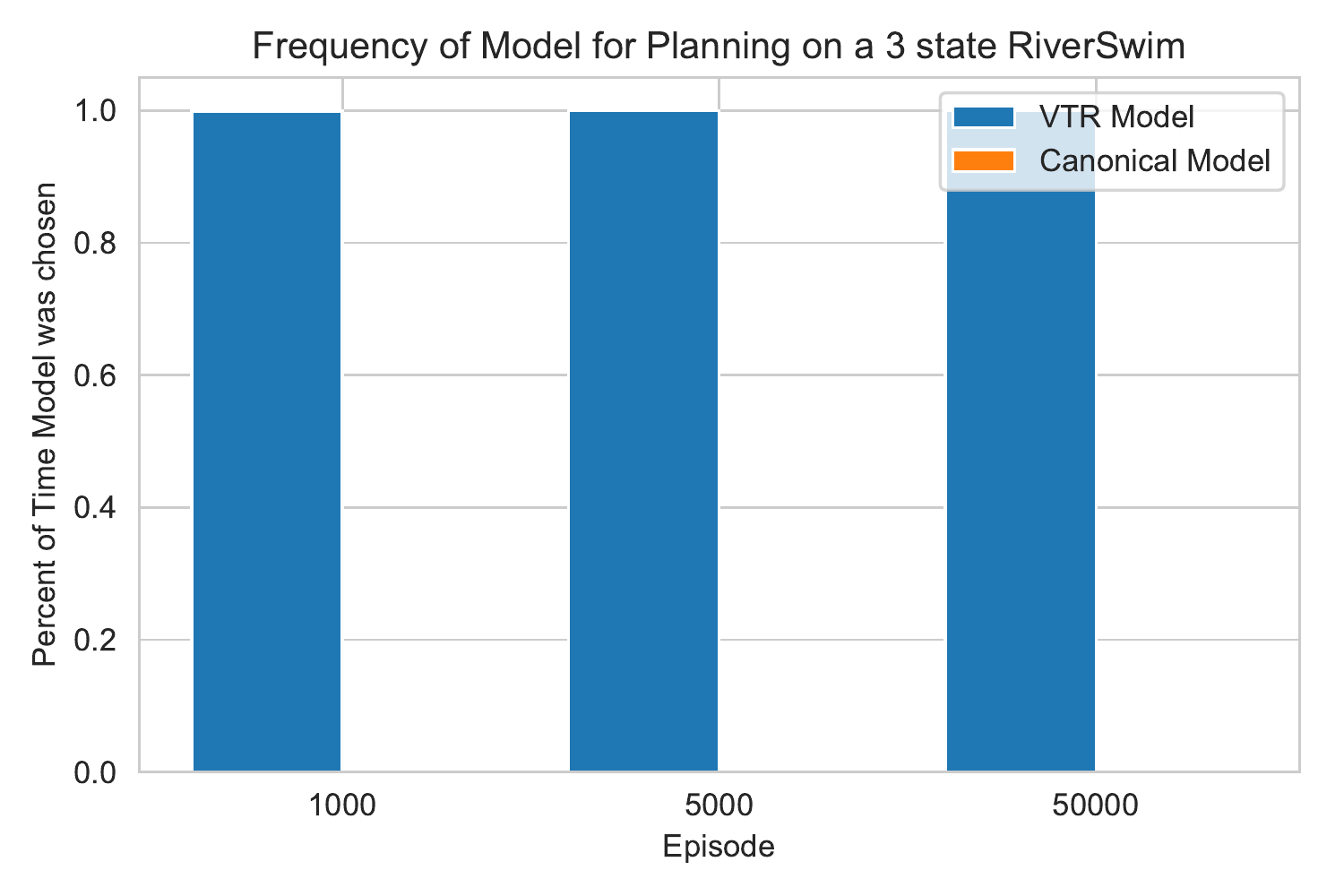}

\end{subfigure}%
\begin{subfigure}{.32\textwidth}
  \centering
  \includegraphics[width=1.1\linewidth]{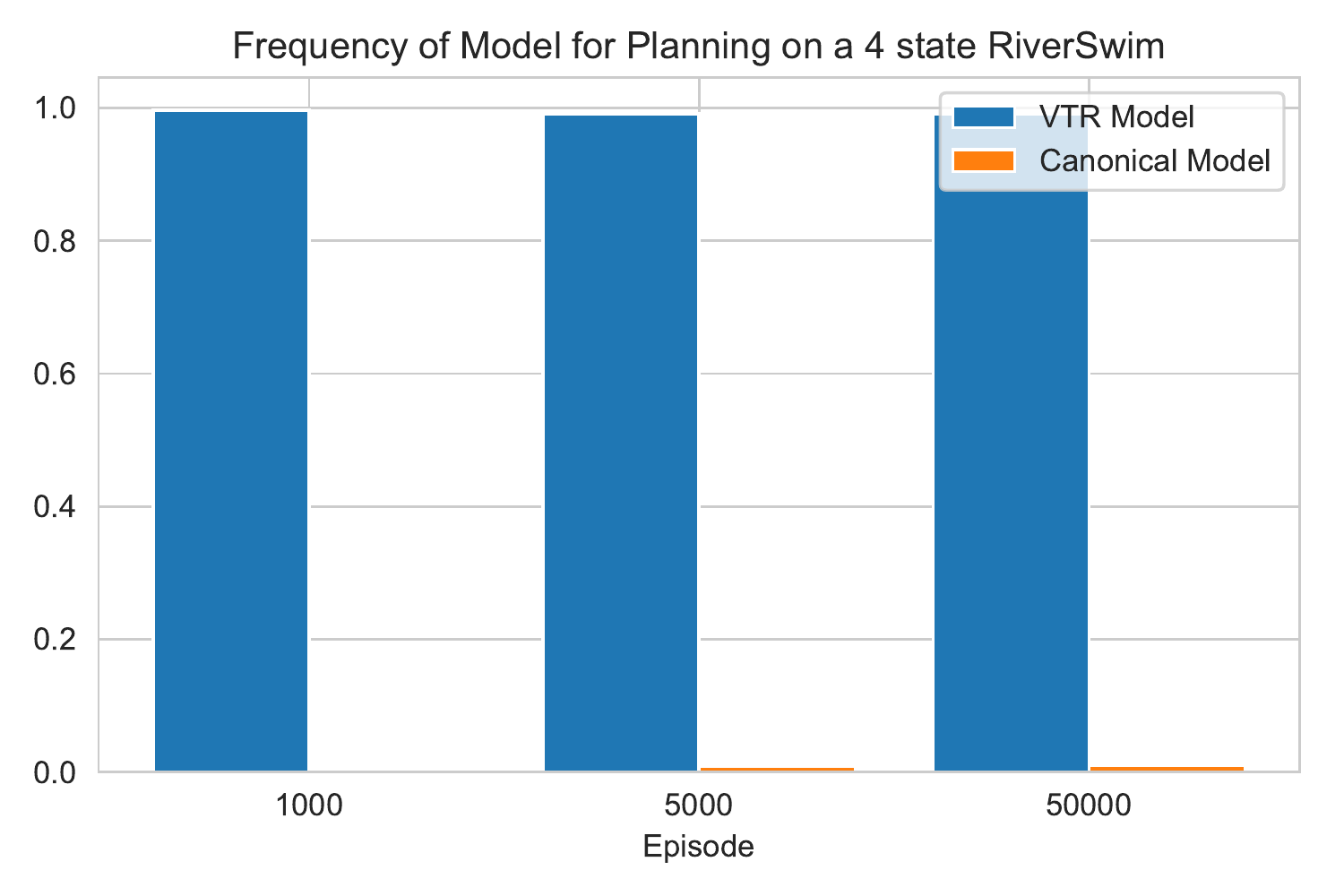}
\end{subfigure}
\begin{subfigure}{.32\textwidth}
  \centering
  \includegraphics[width=1.1\linewidth]{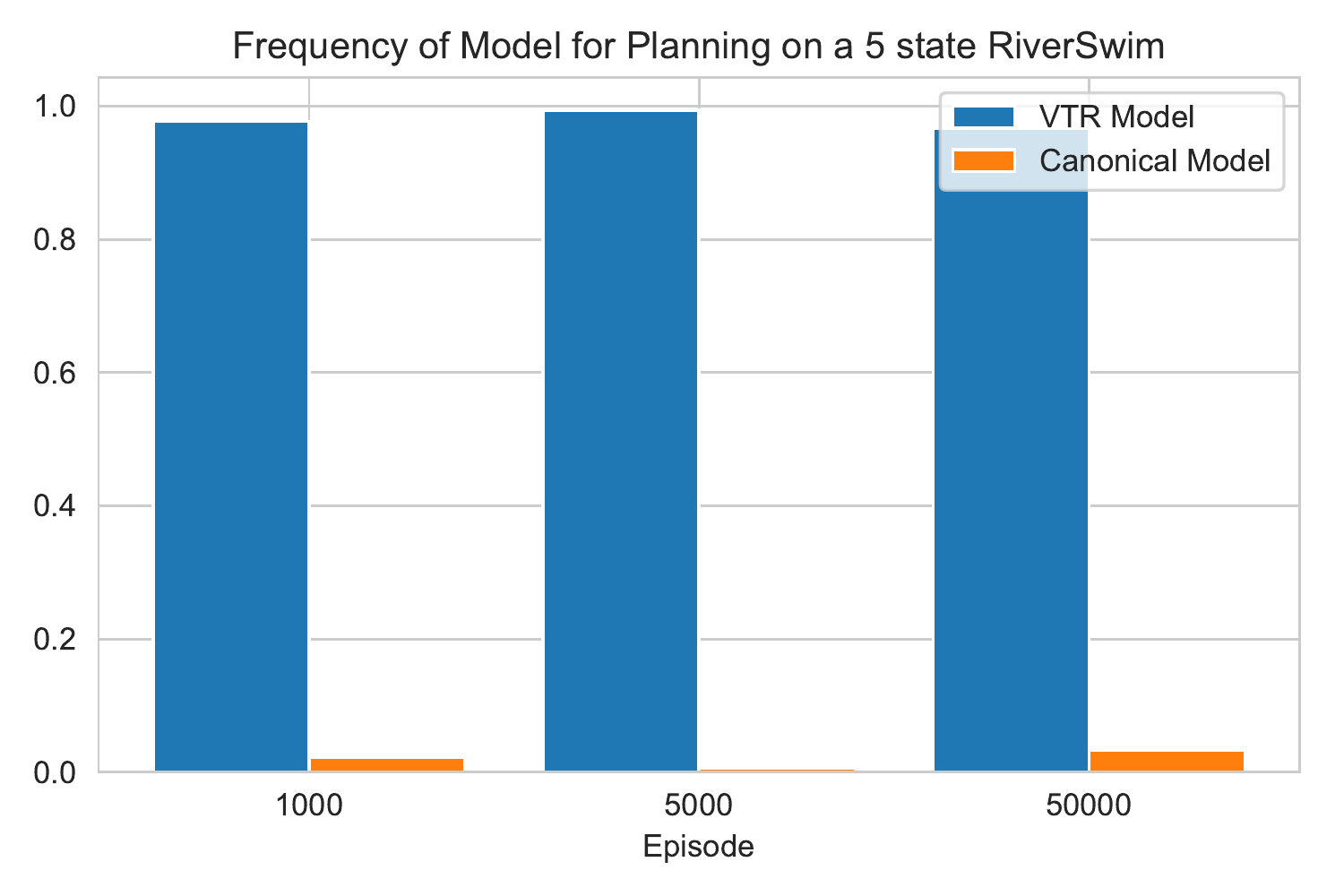}
\end{subfigure}
\caption{UCRL-MIX rarely, if ever, chooses the canonical model for planning on the RiverSwim environments. }
\label{fig:RiverSwimMixFreq}
\end{figure}

From Figure \ref{fig:RiverSwimMixFreq}, we see that on the RiverSwim environment, UCRL-MIX almost always uses the VTR-model for planning. We calculate this frequency by counting the number of times Step 7 of Alg \ref{UCRL-MIX} was observed up until episode $k$ and by counting the number of times Step 9 of Alg \ref{UCRL-MIX} was observed up until episode $k$. We then divide these counts by the sum of the counts to get a percentage. We believe the reason the algorithm overwhelming chose the VTR-model was due to the fact that the confidence intervals for the VTR-model shrink much faster than the confidence intervals for the canonical model. The canonical model is forced to explore much longer than the VTR-model as its objective is to learn a globally optimal model rather than a model that yields high reward. Thus, the canonical model is forced to explore all state-action-next state tuples, even ones that do not yield high reward, in order to meet its objective of learning a globally optimal model while the VTR-model is only forced to explore state-action-next state tuples that fall in-line with its objective of accumulating high reward. The set of all state-action-next state tuples is much larger then the set of state-action-next state tuples that yield high reward which means the confidence intervals for the canonical model shrink slower than the confidence sets of the VTR-model on the RiverSwim environment.


\end{document}